\documentclass[10pt,twocolumn,letterpaper]{article}
\usepackage[]{bibunits}

\usepackage[svgnames,pdf]{pstricks}
\usepackage{pst-node, pst-plot, pst-circ}

\usepackage{cvpr}              
\usepackage[accsupp]{axessibility}

\usepackage{graphicx}
\usepackage{import}
\usepackage{amsmath}
\usepackage{amssymb}
\usepackage{booktabs}
\usepackage{xfrac}

\usepackage{bm}
\usepackage{amsthm}
\usepackage{bbm}
\usepackage{mathrsfs}
\usepackage[notransparent]{svg} 
\usepackage{enumitem}
\usepackage{multirow}
\usepackage{tabularx}

\usepackage[export]{adjustbox}  

\usepackage[accsupp]{axessibility}  
\usepackage{xcolor}  
\usepackage{colortbl}

\usepackage[boxed,noalgohanging,linesnumbered,vlined]{algorithm2e}

\usepackage{rotating}

\usepackage{textcomp}

\defaultbibliographystyle{ieee_fullname}

\graphicspath{./img/}
\usepackage{stmaryrd}

\usepackage{multirow}

\newtheorem{proposition}{Proposition}
\newtheorem{lemma}{Lemma}
\newtheorem{definition}{Definition}
\newtheorem{corollary}{Corollary}

\usepackage{mathtools}
\DeclarePairedDelimiterX{\klx}[2]{(}{)}{%
  #1\,\delimsize\|\,#2%
}
\newcommand{\kl}{\mathrm{KL}\klx}

\newcommand{\paragrax}[1]{{\bf #1}\quad}

\newcolumntype{Y}{>{\centering\arraybackslash}X}

\newcommand{\Lm}[2]{\mathcal{L}_{\text{match}}\left(\hat{\bm{y}}_{#1},\bm{y}_{#2}\right)}
\newcommand{\LM}{\Lm{i}{j}}

 \makeatletter
\def\hyper@natlinkstart#1{%
  \Hy@backout{#1}%
  \hyper@linkstart{cite}{cite.\@bibunitname.#1}%
  \def\hyper@nat@current{#1}%
}

\def\hyper@natlinkbreak#1#2{%
  \hyper@linkend#1\hyper@linkstart{cite}{cite.\@bibunitname.#2}%
}

\def\hyper@natanchorstart#1{%
  \hyper@anchorstart{cite.\@bibunitname.#1}%
}

\def\bibcite#1#2{%
  \@newl@bel{b}{#1}{\hyper@@link[cite]{}{cite.\@bibunitname.#1}{#2}}%
}%

\def\@lbibitem[#1]#2{%
  \@skiphyperreftrue
  \H@item[\hyper@anchorstart{cite.\@bibunitname.#2}%
  \@BIBLABEL{#1}\hyper@anchorend\hfill]%
  \@skiphyperreffalse
  \if@filesw
    \begingroup
      \let\protect\noexpand
      \immediate\write\@auxout{%
        \string\bibcite{#2}{#1}%
      }%
    \endgroup
  \fi
  \ignorespaces
}%

\def\@bibitem#1{%
  \@skiphyperreftrue\H@item\@skiphyperreffalse
  \hyper@anchorstart{cite.\@bibunitname.#1}\relax\hyper@anchorend
  \if@filesw
    \begingroup
      \let\protect\noexpand
      \immediate\write\@auxout{%
        \string\bibcite{#1}{\the\value{\@listctr}}%
      }%
    \endgroup
  \fi
  \ignorespaces
}%

\def\@citex[#1]#2{%
  \let\@citea\@empty
  \@cite{%
    \@for\@citeb:=#2\do{%
      \@citea
      \def\@citea{,\penalty\@m\ }%
      \edef\@citeb{\expandafter\@firstofone\@citeb}%
      \if@filesw
        \immediate\write\@auxout{\string\citation{\@citeb}}%
      \fi
      \@ifundefined{b@\@citeb}{%
        \mbox{\reset@font\bfseries ?}%
        \G@refundefinedtrue
        \@latex@warning{%
          Citation `\@citeb' on page \thepage \space undefined%
        }%
      }{%
        \hyper@natlinkstart{#2}%
            \hbox{\csname b@\@citeb\endcsname}%
        \hyper@natlinkend%
      }%
    }%
  }{#1}%
}%

\makeatother

\usepackage{hyperref}
\hypersetup{pagebackref,breaklinks,colorlinks,allcolors=black,pdftex}

\usepackage[capitalize]{cleveref}
\crefname{section}{Sec.}{Secs.}
\Crefname{section}{Section}{Sections}
\Crefname{table}{Table}{Tables}
\crefname{table}{Tab.}{Tabs.}

\def\cvprPaperID{7910} 
\def\confName{CVPR}
\def\confYear{2023}

\begin{document}

\title{Unbalanced Optimal Transport: A Unified Framework for Object Detection}


\author{Henri De Plaen\textsuperscript{1}\textsuperscript{*} \quad Pierre-François De Plaen\textsuperscript{2}\textsuperscript{*} \quad Johan A. K. Suykens\textsuperscript{1} \quad Marc Proesmans\textsuperscript{2} \\Tinne Tuytelaars\textsuperscript{2} \quad Luc Van Gool\textsuperscript{2,3}\\
\textsuperscript{1}ESAT-STADIUS, KU Leuven, Belgium \quad \textsuperscript{2}ESAT-PSI, KU Leuven, Belgium \\ \textsuperscript{3}Computer Vision Lab, ETH Zürich, Switzerland
}

\maketitle 
\def\thefootnote{*}
\begin{NoHyper}
\footnotetext{These authors contributed equally.}
\end{NoHyper}
\def\thefootnote{\arabic{footnote}}

\begin{abstract}
During training, supervised object detection tries to correctly match the predicted bounding boxes and associated classification scores to the ground truth. This is essential to determine which predictions are to be pushed towards which solutions, or to be discarded. Popular matching strategies include matching to the closest ground truth box (mostly used in combination with anchors), or matching via the Hungarian algorithm (mostly used in anchor-free methods). Each of these strategies comes with its own properties, underlying losses, and heuristics. We show how Unbalanced Optimal Transport unifies these different approaches and opens a whole continuum of methods in between. This allows for a finer selection of the desired properties. Experimentally, we show that training an object detection model with Unbalanced Optimal Transport is able to reach the state-of-the-art both in terms of Average Precision and Average Recall as well as to provide a faster initial convergence. The approach is well suited for GPU implementation, which proves to be an advantage for large-scale models. 
\end{abstract}


\begin{figure}
    \begin{tabular}[t]{ccc}
    \multicolumn{2}{c}{%
    \begin{subfigure}{0.26\textwidth}
        \def\svgwidth{\textwidth}\footnotesize
\begingroup%
  \makeatletter%
  \providecommand\color[2][]{%
    \errmessage{(Inkscape) Color is used for the text in Inkscape, but the package 'color.sty' is not loaded}%
    \renewcommand\color[2][]{}%
  }%
  \providecommand\transparent[1]{%
    \errmessage{(Inkscape) Transparency is used (non-zero) for the text in Inkscape, but the package 'transparent.sty' is not loaded}%
    \renewcommand\transparent[1]{}%
  }%
  \providecommand\rotatebox[2]{#2}%
  \newcommand*\fsize{\dimexpr\f@size pt\relax}%
  \newcommand*\lineheight[1]{\fontsize{\fsize}{#1\fsize}\selectfont}%
  \ifx\svgwidth\undefined%
    \setlength{\unitlength}{374.9999968bp}%
    \ifx\svgscale\undefined%
      \relax%
    \else%
      \setlength{\unitlength}{\unitlength * \real{\svgscale}}%
    \fi%
  \else%
    \setlength{\unitlength}{\svgwidth}%
  \fi%
  \global\let\svgwidth\undefined%
  \global\let\svgscale\undefined%
  \makeatother%
  \begin{picture}(1,0.864)%
    \lineheight{1}%
    \setlength\tabcolsep{0pt}%
    \put(0,0){\includegraphics[width=\unitlength,page=1]{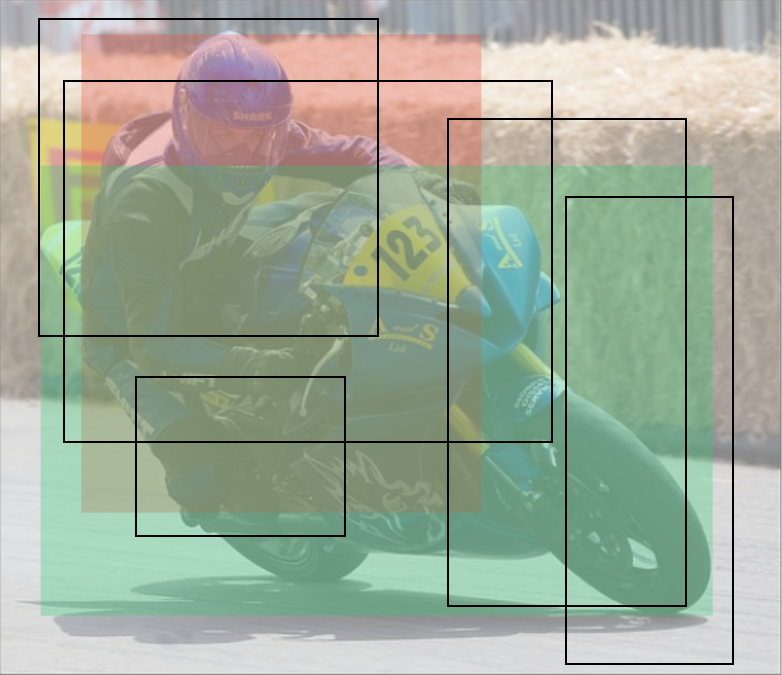}}%
    \put(0.118,0.756){\color[rgb]{0.72156863,0.01960784,0.01960784}\makebox(0,0)[lt]{\lineheight{1.25}\smash{\begin{tabular}[t]{l}A\end{tabular}}}}%
    \put(0.066,0.58800001){\color[rgb]{0,0.59607843,0.14901961}\makebox(0,0)[lt]{\lineheight{1.25}\smash{\begin{tabular}[t]{l}B\end{tabular}}}}%
    \put(0.064,0.776){\makebox(0,0)[lt]{\lineheight{1.25}\smash{\begin{tabular}[t]{l}1\end{tabular}}}}%
    \put(0.73800003,0.54800001){\makebox(0,0)[lt]{\lineheight{1.25}\smash{\begin{tabular}[t]{l}2\end{tabular}}}}%
    \put(0.18799999,0.31800001){\makebox(0,0)[lt]{\lineheight{1.25}\smash{\begin{tabular}[t]{l}3\end{tabular}}}}%
    \put(0.096,0.69600001){\makebox(0,0)[lt]{\lineheight{1.25}\smash{\begin{tabular}[t]{l}4\end{tabular}}}}%
    \put(0.58799999,0.64800001){\makebox(0,0)[lt]{\lineheight{1.25}\smash{\begin{tabular}[t]{l}5\end{tabular}}}}%
  \end{picture}%
\endgroup%
 
        \caption{Image \textnumero 163 from the VOC training dataset. The ground truth boxes are colored, and the predictions are outlined in black.%
        \label{fig:general-img}}
    \end{subfigure}}%
    &
    \begin{subfigure}{.13\textwidth}
        \centering \footnotesize
        \def\svgwidth{0.85\textwidth}
        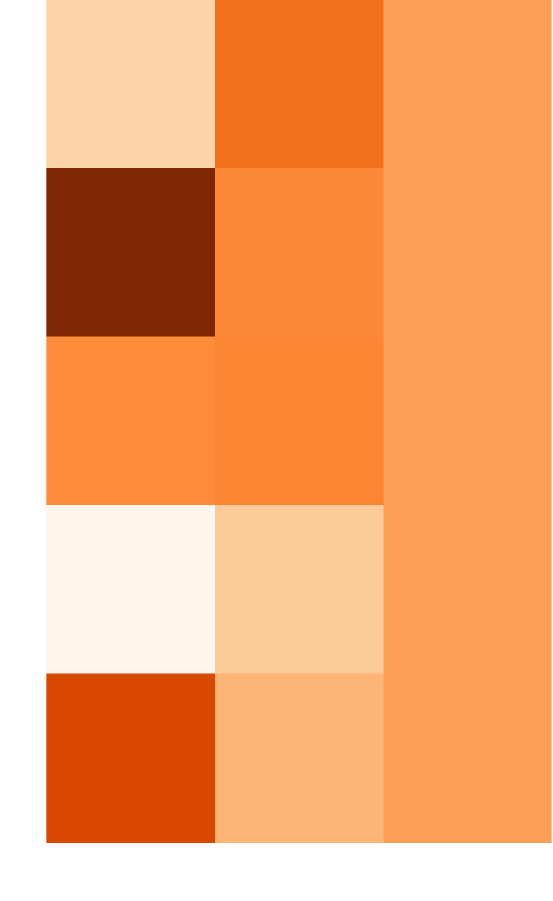
        \caption{Costs between the predictions and the ground truth ($1-\mathrm{GIoU}$). The background cost is $c_{\varnothing} = 0.8$.%
        \label{fig:general-cost}}
    \end{subfigure}\\%
    \subcaptionbox{Prediction to best ground truth (Unbalanced OT with $\epsilon=0$, $\tau_1\rightarrow+\infty$ and $\tau_2=0$).%
        \label{fig:general-min1}}[.13\textwidth]%
        {\centering \footnotesize
        \def\svgwidth{.11\textwidth}
\begingroup%
  \makeatletter%
  \providecommand\color[2][]{%
    \errmessage{(Inkscape) Color is used for the text in Inkscape, but the package 'color.sty' is not loaded}%
    \renewcommand\color[2][]{}%
  }%
  \providecommand\transparent[1]{%
    \errmessage{(Inkscape) Transparency is used (non-zero) for the text in Inkscape, but the package 'transparent.sty' is not loaded}%
    \renewcommand\transparent[1]{}%
  }%
  \providecommand\rotatebox[2]{#2}%
  \newcommand*\fsize{\dimexpr\f@size pt\relax}%
  \newcommand*\lineheight[1]{\fontsize{\fsize}{#1\fsize}\selectfont}%
  \ifx\svgwidth\undefined%
    \setlength{\unitlength}{204.59437554bp}%
    \ifx\svgscale\undefined%
      \relax%
    \else%
      \setlength{\unitlength}{\unitlength * \real{\svgscale}}%
    \fi%
  \else%
    \setlength{\unitlength}{\svgwidth}%
  \fi%
  \global\let\svgwidth\undefined%
  \global\let\svgscale\undefined%
  \makeatother%
  \begin{picture}(1,1.61983314)%
    \lineheight{1}%
    \setlength\tabcolsep{0pt}%
    \put(0,0){\includegraphics[width=\unitlength,page=1]{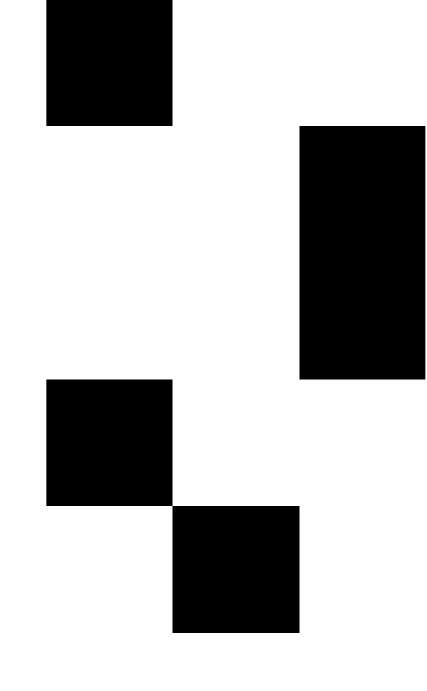}}%
    \put(0.25715823,0.00015279){\color[rgb]{0.90588235,0.29803922,0.23529412}\makebox(0,0)[t]{\lineheight{1.25}\smash{\begin{tabular}[t]{c}A\end{tabular}}}}%
    \put(0.5535129,0.00015279){\color[rgb]{0.18039216,0.8,0.44313725}\makebox(0,0)[t]{\lineheight{1.25}\smash{\begin{tabular}[t]{c}B\end{tabular}}}}%
    \put(0.78104848,0.00152741){\makebox(0,0)[lt]{\lineheight{1.25}\smash{\begin{tabular}[t]{l}$\varnothing$\end{tabular}}}}%
    \put(0.09187386,1.41027825){\makebox(0,0)[rt]{\lineheight{1.25}\smash{\begin{tabular}[t]{r}1\end{tabular}}}}%
    \put(0.09187386,1.11392358){\makebox(0,0)[rt]{\lineheight{1.25}\smash{\begin{tabular}[t]{r}2\end{tabular}}}}%
    \put(0.09187386,0.8175689){\makebox(0,0)[rt]{\lineheight{1.25}\smash{\begin{tabular}[t]{r}3\end{tabular}}}}%
    \put(0.09187386,0.52121426){\makebox(0,0)[rt]{\lineheight{1.25}\smash{\begin{tabular}[t]{r}4\end{tabular}}}}%
    \put(0.09187386,0.22485963){\makebox(0,0)[rt]{\lineheight{1.25}\smash{\begin{tabular}[t]{r}5\end{tabular}}}}%
    \put(0,0){\includegraphics[width=\unitlength,page=2]{min1.pdf}}%
  \end{picture}%
\endgroup%

        }%
    &%
    \subcaptionbox{Hungarian matching (OT with $\epsilon=0$, $\tau_1 \rightarrow +\infty$ and $\tau_2 \rightarrow + \infty$).%
        \label{fig:general-hungarian}}[.13\textwidth]%
        {\centering \footnotesize
        \def\svgwidth{.11\textwidth}
\begingroup%
  \makeatletter%
  \providecommand\color[2][]{%
    \errmessage{(Inkscape) Color is used for the text in Inkscape, but the package 'color.sty' is not loaded}%
    \renewcommand\color[2][]{}%
  }%
  \providecommand\transparent[1]{%
    \errmessage{(Inkscape) Transparency is used (non-zero) for the text in Inkscape, but the package 'transparent.sty' is not loaded}%
    \renewcommand\transparent[1]{}%
  }%
  \providecommand\rotatebox[2]{#2}%
  \newcommand*\fsize{\dimexpr\f@size pt\relax}%
  \newcommand*\lineheight[1]{\fontsize{\fsize}{#1\fsize}\selectfont}%
  \ifx\svgwidth\undefined%
    \setlength{\unitlength}{204.59437554bp}%
    \ifx\svgscale\undefined%
      \relax%
    \else%
      \setlength{\unitlength}{\unitlength * \real{\svgscale}}%
    \fi%
  \else%
    \setlength{\unitlength}{\svgwidth}%
  \fi%
  \global\let\svgwidth\undefined%
  \global\let\svgscale\undefined%
  \makeatother%
  \begin{picture}(1,1.61983314)%
    \lineheight{1}%
    \setlength\tabcolsep{0pt}%
    \put(0,0){\includegraphics[width=\unitlength,page=1]{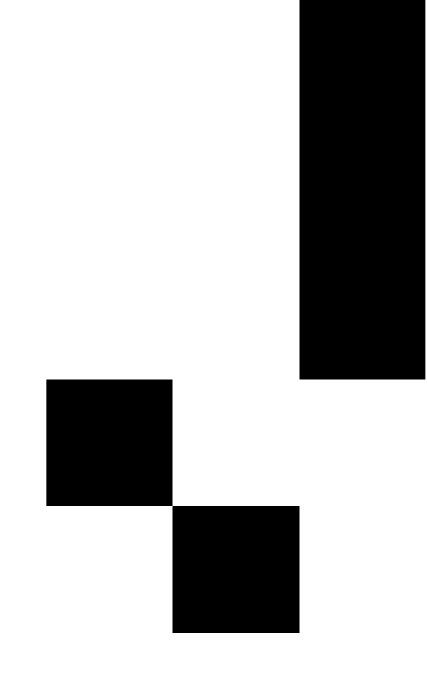}}%
    \put(0.25715823,0.00015279){\color[rgb]{0.90588235,0.29803922,0.23529412}\makebox(0,0)[t]{\lineheight{1.25}\smash{\begin{tabular}[t]{c}A\end{tabular}}}}%
    \put(0.5535129,0.00015279){\color[rgb]{0.18039216,0.8,0.44313725}\makebox(0,0)[t]{\lineheight{1.25}\smash{\begin{tabular}[t]{c}B\end{tabular}}}}%
    \put(0.78104848,0.00152741){\makebox(0,0)[lt]{\lineheight{1.25}\smash{\begin{tabular}[t]{l}$\varnothing$\end{tabular}}}}%
    \put(0.09187386,1.41027825){\makebox(0,0)[rt]{\lineheight{1.25}\smash{\begin{tabular}[t]{r}1\end{tabular}}}}%
    \put(0.09187386,1.11392358){\makebox(0,0)[rt]{\lineheight{1.25}\smash{\begin{tabular}[t]{r}2\end{tabular}}}}%
    \put(0.09187386,0.8175689){\makebox(0,0)[rt]{\lineheight{1.25}\smash{\begin{tabular}[t]{r}3\end{tabular}}}}%
    \put(0.09187386,0.52121426){\makebox(0,0)[rt]{\lineheight{1.25}\smash{\begin{tabular}[t]{r}4\end{tabular}}}}%
    \put(0.09187386,0.22485963){\makebox(0,0)[rt]{\lineheight{1.25}\smash{\begin{tabular}[t]{r}5\end{tabular}}}}%
    \put(0,0){\includegraphics[width=\unitlength,page=2]{hungarian.pdf}}%
  \end{picture}%
\endgroup%

        }%
    &%
    \subcaptionbox{Ground truth to best prediction (Unbalanced OT with $\epsilon=0$, $\tau_1 = 0 $ and $\tau_2 \rightarrow +\infty$).%
        \label{fig:general-min2}}[.13\textwidth]%
        {\centering \footnotesize
        \def\svgwidth{.11\textwidth}
\begingroup%
  \makeatletter%
  \providecommand\color[2][]{%
    \errmessage{(Inkscape) Color is used for the text in Inkscape, but the package 'color.sty' is not loaded}%
    \renewcommand\color[2][]{}%
  }%
  \providecommand\transparent[1]{%
    \errmessage{(Inkscape) Transparency is used (non-zero) for the text in Inkscape, but the package 'transparent.sty' is not loaded}%
    \renewcommand\transparent[1]{}%
  }%
  \providecommand\rotatebox[2]{#2}%
  \newcommand*\fsize{\dimexpr\f@size pt\relax}%
  \newcommand*\lineheight[1]{\fontsize{\fsize}{#1\fsize}\selectfont}%
  \ifx\svgwidth\undefined%
    \setlength{\unitlength}{204.59437554bp}%
    \ifx\svgscale\undefined%
      \relax%
    \else%
      \setlength{\unitlength}{\unitlength * \real{\svgscale}}%
    \fi%
  \else%
    \setlength{\unitlength}{\svgwidth}%
  \fi%
  \global\let\svgwidth\undefined%
  \global\let\svgscale\undefined%
  \makeatother%
  \begin{picture}(1,1.61983314)%
    \lineheight{1}%
    \setlength\tabcolsep{0pt}%
    \put(0,0){\includegraphics[width=\unitlength,page=1]{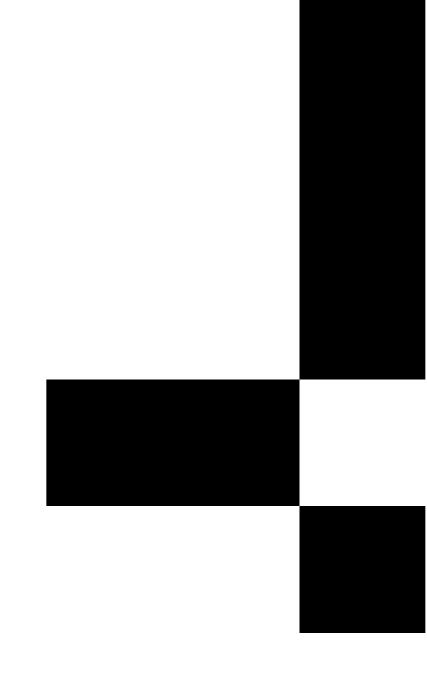}}%
    \put(0.25715823,0.00015279){\color[rgb]{0.90588235,0.29803922,0.23529412}\makebox(0,0)[t]{\lineheight{1.25}\smash{\begin{tabular}[t]{c}A\end{tabular}}}}%
    \put(0.5535129,0.00015279){\color[rgb]{0.18039216,0.8,0.44313725}\makebox(0,0)[t]{\lineheight{1.25}\smash{\begin{tabular}[t]{c}B\end{tabular}}}}%
    \put(0.78104848,0.00152741){\makebox(0,0)[lt]{\lineheight{1.25}\smash{\begin{tabular}[t]{l}$\varnothing$\end{tabular}}}}%
    \put(0.09187386,1.41027825){\makebox(0,0)[rt]{\lineheight{1.25}\smash{\begin{tabular}[t]{r}1\end{tabular}}}}%
    \put(0.09187386,1.11392358){\makebox(0,0)[rt]{\lineheight{1.25}\smash{\begin{tabular}[t]{r}2\end{tabular}}}}%
    \put(0.09187386,0.8175689){\makebox(0,0)[rt]{\lineheight{1.25}\smash{\begin{tabular}[t]{r}3\end{tabular}}}}%
    \put(0.09187386,0.52121426){\makebox(0,0)[rt]{\lineheight{1.25}\smash{\begin{tabular}[t]{r}4\end{tabular}}}}%
    \put(0.09187386,0.22485963){\makebox(0,0)[rt]{\lineheight{1.25}\smash{\begin{tabular}[t]{r}5\end{tabular}}}}%
    \put(0,0){\includegraphics[width=\unitlength,page=2]{min2.pdf}}%
  \end{picture}%
\endgroup%

        }%
    \\
    \subcaptionbox{Unbalanced OT with $\epsilon = 0.05$, $\tau_1=100$ and $\tau_2=0.01$.%
        \label{fig:general-unb1}}[.13\textwidth]%
        {\centering \footnotesize
        \def\svgwidth{.11\textwidth}
\begingroup%
  \makeatletter%
  \providecommand\color[2][]{%
    \errmessage{(Inkscape) Color is used for the text in Inkscape, but the package 'color.sty' is not loaded}%
    \renewcommand\color[2][]{}%
  }%
  \providecommand\transparent[1]{%
    \errmessage{(Inkscape) Transparency is used (non-zero) for the text in Inkscape, but the package 'transparent.sty' is not loaded}%
    \renewcommand\transparent[1]{}%
  }%
  \providecommand\rotatebox[2]{#2}%
  \newcommand*\fsize{\dimexpr\f@size pt\relax}%
  \newcommand*\lineheight[1]{\fontsize{\fsize}{#1\fsize}\selectfont}%
  \ifx\svgwidth\undefined%
    \setlength{\unitlength}{204.59437554bp}%
    \ifx\svgscale\undefined%
      \relax%
    \else%
      \setlength{\unitlength}{\unitlength * \real{\svgscale}}%
    \fi%
  \else%
    \setlength{\unitlength}{\svgwidth}%
  \fi%
  \global\let\svgwidth\undefined%
  \global\let\svgscale\undefined%
  \makeatother%
  \begin{picture}(1,1.61983314)%
    \lineheight{1}%
    \setlength\tabcolsep{0pt}%
    \put(0,0){\includegraphics[width=\unitlength,page=1]{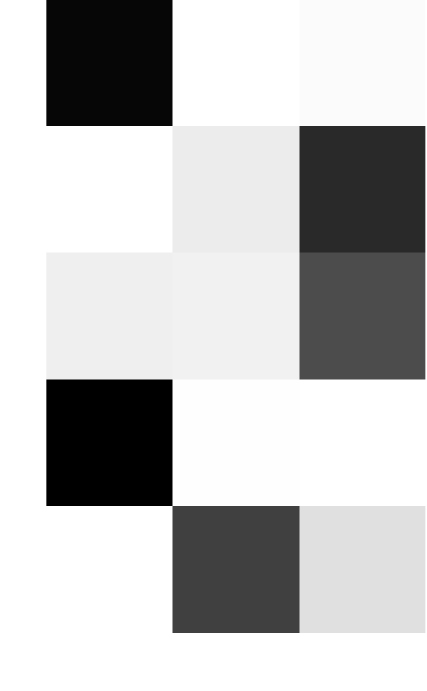}}%
    \put(0.25715823,0.00015279){\color[rgb]{0.90588235,0.29803922,0.23529412}\makebox(0,0)[t]{\lineheight{1.25}\smash{\begin{tabular}[t]{c}A\end{tabular}}}}%
    \put(0.5535129,0.00015279){\color[rgb]{0.18039216,0.8,0.44313725}\makebox(0,0)[t]{\lineheight{1.25}\smash{\begin{tabular}[t]{c}B\end{tabular}}}}%
    \put(0.78104848,0.00152741){\makebox(0,0)[lt]{\lineheight{1.25}\smash{\begin{tabular}[t]{l}$\varnothing$\end{tabular}}}}%
    \put(0.09187386,1.41027825){\makebox(0,0)[rt]{\lineheight{1.25}\smash{\begin{tabular}[t]{r}1\end{tabular}}}}%
    \put(0.09187386,1.11392358){\makebox(0,0)[rt]{\lineheight{1.25}\smash{\begin{tabular}[t]{r}2\end{tabular}}}}%
    \put(0.09187386,0.8175689){\makebox(0,0)[rt]{\lineheight{1.25}\smash{\begin{tabular}[t]{r}3\end{tabular}}}}%
    \put(0.09187386,0.52121426){\makebox(0,0)[rt]{\lineheight{1.25}\smash{\begin{tabular}[t]{r}4\end{tabular}}}}%
    \put(0.09187386,0.22485963){\makebox(0,0)[rt]{\lineheight{1.25}\smash{\begin{tabular}[t]{r}5\end{tabular}}}}%
    \put(0,0){\includegraphics[width=\unitlength,page=2]{unb1.pdf}}%
  \end{picture}%
\endgroup%

        }%
    &%
    \subcaptionbox{OT with $\epsilon=0.05$ ($\tau_1 \rightarrow +\infty$ and $\tau_2 \rightarrow +\infty$).%
        \label{fig:general-reg}}[.13\textwidth]%
        {\centering \footnotesize
        \def\svgwidth{.11\textwidth}
\begingroup%
  \makeatletter%
  \providecommand\color[2][]{%
    \errmessage{(Inkscape) Color is used for the text in Inkscape, but the package 'color.sty' is not loaded}%
    \renewcommand\color[2][]{}%
  }%
  \providecommand\transparent[1]{%
    \errmessage{(Inkscape) Transparency is used (non-zero) for the text in Inkscape, but the package 'transparent.sty' is not loaded}%
    \renewcommand\transparent[1]{}%
  }%
  \providecommand\rotatebox[2]{#2}%
  \newcommand*\fsize{\dimexpr\f@size pt\relax}%
  \newcommand*\lineheight[1]{\fontsize{\fsize}{#1\fsize}\selectfont}%
  \ifx\svgwidth\undefined%
    \setlength{\unitlength}{204.59437554bp}%
    \ifx\svgscale\undefined%
      \relax%
    \else%
      \setlength{\unitlength}{\unitlength * \real{\svgscale}}%
    \fi%
  \else%
    \setlength{\unitlength}{\svgwidth}%
  \fi%
  \global\let\svgwidth\undefined%
  \global\let\svgscale\undefined%
  \makeatother%
  \begin{picture}(1,1.61983314)%
    \lineheight{1}%
    \setlength\tabcolsep{0pt}%
    \put(0,0){\includegraphics[width=\unitlength,page=1]{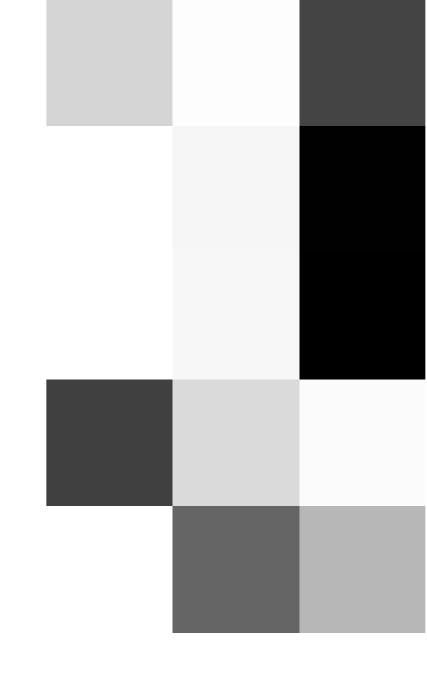}}%
    \put(0.25715823,0.00015279){\color[rgb]{0.90588235,0.29803922,0.23529412}\makebox(0,0)[t]{\lineheight{1.25}\smash{\begin{tabular}[t]{c}A\end{tabular}}}}%
    \put(0.5535129,0.00015279){\color[rgb]{0.18039216,0.8,0.44313725}\makebox(0,0)[t]{\lineheight{1.25}\smash{\begin{tabular}[t]{c}B\end{tabular}}}}%
    \put(0.78104848,0.00152741){\makebox(0,0)[lt]{\lineheight{1.25}\smash{\begin{tabular}[t]{l}$\varnothing$\end{tabular}}}}%
    \put(0.09187386,1.41027825){\makebox(0,0)[rt]{\lineheight{1.25}\smash{\begin{tabular}[t]{r}1\end{tabular}}}}%
    \put(0.09187386,1.11392358){\makebox(0,0)[rt]{\lineheight{1.25}\smash{\begin{tabular}[t]{r}2\end{tabular}}}}%
    \put(0.09187386,0.8175689){\makebox(0,0)[rt]{\lineheight{1.25}\smash{\begin{tabular}[t]{r}3\end{tabular}}}}%
    \put(0.09187386,0.52121426){\makebox(0,0)[rt]{\lineheight{1.25}\smash{\begin{tabular}[t]{r}4\end{tabular}}}}%
    \put(0.09187386,0.22485963){\makebox(0,0)[rt]{\lineheight{1.25}\smash{\begin{tabular}[t]{r}5\end{tabular}}}}%
    \put(0,0){\includegraphics[width=\unitlength,page=2]{reg.pdf}}%
  \end{picture}%
\endgroup%

        }%
    &%
    \subcaptionbox{Unbalanced OT with $\epsilon = 0.05$, $\tau_1=0.01$ and $\tau_2=100$.%
        \label{fig:general-unb2}}[.13\textwidth]%
        {\centering \footnotesize
        \def\svgwidth{.11\textwidth}
\begingroup%
  \makeatletter%
  \providecommand\color[2][]{%
    \errmessage{(Inkscape) Color is used for the text in Inkscape, but the package 'color.sty' is not loaded}%
    \renewcommand\color[2][]{}%
  }%
  \providecommand\transparent[1]{%
    \errmessage{(Inkscape) Transparency is used (non-zero) for the text in Inkscape, but the package 'transparent.sty' is not loaded}%
    \renewcommand\transparent[1]{}%
  }%
  \providecommand\rotatebox[2]{#2}%
  \newcommand*\fsize{\dimexpr\f@size pt\relax}%
  \newcommand*\lineheight[1]{\fontsize{\fsize}{#1\fsize}\selectfont}%
  \ifx\svgwidth\undefined%
    \setlength{\unitlength}{204.59437554bp}%
    \ifx\svgscale\undefined%
      \relax%
    \else%
      \setlength{\unitlength}{\unitlength * \real{\svgscale}}%
    \fi%
  \else%
    \setlength{\unitlength}{\svgwidth}%
  \fi%
  \global\let\svgwidth\undefined%
  \global\let\svgscale\undefined%
  \makeatother%
  \begin{picture}(1,1.61983314)%
    \lineheight{1}%
    \setlength\tabcolsep{0pt}%
    \put(0,0){\includegraphics[width=\unitlength,page=1]{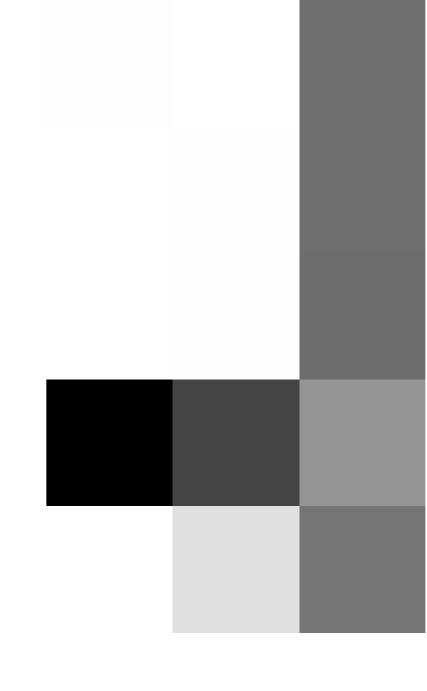}}%
    \put(0.25715823,0.00015279){\color[rgb]{0.90588235,0.29803922,0.23529412}\makebox(0,0)[t]{\lineheight{1.25}\smash{\begin{tabular}[t]{c}A\end{tabular}}}}%
    \put(0.5535129,0.00015279){\color[rgb]{0.18039216,0.8,0.44313725}\makebox(0,0)[t]{\lineheight{1.25}\smash{\begin{tabular}[t]{c}B\end{tabular}}}}%
    \put(0.78104848,0.00152741){\makebox(0,0)[lt]{\lineheight{1.25}\smash{\begin{tabular}[t]{l}$\varnothing$\end{tabular}}}}%
    \put(0.09187386,1.41027825){\makebox(0,0)[rt]{\lineheight{1.25}\smash{\begin{tabular}[t]{r}1\end{tabular}}}}%
    \put(0.09187386,1.11392358){\makebox(0,0)[rt]{\lineheight{1.25}\smash{\begin{tabular}[t]{r}2\end{tabular}}}}%
    \put(0.09187386,0.8175689){\makebox(0,0)[rt]{\lineheight{1.25}\smash{\begin{tabular}[t]{r}3\end{tabular}}}}%
    \put(0.09187386,0.52121426){\makebox(0,0)[rt]{\lineheight{1.25}\smash{\begin{tabular}[t]{r}4\end{tabular}}}}%
    \put(0.09187386,0.22485963){\makebox(0,0)[rt]{\lineheight{1.25}\smash{\begin{tabular}[t]{r}5\end{tabular}}}}%
    \put(0,0){\includegraphics[width=\unitlength,page=2]{unb2.pdf}}%
  \end{picture}%
\endgroup%

        }%
    \end{tabular}
    \caption{Different matching strategies. All are particular cases of \emph{Unbalanced Optimal Transport}. A match ($\hat{P}_{i,j}=1$) is denoted by a black square and it is white if there is no match ($\hat{P}_{i,j}=0$).}%
    \label{fig:general}
    \vspace{-2em}
\end{figure}

\section{Introduction}
Object detection models are in essence multi-task models, having to both localize objects in an image and classify them. In the context of supervised learning, each of these tasks heavily depends on a matching strategy. Indeed, determining which predicted object matches which ground truth object is a non-trivial yet essential task during the training (Figure~\ref{fig:general-img}). In particular, the matching strategy must ensure that there is ideally exactly one prediction per ground truth object, at least during inference. Various strategies have emerged, often relying on hand-crafted components. They are proposed as scattered approaches that seem to have nothing in common, at least at first glance. 

\subsection{A Unifying Framework}
To perform any match, a matching cost has to be determined. The example at Fig.~\ref{fig:general-cost} uses the \emph{Generalized Intersection over Union} ($\mathrm{GIoU}$)~\cite{giou}. Given such a cost matrix, matching strategies include:

\begin{itemize}
    \item Matching each prediction to the closest ground truth object. This often requires that the cost lies under a certain threshold~\cite{liu2016ssd,ren2015fasterrcnn,redmon2016yolo,lin2017focalloss}, to avoid matching predictions that may be totally irrelevant for the current image. The disadvantage of this strategy is its redundancy: many predictions may point towards the same ground truth object. In Fig.~\ref{fig:general-min1}, both predictions 1 and 4 are matched towards ground truth object \textcolor{red}{A}. Furthermore, some ground truth objects may be unmatched. A solution to this is to increase the number of predicted boxes drastically. This is typically the case with anchors boxes and region proposal methods.
    
    \item The opposite strategy is to match each ground truth object to the best prediction~\cite{he2015spatial,liu2016ssd}. This ensures that there is no redundancy and every ground truth object is matched. This also comes with the opposite problem: multiple ground truth objects may be matched to the same prediction. In Fig.~\ref{fig:general-min2}, both ground truth objects \textcolor{red}{A} and \textcolor{green}{B} are matched to prediction 4. This can be mitigated by having more predictions, but then many of those are left unmatched, slowing convergence~\cite{liu2016ssd}.
    
    \item A compromise is to perform a \emph{Bipartite Matching} (BM), using the \emph{Hungarian algorithm}~\cite{kuhn1955hungarian,munkres1957algorithmstransportationhungarian}, for example~\cite{carion2020detr,zhu2020deformabledetr}. The matching is one-to-one, minimizing the total cost (Definition~\ref{def:lap}). Every ground truth object is matched to a unique prediction, thus reducing the number of predictions needed, as shown in Fig.~\ref{fig:general-hungarian}. A downside is that the one-to-one matches may vary from one epoch to the next, again slowing down convergence~\cite{li2022dndetr}. This strategy is difficult to parallelize, \ie to take advantage of GPU architectures.
\end{itemize}

All of these strategies have different properties and it seems that one must choose either one or the other, optionally combining them using savant heuristics~\cite{liu2016ssd}. There is a need for a unifying framework. As we show in this paper, \emph{Unbalanced Optimal Transport}\cite{chizat2018unbalanced} offers a good candidate for this (Figure~\ref{fig:general}). It not only unifies the different strategies here above, but also allows to explore all cases in between. The cases presented in Figures~\ref{fig:general-min1}, \ref{fig:general-hungarian} and \ref{fig:general-min2} correspond to the limit cases. This opens the door for all intermediate settings. Furthermore, we show how regularizing the problem induces smoother matches, leading to faster convergence of DETR, avoiding the problem described for the BM. In addition, the particular choice of entropic regularization leads to a class of fast parallelizable algorithms on GPU known as \emph{scaling algorithms}~\cite{cuturi2013sinkhorn,chizat2018scaling}, of which we provide a compiled implementation on GPU. Our code and additional resources are publicly available\footnote{\url{https://hdeplaen.github.io/uotod}}.


\subsection{Related Work}
\paragrax{Matching Strategies}
Most \emph{two-stage} models often rely on a huge number of initial predictions, which is then progressively reduced in the region proposal stage and refined in the classification stage. Many different strategies have been proposed for the initial propositions and subsequent reductions, ranging from training no deep learning networks~\cite{girshick2014rcnn}, to only train those for the propositions~\cite{girshick2015fast, lin2017feature,he2015spatial}, to training networks for both propositions and reductions~\cite{ren2015fasterrcnn,pang2019libra,he2017maskrcnn, cai2018cascade, dai2016r}. Whenever a deep learning network is trained, each prediction is matched to the closest ground truth object provided it lies beneath a certain threshold. Moreover, the final performance of these models heavily depends on the hand-crafted anchors \cite{liu2020objectdetectionsurvey}.

Many \emph{one-stage} models rely again on predicting a large number of initial predictions or \emph{anchor boxes}, covering the entire image. As before, each anchor box is matched towards the closest ground truth object with certain threshold constraints~\cite{redmon2016yolo,lin2017focalloss}. In~\cite{liu2016ssd}, this is combined with matching each ground truth object to the closest anchor box and a specific ratio heuristic between the matched and unmatched predictions. The matching of the fixed anchors is justified to avoid a collapse of the predictions towards the same ground truth objects. 
Additionally, this only works if the number of initial predictions is sufficiently large to ensure that every ground truth object is matched by at least one prediction. Therefore, it requires further heuristics, such as \emph{Non-Maximal Suppression} (NMS) to guarantee a unique prediction per ground truth object, at least during the inference.


By using the \emph{Hungarian algorithm}, DETR \cite{carion2020detr} removed the need for a high number of initial predictions. The matched predictions are improved with a multi-task loss, and the remaining predictions are trained to predict the background  class $\varnothing$. Yet, the model converges slowly due to the instability of BM, causing inconsistent optimization goals at early training stages \cite{li2022dndetr}. Moreover, the sequential nature of the Hungarian algorithm does not take full advantage of the GPU architecture. Several subsequent works accelerate the convergence of DETR by improving the architecture of the model \cite{zhu2020deformabledetr,liu2021dabdetr} and by adding auxiliary losses \cite{li2022dndetr}, but not by exploring the matching procedure.

\paragrax{Optimal Transport}
The theory of \emph{Optimal Transport} (OT) emerges from an old problem~\cite{monge1781memoire}, relaxed by a newer formulation~\cite{kantorovich}. It gained interest in the machine learning community since the re-discovery of \emph{Sinkhorn's algorithm}~\cite{cuturi2013sinkhorn} and opened the door for improvements in a wide variety of applications ranging from graphical models~\cite{montavon2016wasserstein}, kernel methods~\cite{kolouri2016sliced,de2020wasserstein}, loss design~\cite{frogner2015learningwasserstein}, auto-encoders~\cite{tolstikhin2017wasserstein,kolouri2018sliced,rubenstein2018latent} or generative adversarial networks~\cite{arjovsky2017wasserstein, gulrajani2017improved}. 

More recent incursions in computer vision have been attempted, \eg for the matching of predicted classes~\cite{Han_2020_CVPR_Workshops}, a loss for rotated object boxes~\cite{pmlr-v139-yang21l} or a new metric for performance evaluation~\cite{otani2022optimal}.
Considering the matching of predictions to ground truth objects, recent attempts using OT bare promising results~\cite{ge2021ota,ge2021yolox}. However, when the \emph{Hungarian algorithm} is mentioned, it is systematically presented in opposition to OT~\cite{ge2021ota,vo2022review}. We lay a rigorous connection between those two approaches in computer vision. 

\emph{Unbalanced OT} has seen a much more recent theoretical development~\cite{chizat2018unbalanced,chizat-these}. The hard mass conservation constraints in the objective function are replaced by soft penalization terms. Its applications are scarcer, but we must mention here relatively recent machine learning applications in motion tracking~\cite{9152115} and domain adaptation~\cite{fatras2021unbalanced}.
\subsection{Contributions}
\begin{enumerate}
    \item We propose a unifying matching framework based on \emph{Unbalanced Optimal Transport}. It encompasses both the \emph{Hungarian algorithm}, the matching of the predictions to the closest ground truth boxes and the ground truth boxes to the closest predictions;
    \item We show that these three strategies correspond to particular limit cases and we subsequently present a much broader class of strategies with varying properties;
    \item We demonstrate how entropic regularization can speed up the convergence during training and additionally take advantage of GPU architectures;
    \item We justify the relevancy of our framework by exploring its interaction with NMS and illustrate how it is on par with the state-of-the-art.
\end{enumerate}

\subsection{Notations and Definitions}
\paragrax{Notations} Throughout the paper, we use small bold letters to denote a vector $\bm{a} \in \mathbb{R}^N$, with elements $a_i \in \mathbb{R}$. Similarly, matrices are denoted by bold capital letters such as $\bm{A} \in \mathbb{R}^{N \times M}$, with elements $A_{i,j} \in \mathbb{R}$. The notation $\bm{1}_N$ represents a column-vector of ones, of size $N$, and $\bm{1}_{N \times M}$ the matrix equivalent of size $N \times M$. The identity matrix of size $N$ is $\bm{I}_{N,N}$. With $\llbracket N \rrbracket = \{1,2,\ldots,N\}$, we denote the set of integers from $1$ to $N$. The probability simplex uses the notation $\Delta^N = \left.\left\{ \bm{u} \in \mathbb{R}_{\geq 0}^N \right| \sum_i u_i = 1\right\}$ and represents the set of discrete probability distributions of dimension $N$. This extends to the set of discrete joint probability distributions $\Delta^{N \times M}$.

\paragrax{Definitions} For each image, the set $\{\hat{\bm{y}}_i \}_{i=1}^{N_p}$ denotes the predictions and $\{\bm{y}_j\}_{j=1}^{N_g}$ the ground truth samples. Each ground truth sample combines a target class and a bounding box position: $\bm{y}_j = \left[\,\bm{c}_j,\, \bm{b}_j\,\right] \in \mathbb{R}^{N_c+4}$ where $\bm{c}_j \in \{0,1\}^{N_c}$ is the target class in one-hot encoding with $N_c$ the number of classes and $\bm{b}_j \in [0,1]^4$ defines the relative bounding box center coordinates and dimensions. The predictions are defined similarly $\hat{\bm{y}}_i = [\,\hat{\bm{c}}_i,\, \hat{\bm{b}}_i\,] \in \mathbb{R}^{N_c+4}$, but the predicted classes may be non-binary $\hat{\bm{c}}_i \in \left[0,1\right]^{N_c}$. Sometimes, predictions are defined relatively to fixed anchor boxes $\tilde{\bm{b}}_i$. 

\begin{figure*}[t]
    \centering
    \subcaptionbox{BM as a particular case of OT with no regularization ($\epsilon = 0$). The \emph{Hungarian algorithm} obtains the same solution.%
        \label{fig:params-influence-bipartite}}[.24\textwidth]%
        {\def\svgwidth{0.20\textwidth}\footnotesize
        	\graphicspath{{img/matchings/}}
        	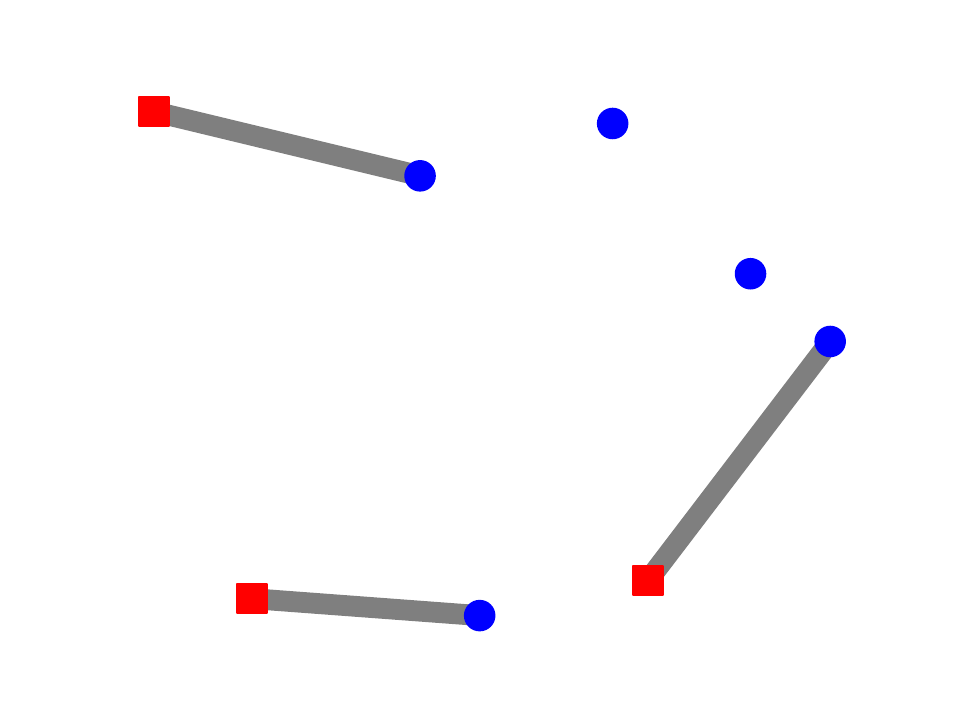 }%
    \hfill%
    \subcaptionbox{OT with regularization ($\epsilon \neq 0$). The regularization smoothens the matching allowing for multiple connections.
        \label{fig:params-influence-reg}}[.24\textwidth]%
        {\def\svgwidth{0.20\textwidth}\footnotesize
        	\graphicspath{{img/matchings/}}
        	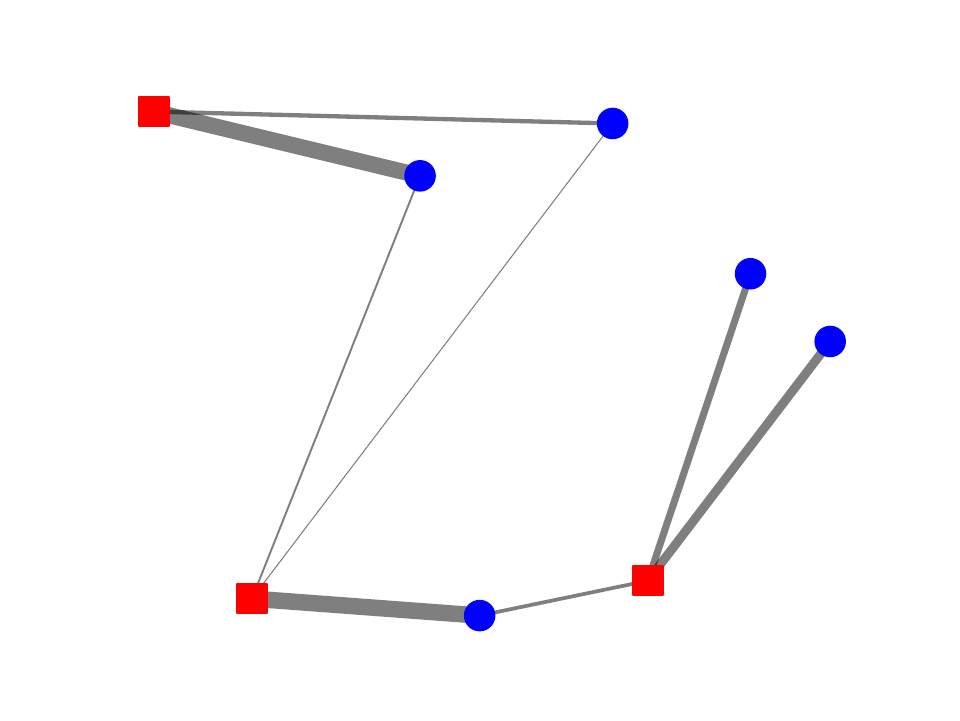 }%
    \hfill%
    \subcaptionbox{Unbalanced OT with regularization ($\epsilon \neq 0$ and $\tau_1 \ll \tau_2$). The smoothing is also visible.%
        \label{fig:params-influence-unbalanced}}[.24\textwidth]%
        {\def\svgwidth{0.20\textwidth}\footnotesize
        	\graphicspath{{img/matchings/}}
        	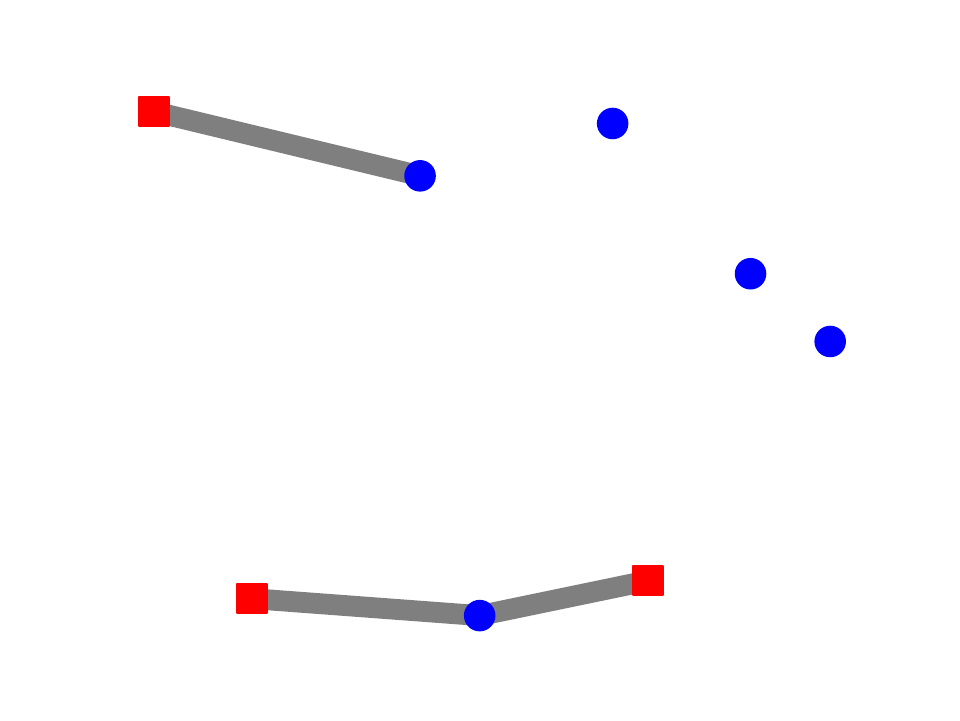 }%
    \hfill%
    \subcaptionbox{Matching each ground truth object to the closest prediction as Unbalanced OT without regularization with $\epsilon=0$, $\tau_1 = 0$ and $\tau_2 \rightarrow \infty$.%
        \label{fig:params-influence-min}}[.24\textwidth]%
        {\def\svgwidth{0.20\textwidth}\footnotesize
        	\graphicspath{{img/matchings/}}
        	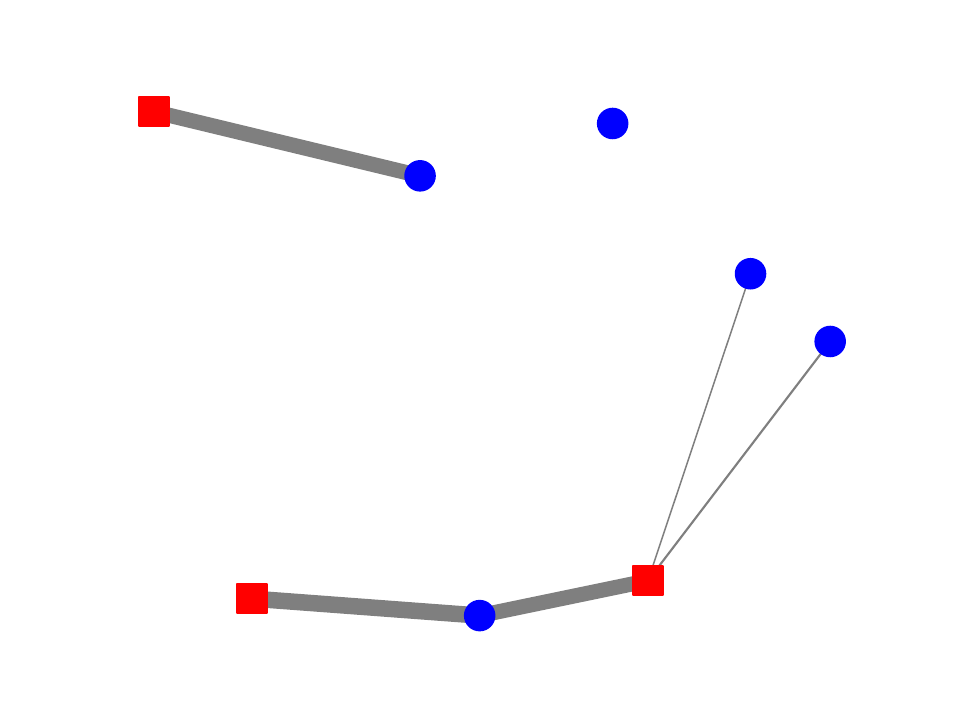 }%
    \caption{Example of the influence of the parameters. The blue dots represent predictions $\hat{\bm{y}}_i$. The red squares represent ground truth objects $\bm{y}_j$. The distributions $\bm{\alpha}$ and $\bm{\beta}$ are defined as in Prop.~\ref{prop:lap}. The thickness of the lines is proportional to the amount transported $P_{i,j}$. Only sufficiently thick lines are plotted. The dummy \emph{background} ground truth $\bm{y}_{N_g+1} = \varnothing$ is not shown, nor are the connections to it.}
    \label{fig:params-influence}
\end{figure*}

\section{Optimal Transport}
In this section, we show how \emph{Optimal Transport} and then its \emph{Unbalanced} extension unify both the \emph{Hungarian algorithm} used in DETR~\cite{carion2020detr}, and matching each prediction to the closest ground truth object used in both Faster R-CNN~\cite{ren2015fasterrcnn} and SSD~\cite{liu2016ssd}. We furthermore stress the advantages of entropic regularization, both computationally and qualitatively. This allows us to explore a new continuum of matching methods, with varying properties.

\begin{definition}[Optimal Transport]
\label{def:OT}
    Given a distribution $\bm{\alpha} \in \Delta^{N_p}$ associated to the predictions $\{\hat{\bm{y}}_i\}_{i=1}^{N_p}$, and another distribution $\bm{\beta} \in \Delta^{N_g}$ associated with the ground truth  objects $\{\bm{y}_j\}_{j=1}^{N_g}$. Let us consider a pair-wise matching cost $\mathcal{L}_{\text{match}}(\hat{\bm{y}}_i,\bm{y}_j)$ between a prediction $\hat{\bm{y}}_i$ and a ground truth object $\bm{y}_j$. We now define \emph{Optimal Transport (OT)} as finding the match $\bm{P}$ that minimizes the following problem:
    \begin{equation}
        \hat{\bm{P}} = \underset{\bm{P} \,\in\, \mathcal{U}(\bm{\alpha},\bm{\beta})}{\mathrm{arg\,min}} \left\{\sum_{i,j=1}^{N_p,N_g} P_{i,j}\mathcal{L}_{\text{match}}\left(\hat{\bm{y}}_i, \bm{y}_j \right)\right\},
    \end{equation}
    with transport polytope (admissible solutions) $\mathcal{U}(\bm{\alpha},\bm{\beta}) = \left\{ \bm{P} \in \mathbb{R}_{\geq 0}^{N_p \times N_g} : \sum_{j=1}^{N_g} P_{i,j} = \alpha_i ,\sum_{i=1}^{N_p} P_{i,j} = \beta_j\right\}.$
\end{definition}
Provided that certain conditions apply to the underlying cost $\mathcal{L}_{\text{match}}$, the minimum defines a distance between $\bm{\alpha}$ and $\bm{\beta}$, referred to as the Wasserstein distance $\mathcal{W}(\bm{\alpha},\bm{\beta})$ (for more information, we refer to monographs~\cite{villani2009optimal,santambrogio,peyre2019computational}; see also Appendix~\ref{app:theory-wass}).

\subsection{The Hungarian Algorithm}
The Hungarian algorithm solves the \emph{Bipartite Matching} (BM). We will now show how this is a particular case of Optimal Transport.

\begin{definition}[Bipartite Matching]
\label{def:lap}
Given the same objects as  in Definition~~\ref{def:OT}, the \emph{Bipartite Matching (BM)} minimizes the cost of the pairwise matches between the ground truth objects with the predictions:
\begin{equation}
    \hat{\sigma} = \mathrm{arg\, min} \left\{\sum_{j=1}^{N_g} \mathcal{L}_{\text{match}}\left(\hat{\bm{y}}_{\sigma(j)},  \bm{y}_{j} \right): \sigma \in \mathcal{P}_{N_g}(\llbracket N_p \rrbracket)\right\},
\end{equation}
where $\mathcal{P}_{N_g}(\llbracket N_p\rrbracket) = \left.\big\{\sigma \in \mathcal{P}(\llbracket N_p \rrbracket) \,\right| |\sigma| = N_g \big\}$ is the set of possible combinations of $N_g$ in $N_p$, with $\mathcal{P}(\llbracket N_p \rrbracket)$ the power set of $\llbracket N_p\rrbracket$ (the set of all subsets).
\end{definition}
BM tries to assign each ground truth $\bm{y}_j$ to a different prediction $\hat{\bm{y}}_i$ in a way to minimize the total cost. In contrast to OT, BM does not consider any underlying distributions $\bm{\alpha}$ and $\bm{\beta}$, all ground truth objects and predictions are implicitly considered to be of same mass. Furthermore, it only allows one ground truth to be matched to a unique prediction, some of these predictions being left aside and matched to nothing (which is then treated as a matching to the background $\varnothing$). The OT must match all ground truth objects to all predictions, not allowing any predictions to be left aside. However, the masses of the ground truth objects are allowed to be split between different predictions and inversely, as long as their masses correctly sum up ($\bm{P} \in \mathcal{U}(\bm{\alpha},\bm{\beta})$).

\paragrax{Particular Case of OT}
A solution for an imbalanced number of predictions compared to the number of ground truth objects would be to add dummy ground truth objects---the background $\varnothing$---to even the balance. Concretely, one could add a new ground truth $\bm{y}_{N_g+1} = \varnothing$, with the mass equal to the unmatched number of predictions. In fact, doing so directly results in performing a BM.

\begin{proposition}
\label{prop:lap}
    The Hungarian algorithm with $N_p$ predictions and $N_g \leq N_p$ ground truth objects is a particular case of OT with $\bm{P} \in \mathcal{U}(\bm{\alpha},\bm{\beta}) \subset \mathbb{R}^{N_p \times (N_g+1)}$, consisting of the predictions and the ground truth objects, with the background added $\left\{\bm{y}_j\right\}_{j=1}^{N_g+1} = \left\{\bm{y}_j\right\}_{j=1}^{N_g} \cup \left(\bm{y}_{N_g+1}=\varnothing\right)$. The chosen underlying distributions are
    \begin{eqnarray}
        \bm{\alpha} &=& \frac{1}{N_p}[\; \underbrace{1, \; 1, \;1,\; \ldots, \; 1}_{\text{$N_p$ predictions}}\;], \\
        \bm{\beta} &=& \frac{1}{N_p}[\; \underbrace{1, \; 1, \; \ldots, \; 1}_{\text{$N_g$ ground truth objects}}, \; \underbrace{(N_p-N_g)}_{\text{background }\varnothing} \;],
    \end{eqnarray}
    provided the background cost is constant:
    $\mathcal{L}_{\text{match}}\left(\hat{\bm{y}}_i,\varnothing\right) = c_{\varnothing}$. In particular for $j \in \llbracket N_g\rrbracket$, we have $\hat{\sigma}(j) = \left\{ i : P_{i,j} \neq 0 \right\}$, or equivalently $\hat{\sigma}(j) = \left\{ i : P_{i,j} = 1/N_p \right\}$.
\end{proposition}
\begin{proof}
    We refer to Appendix~\ref{app:proof-hungarian}.
\end{proof}

In other words, we can read the matching to each ground truth in the columns of $\hat{\bm{P}}$. The last columns represents all the predictions matched to the background $\hat{\sigma}(N_g+1)$. Alternatively and equivalently, we can read the matching of each prediction $i$ in the rows, the ones being matched to the background have a $\hat{P}_{i,N_g+1} = 1/N_p$.

\paragrax{Solving the Problem}
\label{computational}
Both OT and BM are linear programs. Using generic formulations would lead to a $\left(N_p+N_g+1\right) \times N_p\left(N_g+1\right)$ equality constraint matrix. It is thus better to exploit the particular bipartite structure of the problem. In particular, two families of algorithms have emerged: \emph{Dual Ascent Methods} and \emph{Auction Algorithms}~\cite{peyre2019computational}. The Hungarian algorithm is a particular case of the former and classically runs with an $\mathcal{O}\left(N_p^4\right)$ complexity~\cite{munkres1957algorithmstransportationhungarian}, further reduced to cubic by~\cite{hungarian-cubic}. Although multiple GPU implementations of a BM solver have been proposed~\cite{gpu-bipartite,gpu-hungarian,gpu-matching}, the problem remains poorly parallelizable because of its sequential nature. To allow for efficient parallelization, we must consider a slightly amended problem.

\subsection{Regularization}
We show here how we can replace the \emph{Hungarian algorithm} by a class of algorithms well-suited for parallelization, obtained by adding an entropy regularization.
\begin{definition}[OT with regularization]
\label{def:rOT}
    We consider a regularization parameter $\epsilon \in \mathbb{R}_{\geq 0}$. Extending Definition~\ref{def:OT} (OT), we define the \emph{Optimal Transport with regularization} as the following minimization problem:
    \begin{equation}
        \hat{\bm{P}} = \underset{\bm{P}\, \in\, \mathcal{U}(\bm{\alpha},\bm{\beta})}{\mathrm{arg\,min}} \left\{\sum_{i,j=1}^{N_p,N_g} P_{i,j}\mathcal{L}_{\text{match}}\left(\hat{\bm{y}}_i, \bm{y}_j\right) - \epsilon \,\mathrm{H}(\bm{P}) \right\},
    \end{equation}
    with  $\mathrm{H}: \Delta^{N \times M} \rightarrow \mathbb{R}_{\geq 0} : \bm{P} \mapsto -\sum_{i,j} P_{i,j}(\log(P_{i,j})-1)$ the entropy of the match $\bm{P}$, with $0 \ln(0) = 0$ by definition.
\end{definition}

\begin{figure}[ht]
\vspace{-2.5em}
    \centering \footnotesize
    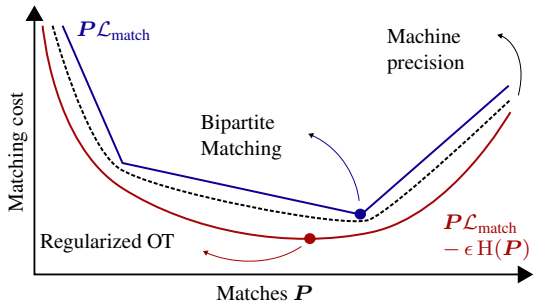
    \vspace{-.7em}
    \caption{\label{fig:ot-reg}Effect of the regularization on the minimization of the matching cost. The red line corresponds to the regularized problem ($\epsilon \neq 0$) and the blue to the unregularized one ($\epsilon = 0$).}
\end{figure}

\paragrax{Sinkhorn's Algorithm}
The entropic regularization used when finding the match $\hat{\bm{P}}$ ensures that the problem is smooth for $\epsilon \neq 0$ (see Figure~\ref{fig:ot-reg}). The advantage is that it can now be solved very efficiently using \emph{scaling algorithms} and in this particular case the algorithm of \emph{Sinkhorn}. It is particularly suited for parallelization~\cite{cuturi2013sinkhorn}, with some later speed refinements~\cite{greenkorn, screenkorn}. Reducing the regularization progressively renders the scaling algorithms numerically unstable, although some approaches have been proposed to reduce the regularization further by working in log-space~\cite{schmitzer2019stabilized,chizat2018scaling}. In the limit of $\epsilon \rightarrow 0$, we recover the exact OT (Definition~\ref{def:OT}) and the scaling algorithms cannot be used anymore. Parallelization is lost and we must resolve to use the sequential algorithms developed in Section~\ref{computational}. In brief, regularization allows to exploit GPU architectures efficiently, whereas the Hungarian algorithm and similar cannot.

\paragrax{Smoother Matches}
When no regularization is used as in the Hungarian algorithm, close predictions and ground truth objects can exchange their matches from one epoch to the other, during the training. This causes a slow convergence of DETR in the early stages of the training~\cite{li2022dndetr}. The advantage of the regularization not only lies in the existence of efficient algorithms but also allows for a reduction of sparsity. This results in a less drastic match than the Hungarian algorithm obtains. A single ground truth could be matched to multiple predictions and inversely. The proportion of these multiple matches is controlled by the regularization parameter $\epsilon$. An illustration can be found in Figures~\ref{fig:params-influence-bipartite} and \ref{fig:params-influence-reg}.

\subsection{Unbalanced Optimal Transport}
We will now show how considering soft constraints instead of hard leads to an even greater generalization of the various matching techniques used in object detection models. In particular, matching each prediction to the closest ground truth is a limit case of the \emph{Unbalanced OT}.
\begin{definition}[Unbalanced OT]
\label{def:ruOT}
We consider two constraint parameters $\tau_1, \tau_2 \in \mathbb{R}_{\geq 0}$. Extending Definition~\ref{def:rOT} (OT with regularization), we define the \emph{Unbalanced OT with regularization}~\cite{chizat2018scaling} as the following minimization problem:
\begin{equation}
\begin{aligned}
    \hat{\bm{P}} = \underset{\bm{P} \in \mathbb{R}_{\geq 0}^{N_p\times N_g}}{\mathrm{arg\, min}}\biggl\{\epsilon\,\kl{\bm{P}}{\bm{K}_{\epsilon}}\, + &\,\tau_1\kl{\bm{P}\bm{1}_{N_g}}{\bm{\alpha}} \\
    + &\,\tau_2\kl{\bm{1}_{N_p}^\top \bm{P}}{ \bm{\beta}}\biggr\},
\end{aligned}%
\end{equation}
where $\mathrm{KL}: \mathbb{R}^{N\times M}_{\geq 0} \times \mathbb{R}^{N\times M}_{> 0} \rightarrow \mathbb{R}_{\geq 0}^{\phantom{N}} : (\bm{U},\bm{V}) \mapsto \sum_{i,j=1}^{N \times M} U_{i,j} \log(U_{i,j} / V_{i,j}) - U_{i,j} + V_{i,j}$ is the \emph{Kullback-Leibler divergence} -- also called \emph{relative entropy} -- between matrices or vectors when $M=1$, with $0 \ln (0) = 0$ by definition. The Gibbs kernel $\bm{K}_{\epsilon}$ is given by $\left(K_{\epsilon}\right)_{i,j} = \exp\left(- \mathcal{L}_{\text{match}}\left(\hat{\bm{y}}_i, \bm{y}_j \right)/ \epsilon \right)$.
\end{definition}

We can see by development that the first term corresponds to the matching term $\bm{P}\mathcal{L}_{\text{match}}$ and an extension of the entropic regularization term $\mathrm{H}(\bm{P})$. The two additional terms replace the transport polytope's hard constraints $\mathcal{U}(\bm{\alpha},\bm{\beta})$ that required an exact equality of mass for both marginals. These new soft constraints allow for a more subtle sensitivity to the mass constraints as it allows to slightly diverge from them. It is clear that in the limit of $\tau_1, \tau_2 \rightarrow +\infty$, we recover the ``balanced'' problem (Definition~\ref{def:rOT}). This definition naturally also defines Unbalanced OT without regularization if $\epsilon = 0$. The matching term would remain and the entropic one disappear.

\paragrax{Matching to the Closest} Another limit case is however particularly interesting in the quest for a unifying framework of the matching strategies. If the mass constraint is to be perfectly respected for the predictions ($\tau_1 \rightarrow \infty$), but not at all for the ground truth objects ($\tau_2 = 0$), it suffices to assign the closest ground truth to each prediction. The same ground truth object could be assigned to multiple predictions and another could not be matched at all, not respecting the hard constraint for the ground truth $\bm{\beta}$. Each prediction however is exactly assigned once, perfectly respecting the mass constraint for the predictions $\bm{\alpha}$. By assigning a low enough value to the background, a prediction would be assigned to it provided all the other ground truth objects are further. In other words, the background cost would play the role of a \emph{threshold} value.

\begin{proposition}[Matching to the closest]
\label{prop:threshold}
We consider the same objects as Proposition~\ref{prop:lap}. In the limit of $\tau_1 \rightarrow \infty$ and $\tau_2 = 0$, Unbalanced OT (Definition~\ref{def:ruOT}) without regularization ($\epsilon = 0$) admits as solution each prediction being matched to the closest ground truth object unless that distance is greater than a threshold value $\mathcal{L}_{\text{match}}\left(\hat{\bm{y}}_i,\bm{y}_{N_g+1}=\varnothing\right) = c_{\varnothing}$. It is then matched to the background $\varnothing$. In particular, we have
\begin{equation}
    \hat{P}_{i,j} = \left\{
    \begin{array}{ll}
        \frac{1}{N_p} & \text{if } j = \mathrm{\arg\, min}_{j \in \llbracket N_g+1 \rrbracket}\left\{\mathcal{L}_{\text{match}}\left(\hat{\bm{y}}_i, \bm{y}_j \right)\right\},\\
        0 & \text{otherwise}.
    \end{array}\right.
\end{equation}
\end{proposition}
\begin{proof}
    We refer to Appendix~\ref{app:proof-minimum}.
\end{proof}

\begin{figure}[h]
 \vspace{-.5em}
    \centering \footnotesize
    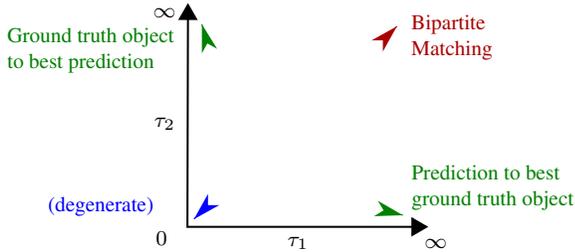
    \vspace{-1em}
    \caption{\label{fig:unbalanced-limits}Limit cases of Unbalanced OT without regularization ($\epsilon=0$).}
\end{figure}

The converse also holds. If the ground truth objects mass constraints were to be perfectly respected ($\tau_2 \rightarrow \infty$), but not the predictions ($\tau_1 \rightarrow 0$), each ground truth would then be matched to the closest prediction. The background would be matched to the remaining predictions. Some predictions could not be matched and other ones multiple times. The limits of Unbalanced OT are illustrated in Fig.~\ref{fig:unbalanced-limits}. By setting the threshold sufficiently high, we get an exact minimum, i.e., where every prediction is matched to the closest ground truth. This can be observed in Figure~\ref{fig:params-influence-min}.

\paragrax{Scaling Algorithm}
Similarly as before, adding entropic regularization ($\epsilon \neq 0$) to the \emph{Unbalanced OT} allows it to be solved efficiently on GPU with a scaling algorithm, as an extension of Sinkhorn's algorithm~\cite{chizat2018scaling, chizat-these}. The regularization still also allows for smoother matches, as shown in Figure~\ref{fig:params-influence-unbalanced}.

\paragrax{Softmax}
In the limit of $\tau_1 \rightarrow +\infty$ and $\tau_2 = 0$, the solution corresponds to a softmax over the ground truth objects for each prediction. The regularization $\varepsilon$ controls then the ``softness'' of the softmax, with $\varepsilon = 1$ corresponding to the conventional softmax and $\varepsilon \rightarrow 0$ the matching to the closest. We refer to Appendix~\ref{app:softmax} for more information.

\section{Matching}

Following previous work~\cite{carion2020detr, zhu2020deformabledetr,ren2015fasterrcnn,redmon2016yolo,liu2016ssd}, we define a multi-task matching cost between a prediction $\hat{\bm{y}}_i$ and a ground truth object $\bm{y}_j$ as the composition of a classification loss ensuring that similar object classes are matched together and a localization loss ensuring the correspondence of the positions and shapes of the matched boxes $\label{eq:matching-loss}\mathcal{L}(\hat{\bm{y}}_i, \bm{y}_j) = \mathcal{L}_{\text{classification}}(\hat{\bm{c}}_i, \bm{c}_j) + \mathcal{L}_{\text{localization}}(\hat{\bm{b}}_i, \bm{b}_j)$. %
Most models, however, do not use the same loss to determine the matches as the one used to train the model. We therefore refer to these two losses as $\mathcal{L}_{\text{match}}$ and $\mathcal{L}_{\text{train}}$. The training procedure is the following: first find a match $\hat{\bm{P}}$ given a matching strategy and matching cost $\mathcal{L}_{\text{match}}$, then compute the loss $N_p \sum_{i=1}^{N_p} \sum_{j=1}^{N_g} \hat{P}_{ij} \mathcal{L}_{\text{train}}(\hat{\bm{y}}_i, \bm{y}_j)$ where the particular training loss for the background ground truth includes only a classification term $\mathcal{L}_{\text{train}}(\hat{\bm{y}}_i, \varnothing) = \mathcal{L}_{\text{classification}}(\hat{\bm{c}}_i, \varnothing)$.

\subsection{Detection Transformer (DETR)} The object detection is performed by matching the predictions to the ground truth boxes with the \emph{Hungarian algorithm} applied to the loss $\mathcal{L}_{\text{match}}(\hat{\bm{y}}_i, \bm{y}_j) = \lambda_{\text{prob}}(1 - \left< \hat{\bm{c}}_i, \bm{c}_j \right>) + \lambda_{\ell^1} \lVert\hat{\bm{b}}_i- \bm{b}_j\rVert_1 +\lambda_{\mathrm{GIoU}}(1-\mathrm{GIoU}(\hat{\bm{b}}_i, \bm{b}_j))$ (Definition~\ref{def:lap}). To do so, the number of predictions and ground truth boxes must be of the same size. This is achieved by padding the ground truths with $(N_p - N_g)$ dummy \emph{background} $\varnothing$ objects. Essentially, this is the same as what is developed in Proposition~\ref{prop:lap}. The obtained match is then used to define an object-specific loss, where each matched prediction is pushed toward its corresponding ground truth object. The predictions that are not matched to a ground truth object are considered to be matched with the background and are pushed to predict the background class. The training loss uses the cross-entropy ($\mathrm{CE}$) for classification: $\mathcal{L}_{\text{train}}(\hat{\bm{y}}_i, \bm{y}_j) = \lambda_{\mathrm{CE}} \mathcal{L}_{\mathrm{CE}}(\hat{\bm{c}}_i, \bm{c}_j) + \lambda_{\ell^1} \lVert\hat{\bm{b}}_i- \bm{b}_j\rVert_1 +\lambda_{\mathrm{GIoU}}(1-\mathrm{GIoU}(\hat{\bm{b}}_i, \bm{b}_j))$. By directly applying Proposition~\ref{prop:lap} and adding entropic regularization (Definition~\ref{def:rOT}), we can use \emph{Sinkhorn's algorithm} and push each prediction $\hat{\bm{y}}_i$ to ground truth $\bm{y}_j$ according to weight $\hat{P}_{i,j}$. In particular, for any non-zero $\hat{P}_{i,N_g+1} \neq 0$, the prediction $\hat{\bm{y}}_i$ is pushed toward the background $\bm{y}_{N_g+1} = \varnothing$ with weight $\hat{P}_{i,N_g+1}$.

\subsection{Single Shot MultiBox Detector (SSD)} The Single Shot MultiBox Detector~\cite{liu2016ssd} uses a matching cost only comprised of the $\mathrm{IoU}$ between the fixed anchor boxes $\tilde{\bm{b}}_i$ and the ground truth boxes: $\mathcal{L}_{\text{match}}(\hat{\bm{y}}_i, \bm{y}_j) = 1-\mathrm{IoU}(\tilde{\bm{b}}_i, \bm{b}_j)$ (the $\mathrm{GIoU}$ was not published yet~\cite{giou}). Each ground truth is first matched toward the closest anchor box. Anchor boxes are then matched to a ground truth object if the matching cost is below a threshold of 0.5. In our framework, this corresponds to applying $\tau_1=0$ and $\tau_2\to \infty$ for the first phase and then $\tau_1\to \infty$ and $\tau_2=0$ with $c_{\varnothing} = 0.5$ (see Proposition~\ref{prop:threshold}). Here again, by adding entropic regularization (Definition~\ref{def:ruOT}), we can solve this using a \emph{scaling algorithm}. We furthermore can play with the parameters $\tau_1$ and $\tau_2$ to make the matching tend slightly more towards a matching done with the \emph{Hungarian algorithm} (Figure~\ref{fig:params-influence}). Again, the training uses a different loss than the matching, in particular $\mathcal{L}_{\text{train}}(\hat{\bm{y}}_i, \bm{y}_j) = \lambda_{\mathrm{CE}} \mathcal{L}_{\mathrm{CE}}(\hat{\bm{c}}_i, \bm{c}_j) + \lambda_{\text{smooth }\ell^1} \mathcal{L}_{\text{smooth }\ell^1}(\hat{\bm{b}}_i, \bm{b}_j)$.


\paragrax{Hard Negative Mining} Instead of using all negative examples $N_{\text{neg}} = (N_p - N_g)$ (predictions matched to background), the method sorts them using the highest confidence loss $\mathcal{L}_{\mathrm{CE}}(\hat{\bm{c}}_i,\varnothing)$ and picks the top ones so that the ratio between the hard negatives and positives $N_{\text{pos}} = N_g$ is at most 3 to 1.
Since $\hat{\bm{P}}$ is non-binary, we define the number of negatives and positives to be the sum of the matches to the background $N_{\text{neg}} = N_p\sum_{i=1}^{N_p} \hat{P}_{i,(N_g+1)}$ and to the ground truth objects $N_{\text{pos}} = N_p\sum_{j=1}^{N_g}\sum_{i=1}^{N_p} \hat{P}_{ij}$. We verify that for any $\bm{P} \in \mathcal{U}(\bm{\alpha},\bm{\beta})$, we have the same number of positives and negatives as the initial model: $N_{\text{neg}} = (N_p-N_g)$ and
$N_{\text{pos}} = N_g$. Hence, hard negatives are the $K$ predictions with the highest confidence loss $\hat{P}_{k,(N_g+1)} \mathcal{L}_{\mathrm{CE}}(\hat{\bm{c}}_k, \varnothing)$ such that the mass of kept negatives is at most triple the number of positives: $N_p \sum_{k=1}^K \hat{P}_{k,(N_g+1)}^{s} \leq 3 N_{\text{pos}}$, where $\hat{\bm{P}}^s$ is a permutation of transport matrix $\hat{\bm{P}}$ with rows sorted by highest confidence loss.

\section{Experimental Results \& Discussion}

We show that matching based on \textit{Unbalanced Optimal Transport} generalizes many different matching strategies and performs on par with methods that use either \textit{Bipartite Matching} or anchor boxes along with matching each prediction to the closest ground truth box with a threshold. We then analyze the influence of constraint parameter $\tau_2$ by training SSD with and without NMS for multiple parameter values. Finally, we show that OT with entropic regularization both improves the convergence and is faster to compute than the Hungarian algorithm in case of many matches.

\begin{figure}
    \centering
    \vspace{-0.5em}
    \def\svgwidth{0.45\textwidth}\footnotesize
    \graphicspath{{img/}}
    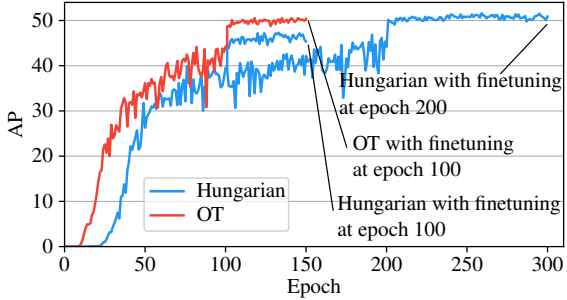
    \vspace{-1em}
    \caption{Convergence curves for DETR on the Color Boxes dataset. The model converges faster with a regularized matching.}
    \label{fig:DETRmAPtoy}
\end{figure}

\subsection{Setup}
\paragrax{Datasets} We perform experiments on a synthetic object detection dataset with 4.800 training and 960 validation images and on the large-scale COCO \cite{lin2014microsoft} dataset with 118,287 training and 5,000 validation test images. We report on mean Average Precision (AP) and mean Average Recall (AR). The two metrics are an average of the per-class metrics following COCO's official evaluation procedure. For the Color Boxes synthetic dataset, we uniformly randomly draw between 0 and 30 rectangles of 20 different colors from each image. Appendix \ref{app:color_boxes_dataset} provides the detailed generation procedure and sample images.


\paragrax{Training} For a fair comparison, the classification and localization costs for matching and training are identical to the ones used by the models. Unless stated otherwise, we train the models with their default hyper-parameter sets.  DETR and Deformable DETR are trained with hyper-parameters $\lambda_{\text{prob}}=\lambda_{\mathrm{CE}}=2$,  $\lambda_{\ell^1} = 5$ and $\lambda_{\mathrm{GIoU}}=2$. For Deformable DETR, we found the classification cost to be overwhelmed by the localization costs in the regularized minimization problem (Definition \ref{def:rOT}). We therefore set $\lambda_{\text{prob}} = 5$. We, however keep $\lambda_{\mathrm{CE}} = 2$ so that the final loss value for a given matching remains unchanged. SSD is trained with original hyper-parameters $\lambda_{\mathrm{CE}}=\lambda_{\text{smooth }\ell^1} =1$.
For OT, we set the entropic regularization to $\epsilon = \epsilon_0/(\log{(2N_p)} + 1)$ where $\epsilon_0=0.12$ for all models (App.~\ref{app:params}). In the following experiments, the Unbalanced OT is solved with multiple values of $\tau_2$ whereas $\tau_1$ is fixed to a large value $\tau_1=100$ to simulate a hard constraint. In practice, we limit the number of iterations of the scaling algorithm. This provides a good enough approximation~\cite{ge2021yolox}.

\begin{table} 
    \centering \footnotesize
    \begin{tabular}{lllllll}
        \toprule
        &\textbf{Model} & \textbf{Matching} & $\bm{\tau_2}$ &\textbf{Epochs} & \textbf{AP} & \textbf{AR}  \\
        \midrule
        \multirow{6}{*}{\hspace{-.4em}\begin{sideways}Color Boxes\end{sideways}} & DETR        & Hungarian  & ($\infty$)   & 300   & \textbf{50.9}  & \textbf{65.7} \\
        &DETR        &   Hungarian & ($\infty$) &150   & 45.3  & 60.7 \\
        &DETR        & OT          & ($\infty$) & \textbf{150}   & 50.3  & \textbf{65.7}  \\ 
        \cmidrule(lr){2-7}
        &D. DETR   & Hungarian    &($\infty$)& 50    & \textbf{64.0}  & 75.9 \\
        &D. DETR   & OT           &($\infty$)& 50    & 63.5  & \textbf{76.5} \\ 
        \midrule
        \multirow{4.5}{*}{\hspace{-.4em}\begin{sideways}COCO\end{sideways}}&D. DETR  &Hungarian & ($\infty$)& 50            & \textbf{44.5}  & \textbf{63.0} \\
        &D. DETR   & OT       &($\infty$)& 50            & 44.2  & 62.0 \\ 
        \cmidrule(lr){2-7}
        &SSD300         & Two Stage & ---- & 120            & \textbf{24.9}  & \textbf{36.8} \\
        &SSD300         & Unb. OT  & $0.01$ & 120            & 24.7    & 36.4 \\
        \bottomrule
    \end{tabular}
    \caption{Object detection metrics for different models and loss functions on the Color Boxes and COCO datasets.}
    \label{tab:metrics_stats}
\end{table}


\subsection{Unified Matching Strategy}
\paragrax{DETR and Deformable DETR} Convergence curves for DETR on the Color Boxes dataset are shown in Fig.~\ref{fig:DETRmAPtoy} and associated metrics are presented in Table~\ref{tab:metrics_stats}. DETR converges in half the number of epochs with the regularized balanced OT formulation. This confirms that one reason for slow DETR convergence is the discrete nature of BM, which is unstable, especially in the early stages of training. Training the model for more epochs with either BM or OT does not improve metrics as the model starts to overfit. Appendix \ref{app:qualitative_results_and_discussion} provides qualitative examples and a more detailed convergence analysis. 
We evaluate how these results translate to faster converging DETR-like models by additionally training Deformable DETR~\cite{zhu2020deformabledetr}.
In addition to model improvements, Deformable DETR makes three times more predictions than DETR and uses a sigmoid focal loss~\cite{lin2017focalloss} instead of a softmax cross-entropy loss for both classification costs.
Table~\ref{tab:metrics_stats} gives results on Color Boxes and COCO. We observe that the entropy term does not lead to faster convergence. Indeed, Deformable DETR converges in $50$ epochs with both matching strategies. Nevertheless, both OT and bipartite matching lead to similar AP and AR. 

\begin{table} \footnotesize
    \centering
    \begin{tabular}{llllll}
        \toprule
        \multirow{2}{*}{\textbf{Matching}}& \multirow{2}{*}{\textbf{$\bm{\tau_2}$}} & \multicolumn{2}{c}{\textbf{with NMS}} & \multicolumn{2}{c}{\textbf{w/o NMS}} \\ \cmidrule(lr){3-4} \cmidrule(lr){5-6}
         &  & \textbf{AP} & \textbf{AR} & \textbf{AP} & \textbf{AR} \\ 
        \midrule
        Two Stage & ----    & \textbf{51.6} & \textbf{67.0} & 23.2 & \textbf{77.8} \\ \midrule
        Unb. OT & 0.01  & 51.1 & 66.3  & 25.3  & 76.5  \\ 
        Unb. OT & 0.1   & 50.9 & 66.8  & 35.9  & 75.4  \\
        Unb. OT & 1     & 48.3 & 64.4  & 44.3  & 73.4   \\
        Unb. OT & 10    & 48.0 & 64.1  & 44.9  & 72.9   \\ \midrule
        OT      & ($\infty$)   & 48.1 & 64.3  & \textbf{45.2} & 73.0  \\
        \bottomrule
    \end{tabular}
    \caption{Comparison of matching strategies on the Color Boxes dataset. SSD300 is evaluated both with and without NMS.}
    \label{tab:removal_of_nms_ssd}
\end{table}

\paragrax{SSD and the Constraint Parameter} 
To better understand how unbalanced OT bridges the gap between DETR's and SSD's matching strategies, we analyze the variation in performance of SSD for different values of $\tau_2$. 
Results for an initial learning rate of 0.0005 are displayed in Table \ref{tab:removal_of_nms_ssd}. 
In the second row, the parameter value is close to zero. From Proposition \ref{prop:threshold} and when $\epsilon \to 0$, each prediction is matched to the closest ground truth box unless the matching cost exceeds 0.5. Thus, multiple predictions are matched to each ground truth box, and NMS is needed to eliminate near duplicates. When NMS is removed, AP drops by 25.8 points and AR increases by 10.2 points. 
We observe similar results for the original SSD matching strategy (1\textsuperscript{st} row), which suggests matching each ground truth box to the closest anchor box does not play a huge role in the two-stage matching procedure from SSD. The lower part of Table \ref{tab:metrics_stats} shows the same for COCO. When $\tau_2 \to +\infty$, one recovers the balanced formulation used in DETR (last row). Removing NMS leads to a 2.9 points drop for AP and a 9.7 points increase for AR. Depending on the field of application, it may be preferable to apply a matching strategy with a low $\tau_2$ and with NMS when precision is more important or without NMS when the recall is more important. Moreover, varying parameter $\tau_2$ offers more control on the matching strategy and therefore on the precision-recall trade-off \cite{buckland1994precisionrecall}.

\begin{figure}
    \centering
    \vspace{-0.5em}
    \def\svgwidth{.49\textwidth}\footnotesize
    \graphicspath{{img/}}
    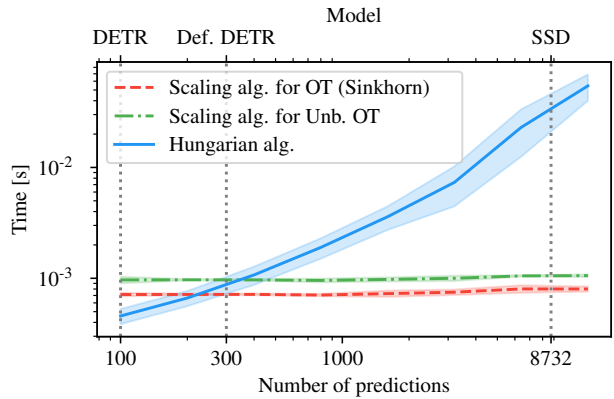
    \vspace{-0.8em}
    \caption{Average and standard deviation of the computation time for different matching strategies on COCO with batch size $16$. The Hungarian algorithm is computed with \textit{SciPy} and its time includes the transfer of the cost matrix from GPU memory to RAM. We run $20$ Sinkhorn iterations. Computed with an Nvidia TITAN X GPU and Intel Core i7-4770K CPU @ 3.50GHz.}
    \label{fig:avg_matching_time}
\end{figure}

\paragrax{Computation Time}
For a relatively small number of predictions, implementations of Sinkhorn perform on par with the Hungarian algorithm~(Fig. \ref{fig:avg_matching_time}). The ``balanced'' algorithm is on average 2.6ms slower than the Hungarian algorithm for $100$ predictions (DETR) and 1.5ms faster for 300 predictions (Deformable DETR). For more predictions, GPU parallelization of the Sinkhorn algorithm makes a large difference (more than 50x speedup). As a reference point, SSD300 and SSD512 make $8,732$ and $24,564$ predictions.

\section{Conclusion and Future Work}

Throughout the paper, we showed both theoretically and experimentally how \emph{Unbalanced Optimal Transport} unifies the \emph{Hungarian algorithm}, matching each ground truth object to the best prediction and each prediction to the best ground truth, with or without threshold. 

Experimentally, using OT and Unbalanced OT with entropic regularization is on par with the state-of-the-art for DETR, Deformable DETR and SSD. Moreover, we showed that entropic regularization lets DETR converge faster on the Color Boxes dataset and that parameter $\tau_2$ offers better control of the precision-recall trade-off. Finally, we showed that the \emph{scaling algorithms} compute large numbers of matches faster than the Hungarian algorithm.

\paragrax{Limitations and Future Work} The convergence improvement of the regularized OT formulation compared to bipartite matching seems to hold only for DETR and on small-scale datasets. Further investigations may include Wasserstein-based matching costs for a further unification of the theory and the reduction of the entropy with time, as it seems to boost convergence only in early phases, but not in fine-tuning.
\section*{Acknowledgements}
{\small EU: The research leading to these results has received funding from the European Research Council under the European Union's Horizon 2020 research and innovation program / ERC Advanced Grant E-DUALITY (787960). This paper reflects only the authors' views and the Union is not liable for any use that may be made of the contained information. Research Council KUL: Optimization frameworks for deep kernel machines C14/18/068. Flemish Government: FWO: projects: GOA4917N (Deep Restricted Kernel Machines: Methods and Foundations), PhD/Postdoc grant; This research received funding from the Flemish Government (AI Research Program). All the authors are also affiliated to Leuven.AI - KU Leuven institute for AI, B-3000, Leuven, Belgium.}

{\small
\bibliographystyle{ieee_fullname}
\bibliography{common}
}

\clearpage
\onecolumn
\appendix

\begin{center}
{\Large \bf Unbalanced Optimal Transport: A Unified Framework for Object Detection\\
Supplementary Material \par}
\end{center}
\vspace*{24pt}


\setcounter{page}{1}
\setcounter{proposition}{0}

\section{Optimal Transport Discussion}
The \emph{Otpimal Transport} formulation presented throughout the paper is formulated in discrete space. In this section, we present the more general formulation of which the discrete one is a particular case of. We also discuss the Wasserstein distance in the particular context of object detection and the effect of regularization on the uniqueness of the solutions. Only the case of the original OT formulation---or ``balanced'' case---is covered here.

\subsection{Continuous Formulation}

More generally, we define Optimal Transport in its continuous form.
\begin{definition}[Continuous Optimal Transport]
\label{def:cotp}
    Given two distributions $\alpha \in \mathscr{P}_+(X)$ and $\beta \in \mathscr{P}_+(Y)$ of same mass $\int \alpha \,\mathrm{d} x = \int \beta \,\mathrm{d}y$, and given an underlying cost function $c : X \times Y \rightarrow [0, +\infty]$, we define \emph{Continuous Optimal Transport} as the minimization of a transport cost
    \begin{equation}
        \inf\left\{ \int_{X \times Y} c \,\mathrm{d}\gamma : \gamma \in U(\alpha,\beta)\right\},
    \end{equation}
    with admissible solutions, here called transport plans
    \begin{equation}
         U(\alpha, \beta) = \bigg\{ \gamma \in \mathscr{P}_+(X \times Y) :\int_Y \mathrm{d}\gamma = \alpha \quad \text{and} \quad \int_X \mathrm{d}\gamma = \beta\bigg\}.
    \end{equation}
    If a minimum exists, it is called the optimal transport plan $\hat{\gamma}$.
\end{definition}

We replace the probability simplex $\Delta^N$ by the space on probability distributions $\mathscr{P}_+(X)$ on $X$. The transport plans are the set of joint probability distribution $\gamma \in \mathscr{P}_+(X \times Y)$, whose marginal distributions are $\alpha$ and $\beta$. The discrete formulation (Definition~\ref{def:OT}) is a particular case where $\alpha = \sum_i \alpha_i \delta_{\hat{\bm{y}}_i}$, $\beta = \sum_j \beta_j \delta_{\bm{y}_j}$ and the cost $c = \mathcal{L}_{\text{match}}$. In this case, a minimum always exists.

\subsection{Wasserstein Distance}
\label{app:theory-wass}
This infimum defines a distance between $\alpha$ and $\beta$, called the Wasserstein distance $\mathcal{W}_p(\alpha, \beta)$, provided that the underlying cost function is also a distance $c = d^p$ up to some exponent $p \in [1, +\infty[$. In our case, $\mathcal{L}_{\text{match}}$ is not a distance. More formally, sum of distances are distances. The $\ell^1$ norm is a distance, and $1-\mathrm{IoU}$, or $1-\mathrm{GIoU}$ also are~\cite{giou}. However, the cross entropy or the focal loss do not satisfy the triangular inequality or the symmetry properties. In consequence, we cannot talk about a Wasserstein distance here.

Furthermore, interpreting a Wasserstein distance $\mathcal{W}_{p} (\alpha, \beta)$ would not make much sense even if the underlying matching cost was to be a distance. Indeed, the distributions $\alpha$ and $\beta$ would be the same at every iteration in our framework. In other words, the distance would always be computed between the same points, but the underlying cost would change and it would be different for each image. Each iteration would be computing the distance of two same points in a changing geometry and each image would have its own evolving geometry.

For completeness, we must mention that the regularized version does not define a distance as $\mathcal{W}_{p,\text{reg.}}(\bm{\alpha},\bm{\alpha}) = -\epsilon\,\mathrm{H}(\bm{I}_{N_p,N_p} / N_p)> 0$ with $\bm{I}_{N_{p}, N_{p}}$ the identity matrix of size $N_p$ (we refer to~\cite{genevay, genevay-these, sinkhorn-divergences} for a broader discussion on the subject).

\subsection{Uniqueness}
We consider here the discrete formulation used throughout the paper. By classical linear programming theory, the non-regularized problem admits a non-unique solution if and only if multiple extreme points minimize the problem. In that case, the set of minimizers is all the linear interpolations between those extreme points. The regularization term however is $\epsilon$-strongly convex; the regularized problem thus always has a unique solution~\cite{peyre2019computational}.

\section{Proofs of the Propositions}
\label{app:proofs}
In this section, we provide the proofs of Propositions~\ref{prop:lap} and \ref{prop:threshold} and enrich them with some insight through a few additional results.

\subsection{Hungarian Algorithm}
Before providing a proof of the particular equivalence between OT and BM, we first consider a more general result.
\begin{lemma}
\label{prop:common-measure}
    We consider the rational probability simplex $\Delta_{\mathbb{Q}}^N = \{\bm{u} \in \mathbb{Q}^N_{\geq 0}| \sum_i u_i = 1\}$. Given an OT problem (Definition~\ref{def:OT}) with underlying distributions $\bm{\alpha} \in \Delta^N_{\mathbb{Q}}$ and $\bm{\beta} \in \Delta^M_{\mathbb{Q}}$. Each extreme point of $\,\mathcal{U}(\bm{\alpha},\bm{\beta})$ is comprised of elements, which are multiples of the \emph{common measure} of $\bm{\alpha}$ and $\bm{\beta}$:
    \begin{equation}
        \text{$\bm{P}$ is an extreme point of $\,\mathcal{U}(\bm{\alpha},\bm{\beta})$} \qquad \Longrightarrow \qquad \bm{P} \in \mathrm{CM}(\bm{\alpha},\bm{\beta}) \cdot \mathbb{N}_{\geq 0}^{N \times M},
    \end{equation}
    where the \emph{common measure} is the greatest rational such that all non-zero elements of both distributions are multiples of it:
    \begin{equation}
        \mathrm{CM}(\bm{\alpha},\bm{\beta}) = \frac{\mathrm{GCD}\left(\mathrm{LCM}\left(\left[\,\bm{\alpha},\,\bm{\beta}\,\right]\right)/\left[\,\bm{\alpha},\,\bm{\beta}\,\right]\right)}{\mathrm{LCM}\left(\left[\,\bm{\alpha},\,\bm{\beta}\,\right]\right)} \in \mathbb{Q}_{> 0},
    \end{equation}
    with $\mathrm{GCD}: \mathbb{N}_{> 0}^{N} \rightarrow \mathbb{N}_{> 0}$ the greatest common divisor and $\mathrm{LCM}: \mathbb{N}_{> 0}^{N} \rightarrow \mathbb{N}_{> 0}$ the lowest common multiple.

The \emph{common measure} extends the $\mathrm{GCD}$ to non-integers. As an example $\mathrm{CM}(\left[\,\sfrac{2}{3},\,\sfrac{4}{5}\,\right])=\sfrac{2}{15}$ and $\mathrm{CM}(\left[\,\sfrac{2}{3},\,\sfrac{5}{6},\,\sfrac{4}{7}\,\right])=\sfrac{1}{42}$.
\end{lemma}
\begin{proof}
    In~\cite{brualdi_2006}, Corollary 8.1.3, an algorithm is given to build the exhaustive list of extreme points. It comprises only minimum and subtraction operations, which leave the common measure unchanged.
\end{proof}

\begin{corollary}
\label{prop:extreme-01}
    Given the underlying distributions as in Proposition~\ref{prop:lap}, the extreme points of $\,\mathcal{U}(\bm{\alpha},\bm{\beta})$ are comprised only of zeros and $1/N_p$:
    \begin{equation}
        \text{$\bm{P}$ is an extreme point of $\,\mathcal{U}(\bm{\alpha},\bm{\beta})$} \qquad \Longrightarrow \qquad \bm{P} \in \left\{0,1/N_p\right\}^{N_p \times \left(N_g+1\right)}.
    \end{equation}
\end{corollary}
This is a direct consequence of Lemma~\ref{prop:common-measure} and the mass constraints directly implying that $P_i \leq 1/N_p$ for all $i$. In this particular case, there is also an equivalence.
\begin{lemma}
\label{prop:extreme-01-equiv}
    Given the underlying distributions as in Proposition~\ref{prop:lap}, the extreme points of $\,\mathcal{U}(\bm{\alpha},\bm{\beta})$ are comprised only of zeros and $1/N_p$:
    \begin{equation}
        \text{$\bm{P}$ is an extreme point of $\,\mathcal{U}(\bm{\alpha},\bm{\beta})$} \qquad \Longleftrightarrow \qquad \bm{P} \in \left\{0,1/N_p\right\}^{N_p \times \left(N_g+1\right)} \quad \text{and} \quad \bm{P} \in \,\mathcal{U}(\bm{\alpha},\bm{\beta}).
    \end{equation}
\end{lemma}
\begin{proof}
    We consider Corollary~\ref{prop:extreme-01} and add the fact that such a match $\bm{P} \in \left\{0,1/N_p\right\}^{N_p \times \left(N_g+1\right)}$ only has one element per row (or prediction if we prefer) to satisfy the mass constraints. Therefore, it cannot be any interpolation of two other extreme points.
\end{proof}

We however also give a more direct proof, based essentially on the same arguments.

\begin{proof}
    We will first show that the elements of the match $\bm{P}$ corresponding to any extreme point, can only be $1/N_p$ or $0$. Therefore we can consider the associated bipartite graph of the problem: each prediction consists in a node $i$ and each ground truth a node $j$. Each non-zero value entry of $\bm{P}$ connects nodes $i$ and $j$ with weight $P_{i,j}$. The solution is admissible if and only if the weight of each node $i$ equals $\alpha_i$ and $j$ equals $\beta_j$. A transport plan $\bm{P}$ is an extreme point if and only if the corresponding bipartite graph only consists in trees, or equivalently, it has no cycle (Theorem 8.1.2 of~\cite{brualdi_2006}).

    Because the mass constraint must all sum up to one for the predictions, we already know that $P_{i,j} \leq 1/N_p$. We will now proceed \emph{ad absurdum} and suppose that there were to be an entry $0 < P_{i,j} < 1/N_p$ connecting a prediction and a ground truth. In order to satisfy the mass constraints, they would both also have to be connected to another prediction and another ground truth. Similarly, these would also have to be connected to at least one prediction and one ground truth, and so on. They would all form a same graph, or be ``linked'' together in other words. By consequence, each new connection must be done to yet ``unlinked'' prediction and ground truth to avoid the formation of a cycle. Considering that there are $N_p$ predictions, there would be at the end at least $2N_p$ edges within the graph. This is incompatible with the fact that there cannot be any cycle (Corollary 8.1.3 of~\cite{brualdi_2006}). By consequence, the entries of $\bm{P}$ must be either $0$ or $1/N_p$.
\end{proof}

We can now proceed to prove the said proposition.

\label{app:proof-hungarian}
\begin{proposition}
    The Hungarian algorithm with $N_p$ predictions and $N_g \leq N_p$ ground truth objects is a particular case of OT with $\bm{P} \in \mathcal{U}(\bm{\alpha},\bm{\beta}) \subset \mathbb{R}^{N_p \times (N_g+1)}$, consisting of the predictions and the ground truth objects, with the background added $\left\{\bm{y}_j\right\}_{j=1}^{N_g+1} = \left\{\bm{y}_j\right\}_{j=1}^{N_g} \cup \left(\bm{y}_{N_g+1}=\varnothing\right)$. The chosen underlying distributions are
    \begin{eqnarray}
        \bm{\alpha} &=& \frac{1}{N_p}[\; \underbrace{1, \; 1, \;1,\; \ldots, \; 1}_{\text{$N_p$ predictions}}\;], \\
        \bm{\beta} &=& \frac{1}{N_p}[\; \underbrace{1, \; 1, \; \ldots, \; 1}_{\text{$N_g$ ground truth objects}}, \; \underbrace{(N_p-N_g)}_{\text{background }\varnothing} \;],
    \end{eqnarray}
    provided the background cost is constant:
    $\mathcal{L}_{\text{match}}\left(\hat{\bm{y}}_i,\varnothing\right) = c_{\varnothing}$. In particular for $j \in \llbracket N_g\rrbracket$, we have $\hat{\sigma}(j) = \left\{ i : P_{i,j} \neq 0 \right\}$, or equivalently $\hat{\sigma}(j) = \left\{ i : P_{i,j} = 1/N_p \right\}$.
\end{proposition}
\begin{proof}
    We will demonstrate that OT with $\bm{\alpha} = \frac{1}{N_p}[\; 1, \; 1, \;1,\; \ldots, \; 1\;]$ and $\bm{\beta} = \frac{1}{N_p}[\; 1, \; 1, \; \ldots, \; 1, \;(N_p-N_g) \;]$ and constant background cost necessarily has the BM as minimal solution. We first observe that because of the linear nature of the problem, there is at least one extreme point that minimizes the total cost. By directly applying Lemma~\ref{prop:extreme-01-equiv}, there must be exactly one match per prediction and exactly one match for each non-background ground truth to satisfy the mass constraints. The added background ground truth has $N_p - N_g$ matches. This is equivalent to saying that disregarding the background ground truth, we wave   $\sigma \in \mathcal{P}_{N_g}(\llbracket N_p\rrbracket)$ with $\hat{\sigma}(j) = \left\{ i : P_{i,j} = 1/N_p \right\}$. The proof is concluded by observing that the part of the background in the total transport cost is equal to $\frac{1}{N_p}\left(N_p - N_g\right)c_{\varnothing}$ and is constant, hence not influencing the minimum.
\end{proof}

\subsection{Minimum Matching with Threshold}
\label{app:proof-minimum}

\begin{proposition}[Matching to the closest]
We consider the same objects as Proposition~\ref{prop:lap}. In the limit of $\tau_1 \rightarrow \infty$ and $\tau_2 = 0$, Unbalanced OT (Definition~\ref{def:ruOT}) without regularization ($\epsilon = 0$) admits as solution each prediction being matched to the closest ground truth object unless that distance is greater than a threshold value $\mathcal{L}_{\text{match}}\left(\hat{\bm{y}}_i,\bm{y}_{N_g+1}=\varnothing\right) = c_{\varnothing}$. It is then matched to the background $\varnothing$. In particular, we have
\begin{equation}
\label{eq:prob-min}
    \hat{P}_{i,j} = \left\{
    \begin{array}{ll}
        \frac{1}{N_p} & \text{if } j = \mathrm{\arg\, min}_{j \in \llbracket N_g+1 \rrbracket}\left\{\mathcal{L}_{\text{match}}\left(\hat{\bm{y}}_i, \bm{y}_j \right)\right\},\\
        0 & \text{otherwise}.
    \end{array}\right.
\end{equation}
\end{proposition}
\begin{proof}
    By taking the limit of $\tau_1 \rightarrow +\infty$ and setting $\epsilon, \tau_2 = 0$, the problem becomes
    \begin{equation}
    \begin{array}{ll}
        \mathrm{arg\,min} & \left\{\left.\sum_{i,j=1}^{N_p,N_g+1} P_{i,j}\mathcal{L}_{\text{match}}\left(\bm{y}_i, \hat{\bm{y}}_j \right)\right| \bm{P} \in \mathbb{R}_{\geq 0}^{N_p \times \left( N_g+1\right)}\right\}, \\
        \text{s.t.} & \sum_j P_{i,j} = 1/N_p \qquad \forall i.
    \end{array}
    \end{equation}
    We can now see that the choice made in each row is independent from the other rows. In other words, each ground truth object can be matched independently of the others. The minimization is then obtained if, for each prediction (or row), all the weight is put on the ground truth object with minimum cost, including the background. This leads to~\cref{eq:prob-min}.
\end{proof}

\begin{corollary}[Matching to the closest without threshold]
Provided the background cost is more expensive than any other cost $c_{\varnothing} > \max\left.\big\{\mathcal{L}_{\text{match}}\left(\hat{\bm{y}}_i,\bm{y}_j\right)\right| i \in \llbracket N_p\rrbracket\, \text{and}\, j \in \llbracket N_g \rrbracket\big\}$, %
each prediction will always be matched to the closest ground truth.
\end{corollary}

In theory, this a much too strong condition, the background cost can just be greater than the minimum cost for each prediction $\mathcal{L}_{\text{match}}\left(\hat{\bm{y}}_i,\varnothing\right) > \min_{j} \left\{\mathcal{L}_{\text{match}}\left(\hat{\bm{y}}_i,\bm{y}_j\right)\right\}$. In practice, however, this does not change much. It suffices to set the background cost high enough and we are assured to get a minimum. One could also imagine a different background cost for each prediction in order to have a more granular threshold.

\section{Scaling Algorithms}
We present here the two scaling algorithms: Sinkhorn's algorithm for ```balanced'' \emph{Optimal Transport} and its variant for \emph{Unbalanced Optimal Transport}. We further show how it is connected to the softmax.

\subsection{Sinkhorn and Variant}
These two algorithms are taken from~\cite{peyre2019computational,chizat2018scaling}. In particular we can see how taking $\tau_1 \rightarrow +\infty$ and $\tau_2 \rightarrow +\infty$ in Algorithm~\ref{alg:unbalanced} leads to Algorithm~\ref{alg:sinkhorn}. Indeed, we have $\lim_{\tau \rightarrow +\infty} \frac{\tau}{\tau + \epsilon} = 1$. By $\oslash$, we denote the element-wise (or Hadamard) division.

\begin{algorithm}[H]
\caption{Sinkhorn's algorithm for ``balanced'' \emph{Optimal Transport} with regularization.}\label{alg:sinkhorn}
\DontPrintSemicolon
\KwData{Distributions $\bm{\alpha}\in \Delta^{N_p}$ and $\bm{\beta} \in \Delta^{N_g+1}$, regularization parameter $\epsilon \in \mathbb{R}_{>0}$ and cost matrix $\bm{C}=\left[\mathcal{L}_{\text{match}}\left(\hat{\bm{y}}_i, \bm{y}_j \right)\right]_{i,j=1}^{N_p,N_g+1} \in \mathbb{R}_{\geq 0}^{N_p \times N_g +1}$ (including background $\bm{y}_{N_g+1} = \varnothing$).}
\KwResult{Match $\hat{\bm{P}} \in \Delta^{N_p,N_g+1}$.}
\Begin{
    $\bm{K}_{\epsilon} \longleftarrow \exp\left(-\bm{C}/\epsilon\right)$ \tcc*{Gramm matrix (element-wise)}
    $\bm{u} \longleftarrow \bm{1}_{N_p} / N_p$ \tcc*{Dual variable associated with $\bm{\alpha}$}
    $\bm{v} \longleftarrow \bm{1}_{N_g+1} / \left( N_g+1 \right)$\tcc*{Dual variable associated with $\bm{\beta}$} 
    \Repeat{convergence}{
        $\bm{u} \longleftarrow \bm{\alpha} \oslash \left(\bm{K}_{\epsilon}\bm{v}\right)$ \tcc*{Scaling iteration for $\bm{u}$}
        $\bm{v} \longleftarrow \bm{\beta} \oslash \left(\bm{K}_{\epsilon}^{\top}\bm{u}\right)$ \tcc*{Scaling iteration for $\bm{v}$} 
    } $\hat{\bm{P}} \longleftarrow \bm{u}\bm{K}_{\epsilon}\bm{v}$  
    }
\end{algorithm}

\begin{algorithm}[H]
\caption{Scaling algorithm for \emph{Unbalanced Optimal Transport} with regularization.}\label{alg:unbalanced}

\DontPrintSemicolon
\KwData{Distributions $\bm{\alpha}\in \Delta^{N_p}$ and $\bm{\beta} \in \Delta^{N_g+1}$, regularization parameter $\epsilon \in \mathbb{R}_{>0}$, constraint parameters $\tau_1, \tau_2 \in \mathbb{R}_{\geq 0}$ and cost matrix $\bm{C}=\left[\mathcal{L}_{\text{match}}\left(\hat{\bm{y}}_i, \bm{y}_j \right)\right]_{i,j=1}^{N_p,N_g+1} \in \mathbb{R}_{\geq 0}^{N_p \times N_g +1}$ (including background $\bm{y}_{N_g+1} = \varnothing$).}
\KwResult{Match $\hat{\bm{P}} \in \mathbb{R}_{\geq 0}^{N_p,N_g+1}$.}
\Begin{
    $\bm{K}_{\epsilon} \longleftarrow \exp\left(-\bm{C}/\epsilon\right)$ \tcc*{Gramm matrix (element-wise)} 
    $\bm{u} \longleftarrow \bm{1}_{N_p} / N_p$ \tcc*{Dual variable associated with $\bm{\alpha}$} 
    $\bm{v} \longleftarrow \bm{1}_{N_g+1} / \left( N_g+1 \right)$\tcc*{Dual variable associated with $\bm{\beta}$} 
    \Repeat{convergence}{
        $\bm{u} \longleftarrow \big(\bm{\alpha} \oslash (\bm{K}_{\epsilon}\bm{v})\big)^{\frac{\tau_1}{\tau_1 + \epsilon}}$ \tcc*{Scaling iteration for $\bm{u}$} 
        $\bm{v} \longleftarrow \big(\bm{\beta} \oslash (\bm{K}_{\epsilon}^{\top}\bm{u})\big)^{\frac{\tau_2}{\tau_2 + \epsilon}}$ \tcc*{Scaling iteration for $\bm{v}$} 
    } $\hat{\bm{P}} \longleftarrow \bm{u}\bm{K}_{\epsilon}\bm{v}$  
    }
\end{algorithm}

\subsection{Connection with the Softmax}
\label{app:softmax}
In this section, we lay a connection between the softmax and the solutions of the scaling algorithms, in particular considering its first iterations. We consider more precisely the softmin, which is the opposite of the softmax: $\left(\mathrm{softmin}(\bm{v})\right)_i = \left(\mathrm{softmax}(-\bm{v})\right)_i = \exp(-v_i) / \sum_{j=1}^N \exp(-v_j)$, for any vector $\bm{v} \in \mathbb{R}^N$. Considering a softmin over $\LM$ is thus the same as considering the softmax over $-\LM$, as in~\cite{loftr}. By simplicity, we will only use the softmax terminology.

\subsubsection{Without Background}
We first consider the case without background, where the underlying distributions are equal to $\bm{\alpha} = \bm{1}_{N_p} / N_p$ and $\bm{\beta} = \bm{1}_{N_g} / N_g$. This does not correspond to the setup of Prop.~\ref{prop:lap} and only approximates a one-to-one match if $N_p = N_g$.
\begin{proposition}
\label{prop:unb-softmin}
Consider the two uniform distributions $\bm{\alpha} = \bm{1}_{N_p} / N_p$ and $\bm{\beta} = \bm{1}_{N_g} / N_g$ with cost $\LM$. The solution of the Unbalanced OT scaling algorithm with regularization $\varepsilon=1$, $\tau_1=0$ and $\tau_2 \rightarrow +\infty$ is proportional to performing a softmax over the predictions, for each ground truth object. In particular, we have
    \begin{equation}
        \hat{P}_{i,j} = \frac{\exp\left(-\LM\right)}{N_g\sum_{i=1}^{N_p} \exp\left(-\LM\right)}.
    \end{equation}
\end{proposition}
\begin{proof}
    We consider the first scaling iteration from Alg.~\ref{alg:unbalanced}. We first observe that the exponents lead to $\lim_{\tau_1 \rightarrow 0} \frac{\tau_1}{\tau_1+\varepsilon} = 0$ and $\lim_{\tau_2 \rightarrow +\infty} \frac{\tau_2}{\tau_2+\varepsilon} = 1$. Starting with $\bm{v}_0 = \bm{1}_{N_g} / N_g$, we obtain the new
    \begin{align}
    \bm{u}_1 = \left(\bm{\alpha} \oslash \left(\bm{K}_{\varepsilon} \bm{v}_0\right)\right)^{0} = \bm{1}_N, &\qquad\text{or } \left(\bm{u}_1\right)_i = 1,  \\
    \bm{v}_1 = \left(\bm{\beta} \oslash \left(\bm{K}_{\varepsilon}^\top\bm{u}_1\right)\right)^{1} =\left( \bm{1}_{N_g}/N_g \right) \oslash \left(\bm{K}_{\varepsilon}^\top \bm{1}_N\right), &\qquad\text{or } \left( \bm{v}_1\right)_j = \frac{1}{N_g\sum_{i=1}^{N_p} \exp\left(-\LM\right)}.
    \end{align}
    We observe that $\bm{u}_2 = \bm{u}_1$ and $\bm{v}_2 = \bm{v}_1$ and conclude that the algorithm converges after only one iteration. Computing the match $\hat{\bm{P}} = \bm{u}\bm{K}_{\epsilon}\bm{v}$ leads to the softmax.
\end{proof}

The exact opposite happens if we consider $\tau_1 \rightarrow +\infty$ and $\tau_2 = 0$ instead: the softmax is taken over the ground truth objects for each prediction. The proof is the same, just inverting $\bm{u}$ and $\bm{v}$ and obtaining factor $1 / N_p$ instead. This can be observed at Fig.~\ref{fig:limits-reg}.

\begin{figure}
    \centering
    \begin{subfigure}{0.45\textwidth}
        \footnotesize \centering 
        \input{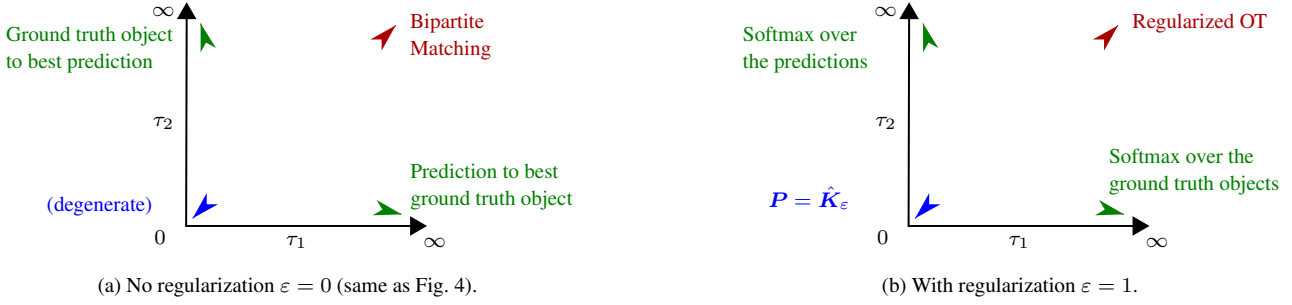}
        \caption{No regularization $\varepsilon=0$ (same as Fig.~\ref{fig:unbalanced-limits}).%
        \label{fig:limits-noreg}}
    \end{subfigure}
    \hfill
    \begin{subfigure}{0.45\textwidth}
        \footnotesize \centering 
        
        \caption{With regularization $\varepsilon=1$.%
        \label{fig:limits-reg}}
    \end{subfigure}
    \caption{Comparison of the different limit cases of \emph{Unbalanced Optimal Transport}, with and without regularization.}
    \label{fig:limits}
\end{figure}

If we would like to exactly obtain the softmax without the factor $1/N_g$  (or $1/N_p$), we could consider only one iteration starting with both initial dual variables $\bm{u}_0$ and $\bm{v}_0$. It would however not be the optimal match $\hat{\bm{P}}$ and will converge to the same solution as in Prop.~\ref{prop:unb-softmin} after the second---and last---iteration. Nevertheless, starting from both initial dual variables is more interesting in the ``balanced'' case.

\begin{proposition}
\label{prop:bal-softmin}
Consider the two uniform distributions $\bm{\alpha} = \bm{1}_{N_p} / N_p$ and $\bm{\beta} = \bm{1}_{N_g} / N_g$ with cost $\LM$. Starting from both initial dual variables, one iteration of the ``balanced'' OT scaling algorithm with regularization $\varepsilon=1$ is equal to
    \begin{equation}
        \label{eq:dual-sm-bis}
        P_{i,j} = \frac{\exp\left(-\LM\right)}{\sum_{i=1}^{N_p} \exp\left(-\LM\right) \cdot \sum_{j=1}^{N_g} \exp\left(-\LM\right)}.
    \end{equation}
\end{proposition}
\begin{proof}
    We consider the first scaling iteration from Alg.~\ref{alg:sinkhorn} with $\bm{\alpha} = \bm{1}_{N_p} / N_p$ and $\bm{\beta} = \bm{1}_{N_g}/N_g$. Starting with $\bm{u}_0 = \bm{1}_{N_p} / N_p$ and $\bm{v}_0 = \bm{1}_{N_g}/N_g$, we obtain the new
    \begin{align}
    \bm{u}_1 = \bm{\alpha} \oslash \left(\bm{K}_{\varepsilon}\bm{v}_0\right) = \left( \bm{1}_{N_p} / N_p \right) \oslash \left(\bm{K}_{\varepsilon} \left(\bm{1}_{N_g} / N_g \right) \right), &\qquad\text{or } \left( \bm{u}_1\right)_i = \frac{N_g}{N_p\sum_{j=1}^{N_g} \exp\left(-\LM\right)},  \\
    \bm{v}_1 = \bm{\beta} \oslash \left(\bm{K}_{\varepsilon}^\top\bm{u}_0\right) = \left(\bm{1}_{N_g} / N_g\right) \oslash \left(\bm{K}_{\varepsilon}^\top \left( \bm{1}_{N_p} / N_p \right) \right), &\qquad\text{or } \left( \bm{v}_1\right)_j = \frac{N_p}{N_g\sum_{i=1}^{N_p} \exp\left(-\LM\right)}.
    \end{align}
    Computing the match $\bm{P} = \bm{u}\bm{K}_{\epsilon}\bm{v}$ leads to the Eq.~\ref{eq:dual-sm-bis}. This is not the optimal match $\hat{\bm{P}}$ as the algorithm did not converge yet.
\end{proof}

The \emph{dual-softmax} considered in~\cite{loftr} is essentially the same as Prop.~\ref{prop:bal-softmin}, with the difference of a factor $2$ in the numerator's exponential:
\begin{eqnarray}
    P_{i,j} &=& \mathrm{softmax}\left( \left[ -\Lm{i}{k} \right]_{k=1}^{N_g} \right)_j \cdot \,\mathrm{softmax}\left( \left[ -\Lm{l}{j} \right]_{j=1}^{N_p} \right)_i, \\
    &=& \frac{\exp\left(-2\LM\right)}{\sum_{i=1}^{N_p} \exp\left(-\LM\right) \cdot \sum_{j=1}^{N_g} \exp\left(-\LM\right)}.
\end{eqnarray}

\subsubsection{With Background}
We now consider the underlying distributions as defined in Prop.~\ref{prop:lap}. Fundamentally, adding a background with a different weight than the other ground truth objects does not change much. The unbalanced case with $\tau_1 \rightarrow +\infty$ and $\tau_2 = 0$ remains exactly the same. The opposite case with $\tau_1 = 0$ and $\tau_2 \rightarrow +\infty$ now becomes 
\begin{equation}
        \hat{P}_{i,j} = \frac{1}{N_p}\frac{\exp\left(-\LM\right)}{\sum_{i=1}^{N_p} \exp\left(-\LM\right)},
\end{equation}
for all $1 \leq j \leq N_g$, and
\begin{equation}
        \hat{P}_{i,j} = \frac{N_p-N_g}{N_p}\frac{\exp\left(-\LM\right)}{\sum_{i=1}^{N_p} \exp\left(-\LM\right)},
\end{equation}
for $j=N_g+1$ (the background). In essence, this ensures that the mass constraints induced by $\tau_2$ are satisfied, as the background has a higher weight.

Similarly, the ``balanced'' case is the same as Eq.~\ref{eq:dual-sm-bis} for all $1 \leq j \leq N_g$. For $j=N_g+1$, we have the same with an added factor:
\begin{equation}
    P_{i,j} = \left(N_p-N_g\right)\frac{\exp\left(-\LM\right)}{\sum_{i=1}^{N_p} \exp\left(-\LM\right) \cdot \sum_{j=1}^{N_g} \exp\left(-\LM\right)}.
\end{equation}

\subsubsection{Other Regularization}
We can also consider other cases that having the regularization $\varepsilon=1$. The regularization $\varepsilon$ controls the ``softness'' of the softmax: the greater is $\varepsilon$, the softer is the minimum; the smaller, the harder. In the case of no regularization at all ($\varepsilon \rightarrow 0$), the softmax is exactly a minimum as proven in Prop.~\ref{prop:threshold}. This can be observed at Fig.~\ref{fig:sm}.

\begin{figure}
    \centering
    \begin{tabular}{cccc}
        \multirow[t]{5}{*}{%
        \begin{subfigure}{.22\textwidth}
            \centering
            \def\svgwidth{0.8\textwidth}\footnotesize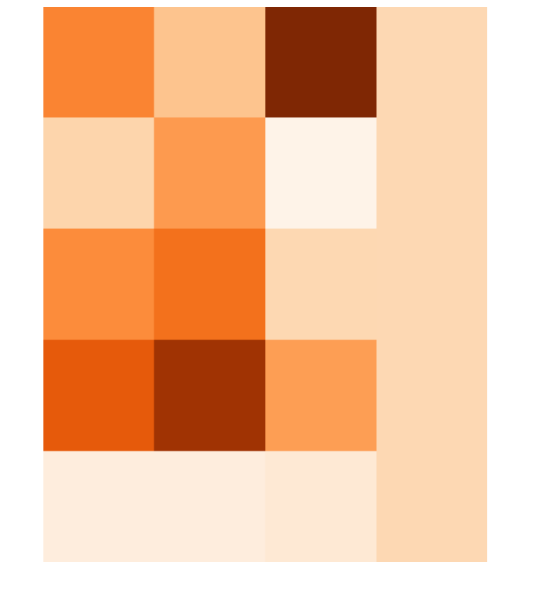
            \def\svgwidth{0.7\textwidth}
            \caption{Cost between the predictions and the ground truth objects.%
            \label{fig:sm-cost}}
        \end{subfigure}\vspace{.5em}}&%
        \begin{subfigure}{.22\textwidth}
            \centering
            \def\svgwidth{0.8\textwidth}\footnotesize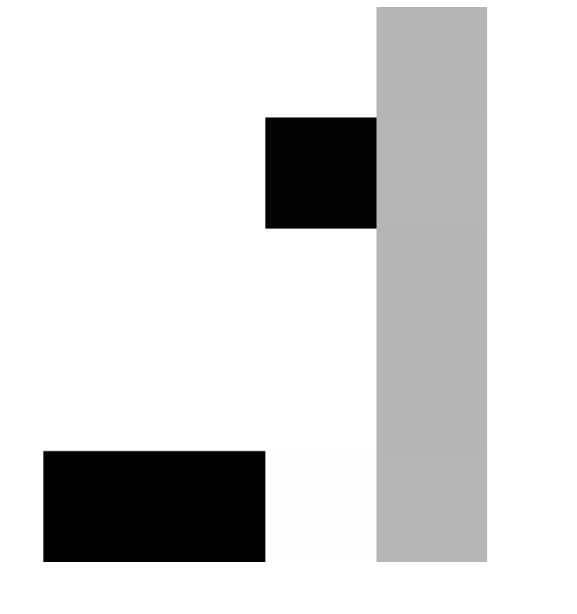
            \def\svgwidth{0.7\textwidth}
            \caption{Unbalanced OT with $\varepsilon=0.2$, $\tau_1=0.001$ and $\tau_2=1000$.%
            \label{fig:sm-unba1}}
        \end{subfigure}\vspace{.5em}&%
        \begin{subfigure}{.22\textwidth}
            \centering
            \def\svgwidth{0.8\textwidth}\footnotesize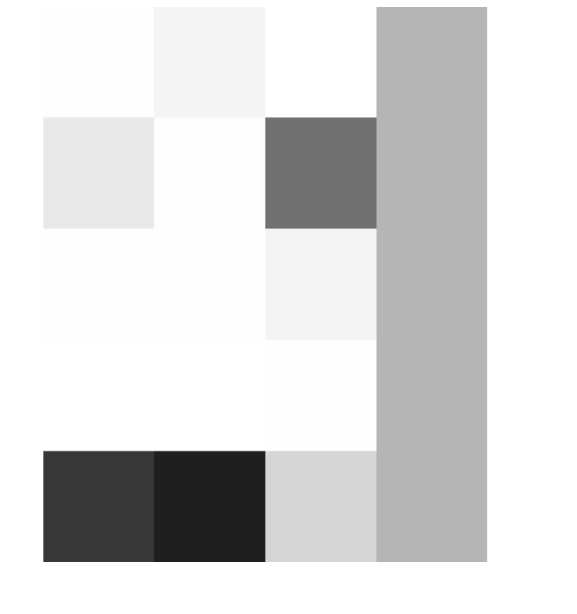
            \def\svgwidth{0.7\textwidth}
            \caption{Unbalanced OT with $\varepsilon=1$, $\tau_1=0.001$ and $\tau_2=1000$.%
            \label{fig:sm-unba2}}
        \end{subfigure}\vspace{.5em}&%
        \begin{subfigure}{.22\textwidth}
            \centering
            \def\svgwidth{0.8\textwidth}\footnotesize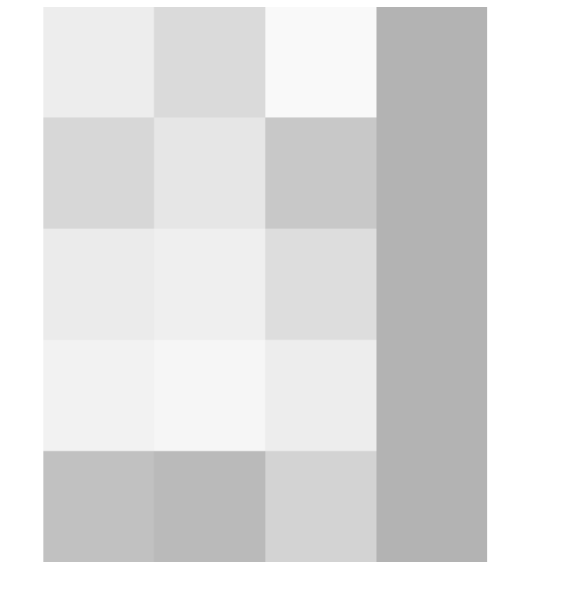
            \def\svgwidth{0.7\textwidth}
            \caption{Unbalanced OT with $\varepsilon=5$, $\tau_1=0.001$ and $\tau_2=1000$.%
            \label{fig:sm-unba3}}
        \end{subfigure}\vspace{.5em}\\&%
        \begin{subfigure}{.22\textwidth}
            \centering
            \def\svgwidth{0.8\textwidth}\footnotesize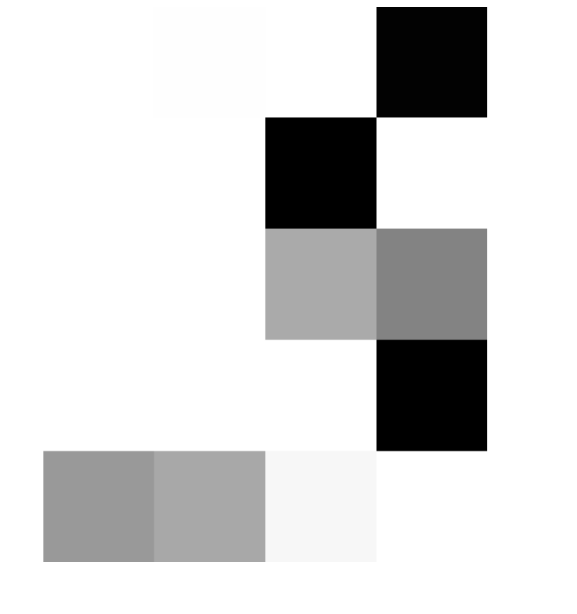
            \def\svgwidth{0.7\textwidth}
            \caption{Unbalanced OT with $\varepsilon=0.2$, $\tau_1=1000$ and $\tau_2=0.001$.%
            \label{fig:sm-unbb1}}
        \end{subfigure}\vspace{.5em}&%
        \begin{subfigure}{.22\textwidth}
            \centering
            \def\svgwidth{0.8\textwidth}\footnotesize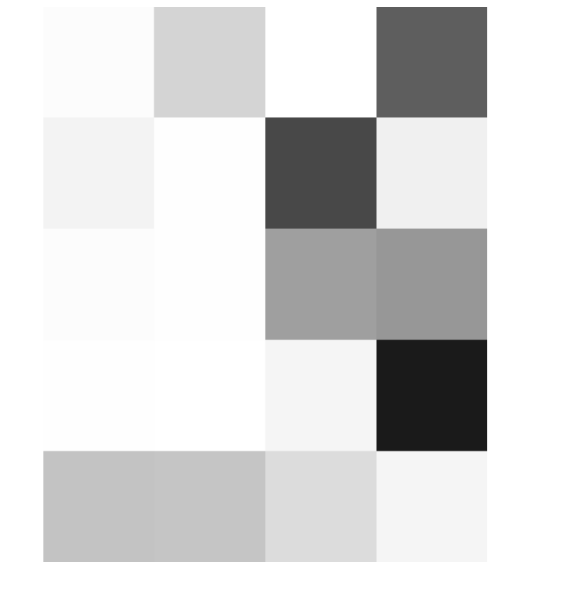
            \def\svgwidth{0.7\textwidth}
            \caption{Unbalanced OT with $\varepsilon=1$, $\tau_1=1000$ and $\tau_2=0.001$.%
            \label{fig:sm-unbb2}}
        \end{subfigure}\vspace{.5em}&%
        \begin{subfigure}{.22\textwidth}
            \centering
            \def\svgwidth{0.8\textwidth}\footnotesize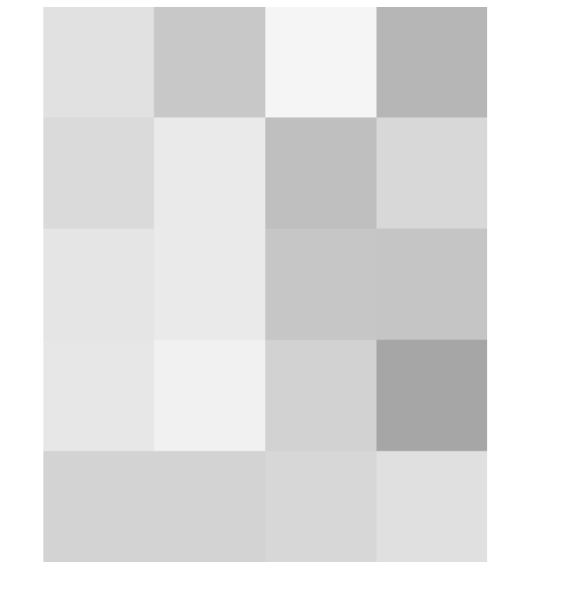
            \def\svgwidth{0.7\textwidth}
            \caption{Unbalanced OT with $\varepsilon=5$, $\tau_1=1000$ and $\tau_2=0.001$.%
            \label{fig:sm-unbb3}}
        \end{subfigure}\vspace{.5em}\\&%
        \begin{subfigure}{.22\textwidth}
            \centering
            \def\svgwidth{0.8\textwidth}\footnotesize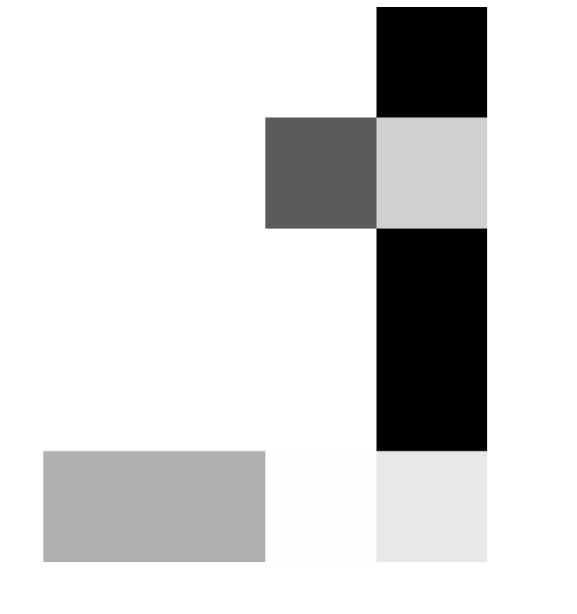
            \def\svgwidth{0.7\textwidth}
            \caption{``Balanced'' OT with only one iteration and $\varepsilon=0.2$.%
            \label{fig:sm-sm1}}
        \end{subfigure}\vspace{.5em}&%
        \begin{subfigure}{.22\textwidth}
            \centering
            \def\svgwidth{0.8\textwidth}\footnotesize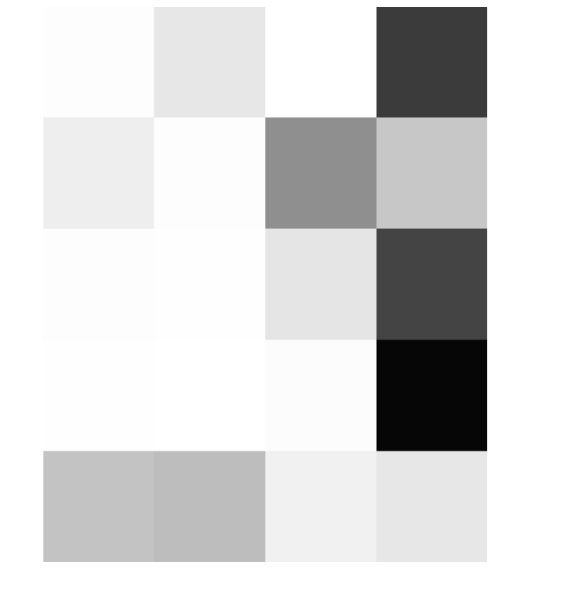
            \def\svgwidth{0.7\textwidth}
            \caption{``Balanced'' OT with only one iteration and $\varepsilon=1$.%
            \label{fig:sm-sm2}}
        \end{subfigure}\vspace{.5em}&%
        \begin{subfigure}{.22\textwidth}
            \centering
            \def\svgwidth{0.8\textwidth}\footnotesize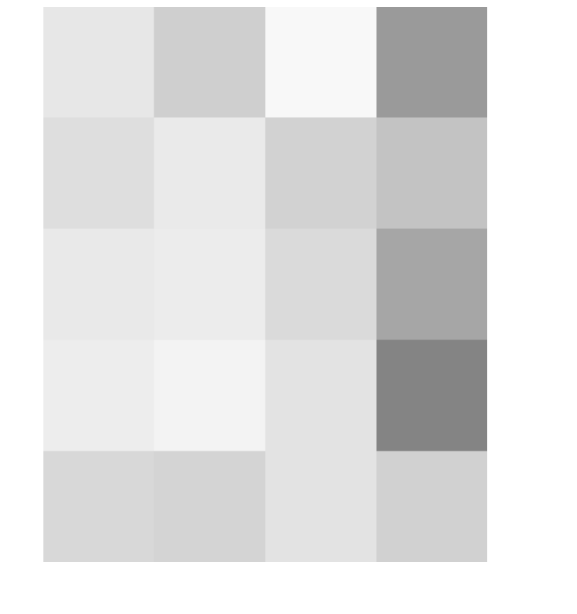
            \def\svgwidth{0.7\textwidth}
            \caption{``Balanced'' OT with only one iteration and $\varepsilon=5$.%
            \label{fig:sm-sm3}}
        \end{subfigure}\vspace{.5em}\\&%
        \begin{subfigure}{.22\textwidth}
            \centering
            \def\svgwidth{0.8\textwidth}\footnotesize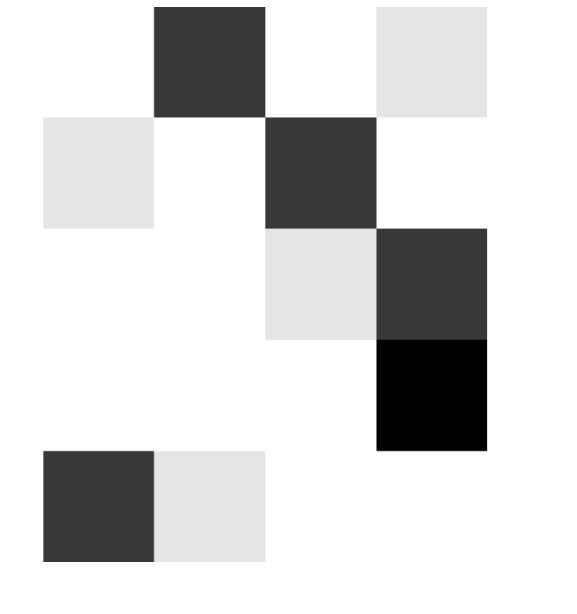
            \def\svgwidth{0.7\textwidth}
            \caption{``Balanced'' OT until convergence with $\varepsilon=0.2$.%
            \label{fig:sm-ot1}}
        \end{subfigure}&%
        \begin{subfigure}{.22\textwidth}
            \centering
            \def\svgwidth{0.8\textwidth}\footnotesize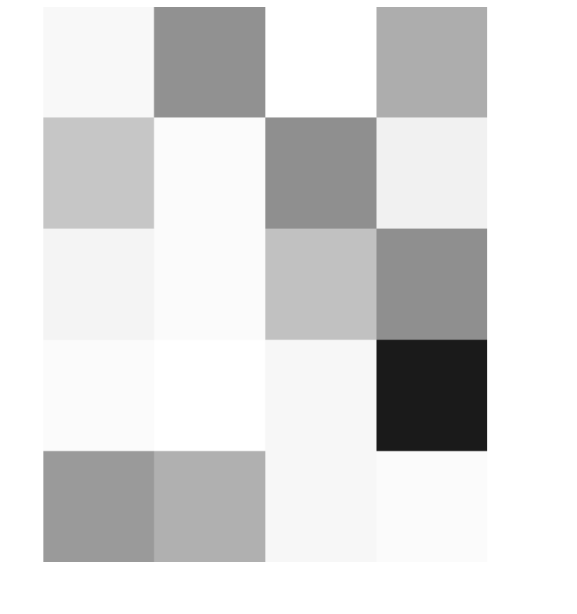
            \def\svgwidth{0.7\textwidth}
            \caption{``Balanced'' OT until convergence with $\varepsilon=1$.%
            \label{fig:sm-ot2}}
        \end{subfigure}&%
        \begin{subfigure}{.22\textwidth}
            \centering
            \def\svgwidth{0.8\textwidth}\footnotesize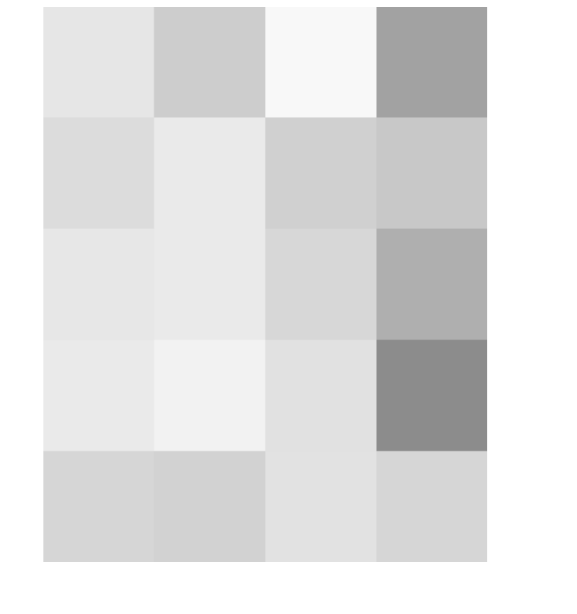
            \def\svgwidth{0.7\textwidth}
            \caption{``Balanced'' OT until convergence with $\varepsilon=5$.%
            \label{fig:sm-ot3}}
        \end{subfigure}\\&%
    \end{tabular}
    \caption{Connection between scaling algorithms and the softmax. The pairwise matching cost between the predictions (numbers) and the ground truth objects (letters) is given in Fig.~\ref{fig:sm-cost}. The background cost is $c_{\varnothing}=2$. The scaling algorithm for Unbalanced OT corresponds to performing the softmax column-wise (Figs.~\ref{fig:sm-unba1}, \ref{fig:sm-unba2} and \ref{fig:sm-unba3}), or row-wise (Figs.~\ref{fig:sm-unbb1}, \ref{fig:sm-unbb2} and \ref{fig:sm-unbb3}). Similarly, one iteration of the scaling algorithm for ``balanced'' OT is almost equivalent to the dual-softmax (Figs.~\ref{fig:sm-sm1}, \ref{fig:sm-sm2} and \ref{fig:sm-sm3}), but does not satisfy the mass constraints unlike when it is run until convergence (Figs.~\ref{fig:sm-ot1}, \ref{fig:sm-ot2} and \ref{fig:sm-ot3}).}
    \label{fig:sm}
    \vspace{5em}
\end{figure}
\section{Scaling the Entropic Parameter}
\label{app:params}
\newcommand{\spt}[1]{\operatorname{spt}\left(#1\right)}

In this section, we consider the particular choice of the entropic regularization parameter. In particular, we study how it scales with the problem size.

\subsection{Uniform Matches}
\begin{definition}[Matches]
    We define a \emph{match} $\bm{P} \in \mathbb{R}_+^{N_p \times \left(N_g+1 \right)}$ as a positive matrix of unity mass $\sum_{i,j}P_{i,j}=1$. The set of all matches of size $N_p \times \left(N_g+1 \right)$ is the joint probability simplex $\Delta^{N_p \times \left(N_g+1 \right)}$.
\end{definition}

We now consider a particular subset of all these matches.
\begin{definition}[Uniform Matches]
    We define the set of \emph{uniform matches} $\Delta^{N_p \times \left(N_g+1 \right)}_{\mathrm{unif.}} \subsetneq \Delta^{N_p \times \left(N_g+1 \right)}$ as the set of matrices $\bm{P}^{\mathrm{unif.}} \in \Delta^{N_p \times \left(N_g+1 \right)}_{\mathrm{unif.}}$, containing only zero elements and all non-zero elements having the same value:
    \begin{equation}
        P_{i,j}^{\mathrm{unif.}} = \left\{
        \begin{array}{ll}
            0 & \text{for some values,} \\
            1 / \left| \mathrm{spt}\left(\bm{P}^{\mathrm{unif.}}\right)\right| & \text{for the other values,} 
        \end{array}\right.
    \end{equation}
    with the support $\mathrm{spt} : \bm{P} \mapsto \left\{ (i,j):P_{i,j} \neq 0\right\}$ and $\left|\,\cdot\, \right|$ the cardinality of a set. 
\end{definition}

We directly see from the definition that the matrices are well defined as they have unity mass. They are uniquely defined by the carnality of their support.

\begin{proposition}[Cardinality]
    The cardinality of $\Delta^{N_p \times \left(N_g+1 \right)}_{\mathrm{unif.}}$ is given by
    \begin{equation}
        \left|\Delta^{N_p \times \left(N_g+1 \right)}_{\mathrm{unif.}}\right| = 2^{N_p\left(N_g+1 \right)}.
    \end{equation}
\end{proposition}
\begin{proof}
    We first notice that the different possible supports $k = \spt{\bm{P}^{\mathrm{unif.}}}$ range from $1 \leq k \leq N_p\left(N_g+1 \right)$. For any support of size $k$, we have to consider a uniform match containing all combinations. The rest follows from the binomial identity $\sum_{k=1}^{N_p\left(N_g+1 \right)} \binom{N_p\left(N_g+1 \right)}{k} = 2^{N_p\left(N_g+1 \right)}$.
\end{proof}

We can also see that the uniform matches cover the set of all matches.

\begin{proposition}[Diameter]
    The diameter of the set of transport matrices $\Delta^{N_p \times \left(N_g+1 \right)}$ and uniform transport matrices $\Delta^{N_p \times \left(N_g+1 \right)}_{\mathrm{unif.}}$, equipped with the Fröbenius norm $\Vert \cdot \Vert_F$, is given by
    \begin{equation}
        \mathrm{diam}\left(\Delta^{N_p \times \left(N_g+1 \right)} \right) = \mathrm{diam}\left(\Delta^{N_p \times \left(N_g+1 \right)}_{\mathrm{unif.}}\right) =\sqrt{2}
    \end{equation}
\end{proposition}
\begin{proof}
    Maximizing the Fröbenius norm is equivalent to considering the maximization of $\sum_i (u_i - v_i)^2$ subject to $\sum_i u_i = 1$ and $\sum_i v_i = 1$, with $\bm{u},\bm{v} \geq 0$. It takes its maximum value on the boundary of the admissible solutions, for $u_i = 1$ (the rest is zero) and $v_j=1$ (the rest zero) for any $j \neq i$. These extreme points are also in $\Delta^{N_p \times \left(N_g+1 \right)}_{\mathrm{unif.}}$, in particular those of unity support $\spt{\bm{P}^{\mathrm{unif.}}}=1$.
\end{proof}

\begin{figure}[]
    \centering
    \hfill
    \def\svgwidth{0.8\textwidth}\footnotesize
\begingroup%
  \makeatletter%
  \providecommand\color[2][]{%
    \errmessage{(Inkscape) Color is used for the text in Inkscape, but the package 'color.sty' is not loaded}%
    \renewcommand\color[2][]{}%
  }%
  \providecommand\transparent[1]{%
    \errmessage{(Inkscape) Transparency is used (non-zero) for the text in Inkscape, but the package 'transparent.sty' is not loaded}%
    \renewcommand\transparent[1]{}%
  }%
  \providecommand\rotatebox[2]{#2}%
  \newcommand*\fsize{\dimexpr\f@size pt\relax}%
  \newcommand*\lineheight[1]{\fontsize{\fsize}{#1\fsize}\selectfont}%
  \ifx\svgwidth\undefined%
    \setlength{\unitlength}{555.59055118bp}%
    \ifx\svgscale\undefined%
      \relax%
    \else%
      \setlength{\unitlength}{\unitlength * \real{\svgscale}}%
    \fi%
  \else%
    \setlength{\unitlength}{\svgwidth}%
  \fi%
  \global\let\svgwidth\undefined%
  \global\let\svgscale\undefined%
  \makeatother%
  \begin{picture}(1,0.45408163)%
    \lineheight{1}%
    \setlength\tabcolsep{0pt}%
    \put(0,0){\includegraphics[width=\unitlength,page=1]{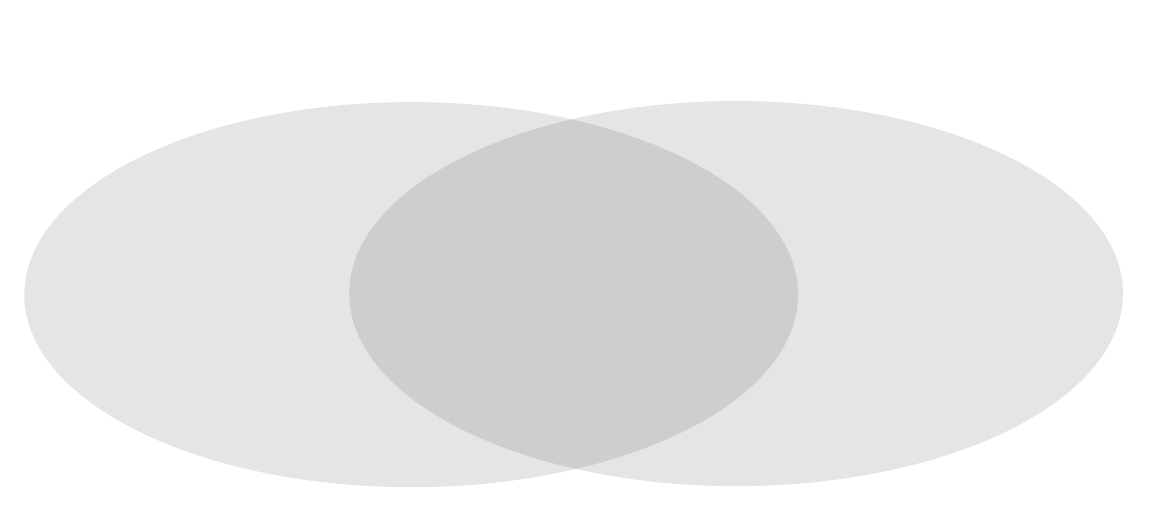}}%
    \put(0.25744808,0.3770119){\color[rgb]{0,0,0}\makebox(0,0)[lt]{\lineheight{1.25}\smash{\begin{tabular}[t]{l}$\spt{\bm{P}_1^{\mathrm{unif.}}}$\end{tabular}}}}%
    \put(0.59335934,0.37711801){\color[rgb]{0,0,0}\makebox(0,0)[lt]{\lineheight{1.25}\smash{\begin{tabular}[t]{l}$\spt{\bm{P}_2^{\mathrm{unif.}}}$\end{tabular}}}}%
    \put(0.08774447,0.22948557){\color[rgb]{0,0,0}\makebox(0,0)[lt]{\lineheight{1.25}\smash{\begin{tabular}[t]{l}\shortstack{$\spt{\bm{P}_1^{\mathrm{unif.}}}$\\$\setminus\spt{\bm{P}_2^{\mathrm{unif.}}}$}\end{tabular}}}}%
    \put(0.71248385,0.23055579){\color[rgb]{0,0,0}\makebox(0,0)[lt]{\lineheight{1.25}\smash{\begin{tabular}[t]{l}\shortstack{$\spt{\bm{P}_2^{\mathrm{unif.}}}$\\$\setminus\spt{\bm{P}_1^{\mathrm{unif.}}}$}\end{tabular}}}}%
    \put(0.40111756,0.22852131){\color[rgb]{0,0,0}\makebox(0,0)[lt]{\lineheight{1.25}\smash{\begin{tabular}[t]{l}\shortstack{$\spt{\bm{P}_1^{\mathrm{unif.}}}$\\$\cap\spt{\bm{P}_2^{\mathrm{unif.}}}$}\end{tabular}}}}%
  \end{picture}%
\endgroup%

    \caption{Decomposition of the non-zero indices of two uniform transport matrices $\bm{P}_1^{\mathrm{unif.}}, \bm{P}_2^{\mathrm{unif.}} \in \Delta^{N_p \times \left(N_g+1 \right)}_{\mathrm{unif.}}$.}
    \label{fig:proj-dist-set}
\end{figure}

\begin{proposition}
    The Fröbenius norm square $\lVert \bm{P}_1^{\mathrm{unif.}} - \bm{P}_2 ^{\mathrm{unif.}}\rVert_F^2$ between two uniform matches $\bm{P}_1^{\mathrm{unif.}}, \bm{P}_2^{\mathrm{unif.}} \in \Delta^{N_p \times \left(N_g+1 \right)}_{\mathrm{unif.}}$ is given by
    \begin{equation}
        \frac{\left|\spt{\bm{P}^{\mathrm{unif.}}_1}\right| + \left|\spt{\bm{P}^{\mathrm{unif.}}_2}\right| - 2\left|\spt{\bm{P}^{\mathrm{unif.}}_1} \cap \spt{\bm{P}^{\mathrm{unif.}}_2}\right|}{\left|\spt{\bm{P}^{\mathrm{unif.}}_1}\right| \left|\spt{\bm{P}^{\mathrm{unif.}}_2}\right|}.
    \end{equation}
\end{proposition}
\begin{proof}
    By decomposing the all indices as in Figure~\ref{fig:proj-dist-set} in
    \begin{eqnarray*}
        \spt{\bm{P}^{\mathrm{unif.}}_1} \cup \spt{\bm{P}^{\mathrm{unif.}}_1} &=& \left(\spt{\bm{P}^{\mathrm{unif.}}_1} \setminus \spt{\bm{P}^{\mathrm{unif.}}_2}\right)\\
        &&\cup \left(\spt{\bm{P}^{\mathrm{unif.}}_2} \setminus \spt{\bm{P}^{\mathrm{unif.}}_1}\right) \\
        &&\cup \left(\spt{\bm{P}^{\mathrm{unif.}}_1} \cap \spt{\bm{P}^{\mathrm{unif.}}_2}\right),
    \end{eqnarray*}
    and noticing that all other values are zero, we have for $\lVert \bm{P}^{\mathrm{unif.}}_1 - \bm{P}^{\mathrm{unif.}}_2 \rVert_F^2$
    \begin{eqnarray*}
         &&\left(\left|\spt{\bm{P}^{\mathrm{unif.}}_1}\right| - \left|\spt{\bm{P}^{\mathrm{unif.}}_1} \cap \spt{\bm{P}^{\mathrm{unif.}}_2}\right|\right)\frac{1}{\left|\spt{\bm{P}^{\mathrm{unif.}}_1}\right|^2} \\
         &+&\left(\left|\spt{\bm{P}^{\mathrm{unif.}}_2}\right| - \left|\spt{\bm{P}^{\mathrm{unif.}}_1} \cap \spt{\bm{P}^{\mathrm{unif.}}_2}\right|\right)\frac{1}{\left|\spt{\bm{P}^{\mathrm{unif.}}_2}\right|^2} \\
         &+&\left|\spt{\bm{P}^{\mathrm{unif.}}_1} \cap \spt{\bm{P}^{\mathrm{unif.}}_2}\right|\left(\frac{1}{\left|\spt{\bm{P}^{\mathrm{unif.}}_1}\right|} - \frac{1}{\left|\spt{\bm{P}^{\mathrm{unif.}}_2}\right|} \right)^2.
    \end{eqnarray*}
    The rest is just a simplification of the latter.
\end{proof}

\begin{corollary}
    Each uniform match $\bm{P}_1^{\mathrm{unif.}} \in \Delta^{N_p \times \left(N_g+1 \right)}_{\mathrm{unif.}}$ has as closest neighbors all other uniform matches $\bm{P}_2^{\mathrm{unif.}} \in \Delta^{N_p \times \left(N_g+1 \right)}_{\mathrm{unif.}}$ of support increased by one $\left|\spt{\bm{P}^{\mathrm{unif.}}_2}\right| = \left|\spt{\bm{P}_1^{\mathrm{unif.}}}\right|+1$ and differing in support for only one entry $\left|\spt{\bm{P}^{\mathrm{unif.}}_2} \setminus\spt{\bm{P}^{\mathrm{unif.}}_1}\right| = 1$. In particular, the square Fröbenius norm is then equal to
    \begin{equation}
        \left\lVert \spt{\bm{P}^{\mathrm{unif.}}_1} -\spt{\bm{P}^{\mathrm{unif.}}_2} \right\rVert_F^2 = \frac{1}{\left|\spt{\bm{P}_1^{\mathrm{unif.}}}\right|\left(\left|\spt{\bm{P}_1^{\mathrm{unif.}}}\right|+1\right)}.
    \end{equation}
    In in the particular limit case of $\left|\spt{\bm{P}_1^{\mathrm{unif.}}}\right| = N_p\left(N_g+1 \right)$, its closest neighbors are all the $\bm{P}_2^{\mathrm{unif.}}$ such that $\left|\spt{\bm{P}_2^{\mathrm{unif.}}}\right| = N_p\left(N_g+1 \right)-1$.
\end{corollary}


    

\begin{proposition}
    We consider the projector $\mathbb{P} : \Delta^{N_p \times \left(N_g+1 \right)} \rightarrow \Delta^{N_p \times \left(N_g+1 \right)}_{\mathrm{unif.}}$, that minimizes the Fröbenius norm. For any $\bm{P} \in \Delta^{N_p \times \left(N_g+1 \right)}$, we consider
    \begin{equation}
        \mathbb{P}(\bm{P}) = \underset{\bm{P}^{\mathrm{unif.}} \in \Delta^{N_p \times \left(N_g+1 \right)}_{\mathrm{unif.}}}{\mathrm{argmin}} \lVert \bm{P} - \bm{P}^{\mathrm{unif.}} \rVert_F.
    \end{equation}
    It is given by the matrix $\bm{P}^{\mathrm{unif.}} \in \Delta^{N_p \times \left(N_g+1 \right)}_{\mathrm{unif.}}$ with the $k$ greatest elements of $\bm{P}$ as support and
    \begin{equation}
    \label{eq:proj-max-k}
        k = \underset{k \in \llbracket N_p\left(N_g+1 \right) \rrbracket}{\mathrm{arg\,max}}\, \frac{1}{k}\left( 2\sum_{\substack{\text{$k$ greatest}\\ \text{elements}}}P_{ij} - 1 \right).
    \end{equation}
\end{proposition}
\begin{proof}
    We consider the distance between any element $\bm{P} \in \Delta^{N_p \times \left(N_g+1 \right)}$ and $\bm{P}^{\mathrm{unif.}} \in \Delta^{N_p \times \left(N_g+1 \right)}_{\mathrm{unif.}}$: 
    $\lVert \bm{P} - \bm{P}^{\mathrm{unif.}} \rVert_F^2 = \sum_{i,j = 1}^{N_p,\left(N_g+1 \right)} \left( P_{i,j} - P^{\mathrm{unif.}}_{i,j}\right)^2 = \sum_{i,j = 1}^{N_p,\left(N_g+1 \right)} P_{i,j}^2 + \left(P^{\mathrm{unif.}}_{i,j}\right)^2 - 2P_{ij}P^{\mathrm{unif.}}_{i,j}$. We notice that because of the uniform nature of $\bm{P}^{\mathrm{unif.}}$, it only has $k$ non-zero elements, all equal to $1/k$. By consequence we have $\sum_{i,j = 1}^{N_p,\left(N_g+1 \right)} \left(P^{\mathrm{unif.}}_{ij}\right)^2 = \frac{1}{k}$ and $\sum_{i,j = 1}^{N_p,\left(N_g+1 \right)}P_{ij}P^{\mathrm{unif.}}_{ij} = \frac{1}{k}\sum_{\mathrm{spt}(\bm{P}^{\mathrm{unif.}})} P_{ij}$.
    The distance is now equal to $\lVert \bm{P} - \bm{P}^{\mathrm{unif.}} \rVert_F^2 = \lVert \bm{P} \rVert^2_F + \frac{1}{k} - \frac{2}{k}\sum_{\mathrm{spt}(\bm{P}')} P_{ij}$,
    and is minimal if $\frac{2}{k}\sum_{\mathrm{spt}(\bm{P}^{\mathrm{unif.}})} P_{ij} - \frac{1}{k}$ is maximal which is unique as it suffices to see that is it reached once
    \begin{equation}
        \label{eq:proj-unif-stop}
        \sum_{\substack{\text{$k$ greatest}\\ \text{elements}}}P_{ij} > P_{\substack{\text{next}\\ \text{greatest}}} + \frac12,
    \end{equation}
    is satisfied.
\end{proof}
The norm with the projector is therefore also given by $\left\lVert \bm{P} - \mathbb{P}(\bm{P}) \right\rVert^2_F = \lVert \bm{P} \rVert_F - 2\sum_{\spt{\mathbb{P}(\bm{P})}}P_{ij} + \frac{1}{\left|\spt{\mathbb{P}(\bm{P})}\right|}$. \Cref{eq:proj-unif-stop} gives a direct algorithm to determine $k$ and thus the projected value of any match $\bm{P}$.

\subsection{Entropy}
The study of uniform matches is relevant as they have an easy formulation for their entropy.

\begin{definition}[Entropy]
    The entropy $\mathrm{H} : \Delta^{N_p \times \left(N_g+1 \right)} \rightarrow \mathbb{R}_{\geq 0}$ of a match $\bm{P}$ is given by
    \begin{equation}
        \mathrm{H}(\bm{P}) \stackrel{\text{def.}}{=} - \sum_{i,j}  P_{i,j} \left( \log\left( P_{i,j}\right)-1\right).
    \end{equation}
    If one of the elements would be zero, i.e., $P_{i,j}=0$, we consider $P_{i,j} \log \left( P_{i,j} -1\right) = 0$.
\end{definition}
The latter condition ensures that the entropy is well defined. This choice is justified as it remains consistent with the limit. Some authors prefer another convention~\cite{peyre2019computational}.

\begin{lemma}
    The entropy of a uniform match $\bm{P}^{\mathrm{unif.}} \in \Delta^{N_p \times \left(N_g+1 \right)}_{\mathrm{unif.}}$ is given by
    \begin{equation}
        \mathrm{H}(\bm{P}^{\mathrm{unif.}}) = \log \left( \left|\mathrm{spt}(\bm{P}^{\mathrm{unif.}})\right| \right) + 1.
    \end{equation}
\end{lemma}
\begin{proof}
    The proof is a direct application of the definition of the entropy:
    \begin{eqnarray}
        \mathrm{H}(\bm{P}^{\mathrm{unif.}}) &=& - \sum_{i,j}  P_{i,j}^{\mathrm{unif.}} \left( \log\left( P^{\mathrm{unif.}}_{i,j}\right)-1\right), \\
        &=& - \sum_{\spt{\bm{P}^{\mathrm{unif.}}}}  P_{i,j}^{\mathrm{unif.}} \left( \log\left( P_{i,j}^{\mathrm{unif.}}\right)-1\right), \\
        &=& - \frac{\left| \mathrm{spt}\left(\bm{P}^{\mathrm{unif.}}\right)\right|}{\left| \mathrm{spt}\left(\bm{P}^{\mathrm{unif.}}\right)\right|} \left( \log\left( \frac{1}{\left| \mathrm{spt}\left(\bm{P}^{\mathrm{unif.}}\right)\right|}\right)-1\right), \\
        &=& \log \left( \left|\mathrm{spt}(\bm{P}^{\mathrm{unif.}})\right| \right) + 1.
    \end{eqnarray}
\end{proof}

\begin{proposition}
    For any match $\bm{P} \in \Delta^{N_p \times \left(N_g+1 \right)}$,
    \begin{equation}
        1 \leq H(\bm{P}) \leq \log(N_p\left(N_g+1 \right)) + 1.
    \end{equation}
\end{proposition}
\begin{proof}
    For an arbitrary coupling matrix, the entropy is always minimal if $P_{i,j} = 1$ for one element and all the others are zero. Similarly, the entropy is always for the uniform match $P_{i,j} = 1 / \left| \mathrm{spt}\left(\bm{P}\right)\right|$ for all $i,j$, with $\left| \mathrm{spt}\left(\bm{P}\right)\right| = N_p \times \left(N_g+1 \right)$.
\end{proof}

\subsection{Rule of Thumb}
We first consider two different matches of different dimensions $\bm{P}_1 \in \Delta^{N_{p,1} \times \left(N_{g,1} + 1\right)}$ and $\bm{P}_2 \in \Delta^{N_{p,2} \times \left(N_{g,2} + 1\right)}$. In this case, the OT with regularization cost (Definition~\ref{def:rOT}) is given by $\sum_{i,j=1}^{N_p,\left(N_g+1 \right)} P_{i,j}\mathcal{L}_{\text{match}}\left(\hat{\bm{y}}_i, \bm{y}_j\right) - \epsilon \,\mathrm{H}(\bm{P})$. The goal is to scale the regularization parameter $\epsilon$ in such a way that the weight of the entropy is proportionally the same. Because of unit mass of any match, we could assume that the first term $\sum_{i,j=1}^{N_p,\left(N_g+1 \right)} P_{i,j}\mathcal{L}_{\text{match}}\left(\hat{\bm{y}}_i, \bm{y}_j\right)$ is independent of $N_p$ and $N_g$ in magnitude. We therefore have to guarantee that $\epsilon_1 \mathrm{H}\left(\bm{P}_1\right) = \epsilon_2 \mathrm{H}\left(\bm{P}_2\right)$. Given an already determined regularization value $\epsilon_1$ for one of the two sizes, the other can be found with $\epsilon_2 = \epsilon_1 \mathrm{H}\left(\bm{P}_1\right) / \mathrm{H}\left(\bm{P}_2\right)$. In practice, however, the entropy is not trivial and we can rely on the projection onto the uniform matches
\begin{equation}
    \epsilon_1 = \epsilon_2 \frac{\log\left(\left|\spt{\mathbb{P}\left(\bm{P}_2\right)}\right|\right)+1}{\log\left(\left|\spt{\mathbb{P}\left(\bm{P}_1\right)}\right|\right)+1}.
\end{equation}
In the particular case of Proposition~\ref{prop:lap}, we can use the approximation $\left|\spt{\mathbb{P}\left(\bm{P}\right)}\right| = N_p$, which gives
\begin{equation}
    \epsilon_1 = \epsilon_2 \frac{\log\left(N_{p,2}\right)+1}{\log\left(N_{p,1}\right)+1}.
\end{equation}
The idea is to determine the optimal value $\epsilon_1$ on toy examples. By setting $N_p = N_{p,2}$, $\epsilon = \epsilon_2$ and $\epsilon_0 = \epsilon_1 \left(\log\left(N_{p,1}\right)+1\right)$, we can use the simple scaling formula $\epsilon = \epsilon_0 / \left(\log\left(N_p\right)+1\right)$. From our experiments, we determined $\epsilon_0 = 0.12$.

    



\section{Qualitative Analysis}
\label{app:qualitative_results_and_discussion}

This section provides qualitative examples (Figure \ref{fig:matchings_visualization_DETR_color-H} and Figure \ref{fig:matchings_visualization_Def-DETR_color-H}) of some matches, as well as a convergence analysis for DETR and Deformable DETR. We compare the losses and matches $\bm{P}$ of the two matching algorithms at different training epochs.

Figure~\ref{fig:matchings_visualization_DETR_color-H} shows some assignments of the two matching algorithms for DETR on the Color Boxes dataset. We sample examples with few ground truth objects for readability. We only show predictions that are matched at least once with a background $\varnothing$ ground truth in three consecutive epochs.
At the beginning of the training, the \emph{Bipartite Matching} with the \emph{Hungarian algorithm} assigns different predictions to the ground truth objects from one epoch to the other. As an example, the algorithm for image \textnumero 630 assigns predictions $\{4,49\}$, $\{8,1\}$ and then $\{99,1\}$ to the ground-truth objects $\{A,B\}$ at epoch 25 to 27 (Figure~\ref{subfig:matchings_visualization_DETR_color-H_630}). The regularized OT match instead provides a smoother solution and is more consistent from one epoch to the other. Later in training, Figure~\ref{subfig:matchings_visualization_DETR_color-H_630} illustrates that the regularized OT matches are one-to-one and behave like the bipartite ones.

Figure \ref{fig:convergence_loss_curves_DETR_color-H} provides the loss curves for DETR on the Color Boxes dataset. The curves suggest that the cross-entropy loss term mainly drives the convergence speedup in the early training epochs. We don't observe such speedups on COCO or with Deformable DETR (Figure \ref{fig:convergence_loss_curves_Def-DETR_color-H}). An explanation could be that the difference between DETR and Deformable DETR is due to the slower convergence of transformers (we also tried DETR with the focal loss from Deformable DETR without improvement). The difference between Color Boxes and COCO is difficult to isolate, but probably due to the wider class diversity in the latter.

\begin{figure}
    \centering
    \begin{subfigure}{0.99\textwidth}
    \begin{tabularx}{\textwidth}{YYY}
        \adjustbox{valign=t}{\def\svgwidth{0.23\textwidth}\footnotesize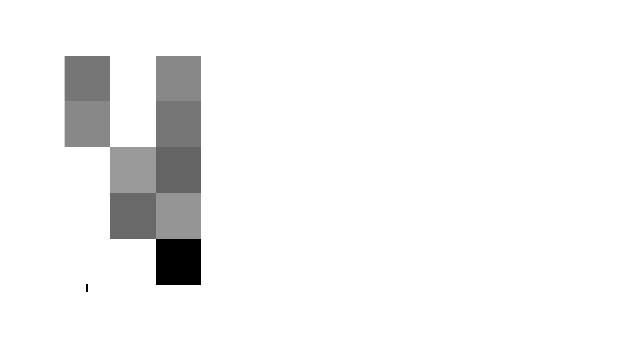}
        & 
        \adjustbox{valign=t}{\def\svgwidth{0.23\textwidth}\footnotesize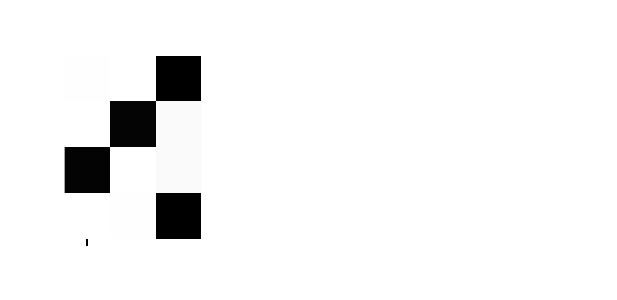}
        & 
        \adjustbox{valign=t}{\def\svgwidth{0.23\textwidth}\footnotesize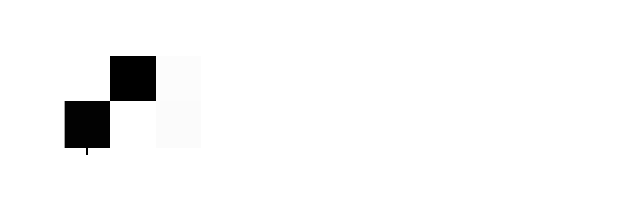}
        \\
        \adjustbox{valign=t}{\def\svgwidth{0.23\textwidth}\footnotesize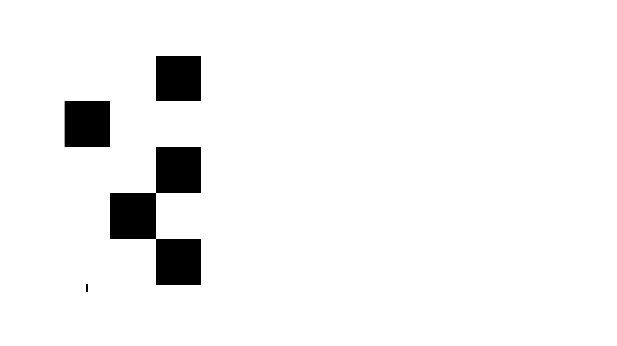}
        & 
        \adjustbox{valign=t}{\def\svgwidth{0.23\textwidth}\footnotesize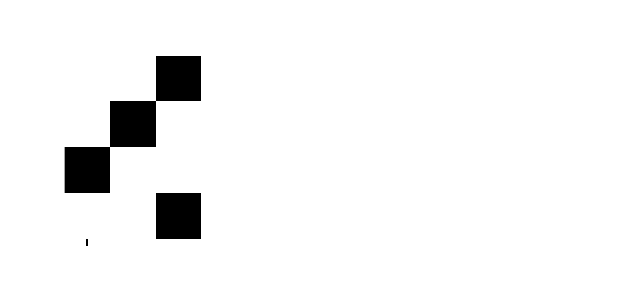}
        & 
        \adjustbox{valign=t}{\def\svgwidth{0.23\textwidth}\footnotesize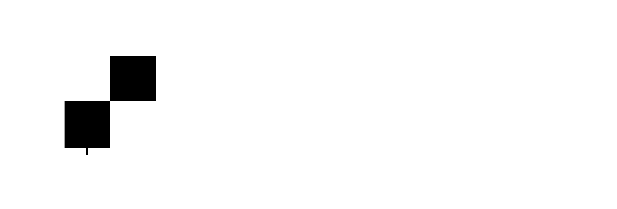}
    \end{tabularx}
    \caption{Result of the OT match (top row) and the Hungarian match (bottom row) on image \textnumero 630}
    \label{subfig:matchings_visualization_DETR_color-H_630}
    \end{subfigure} 
    
    \begin{subfigure}{0.99\textwidth}
    \begin{tabularx}{\textwidth}{YYY}
        \adjustbox{valign=t}{\def\svgwidth{0.26\textwidth}\footnotesize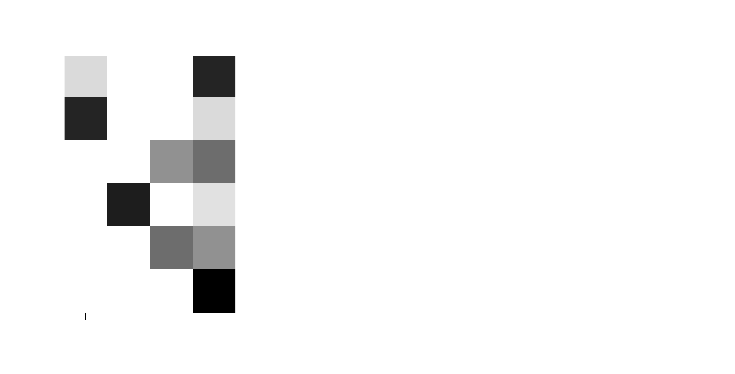}
        & 
        \adjustbox{valign=t}{\def\svgwidth{0.26\textwidth}\footnotesize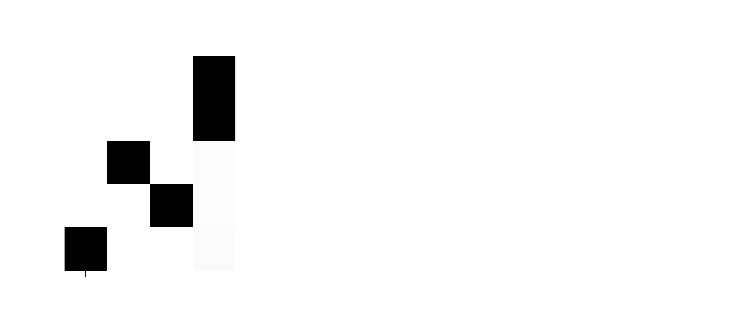}
        & 
        \adjustbox{valign=t}{\def\svgwidth{0.26\textwidth}\footnotesize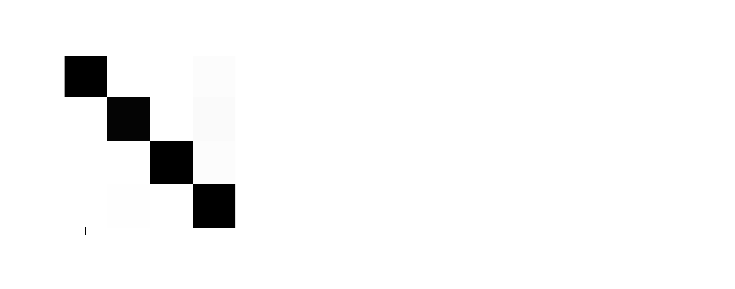}
        \\
        \adjustbox{valign=t}{\def\svgwidth{0.26\textwidth}\footnotesize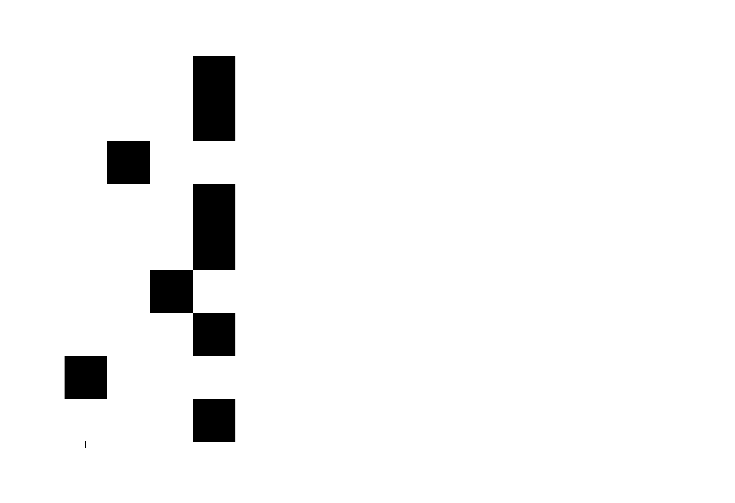}
        & 
        \adjustbox{valign=t}{\def\svgwidth{0.26\textwidth}\footnotesize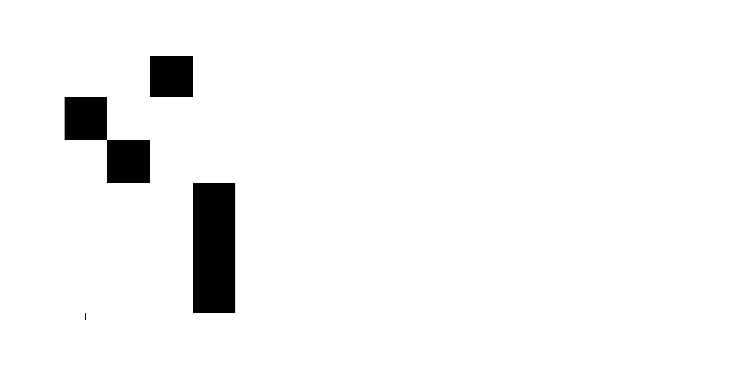}
        & 
        \adjustbox{valign=t}{\def\svgwidth{0.26\textwidth}\footnotesize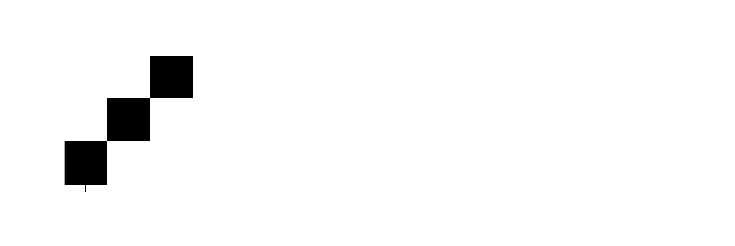}
    \end{tabularx}
    \caption{Result of the OT match (top row) and the Hungarian match (bottom row) on image \textnumero 180}
    \label{subfig:matchings_visualization_DETR_color-H_180}
    \end{subfigure} 

    \begin{subfigure}{0.99\textwidth}
    \begin{tabularx}{\textwidth}{YYY}
        \adjustbox{valign=t}{\def\svgwidth{0.33\textwidth}\footnotesize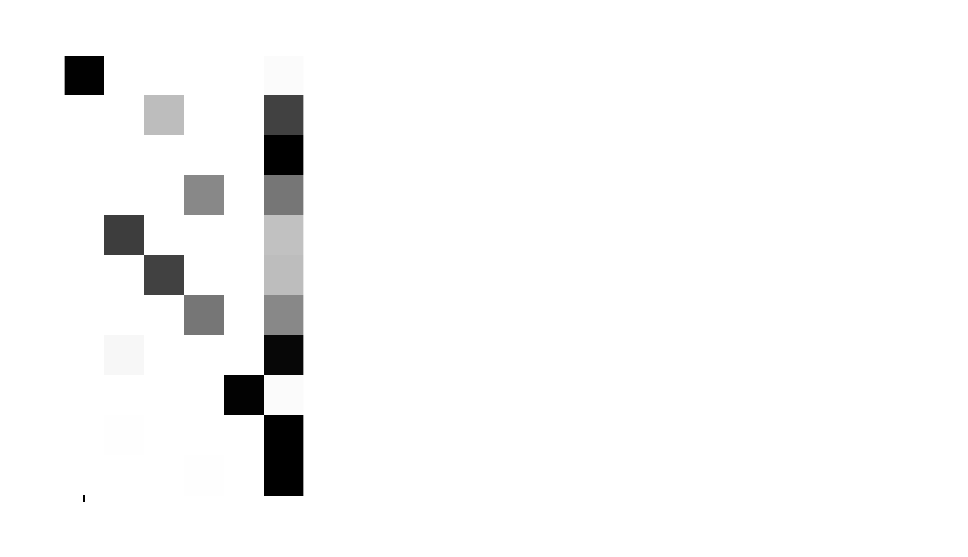}
        & 
        \adjustbox{valign=t}{\def\svgwidth{0.33\textwidth}\footnotesize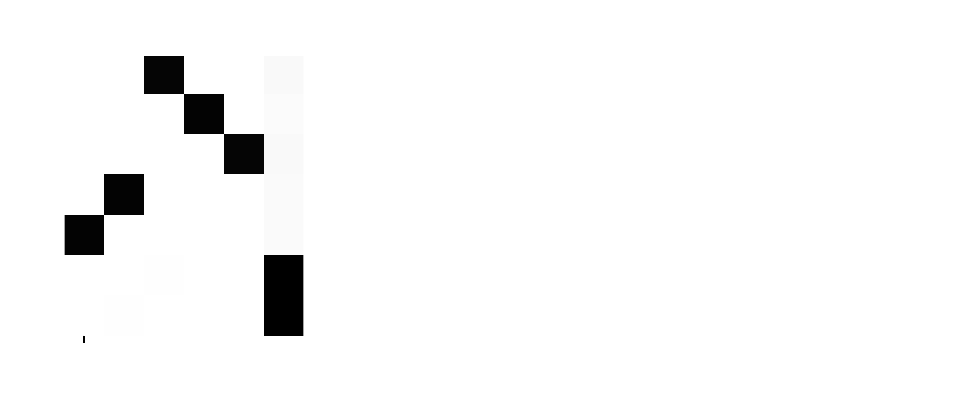}
        & 
        \adjustbox{valign=t}{\def\svgwidth{0.33\textwidth}\footnotesize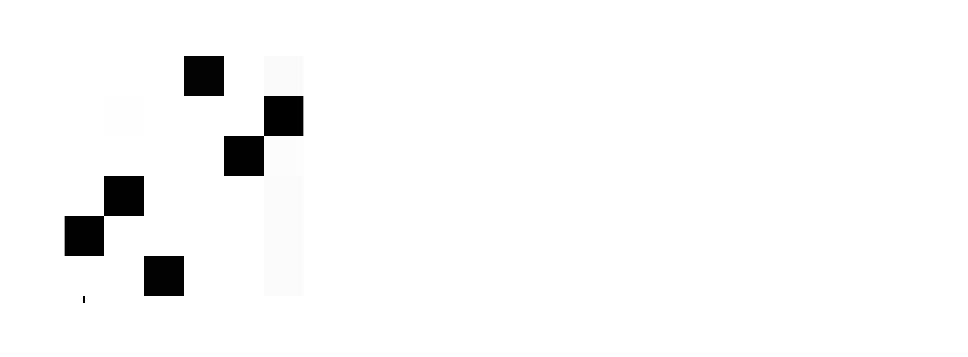}
        \\
        \adjustbox{valign=t}{\def\svgwidth{0.33\textwidth}\footnotesize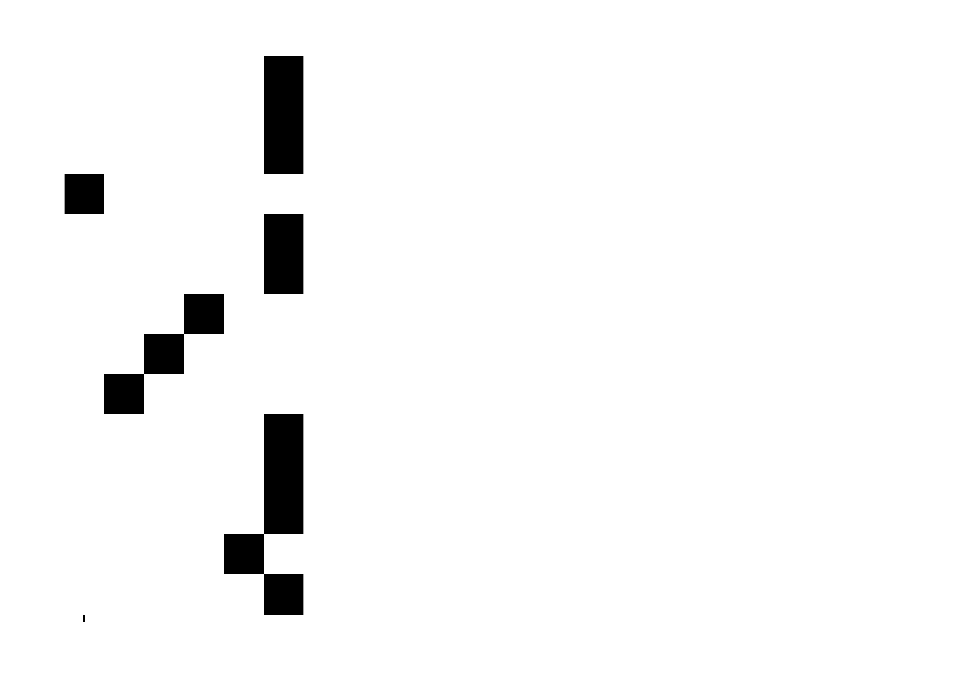}
        & 
        \adjustbox{valign=t}{\def\svgwidth{0.33\textwidth}\footnotesize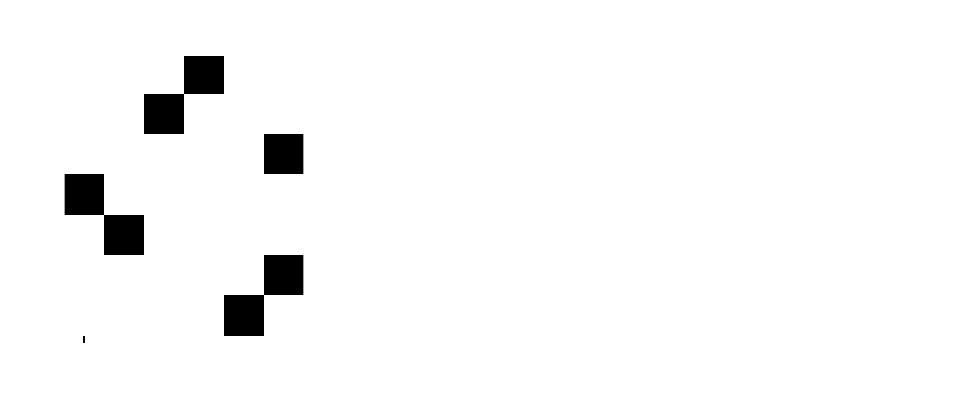}
        & 
        \adjustbox{valign=t}{\def\svgwidth{0.33\textwidth}\footnotesize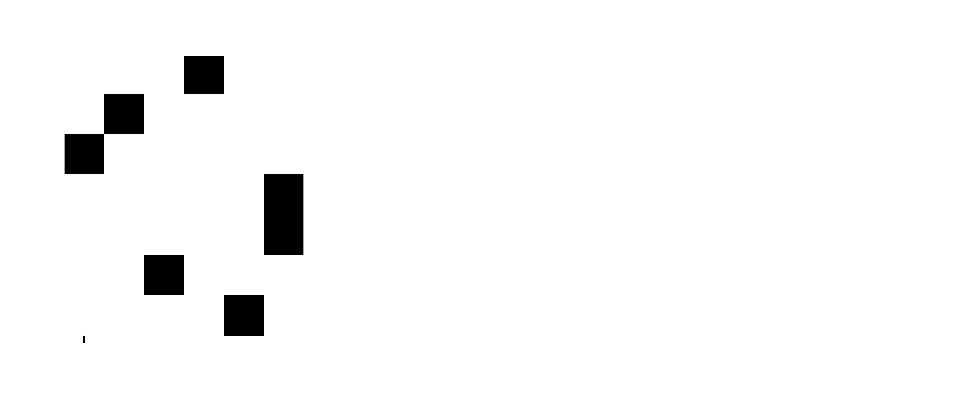}
    \end{tabularx}
    \caption{Result of the OT match (top row) and the Hungarian match (bottom row) on image \textnumero 613}
    \label{subfig:matchings_visualization_DETR_color-H_613}
    \end{subfigure} 

    \caption{Output of the matching algorithms with DETR on the validation set of the Color Boxes Dataset. The model is trained two times: once with an OT match and once with a Hungarian matching.
    The rows indicate the predictions and the columns indicate the ground truth objects (including the background $\varnothing$).
    We sample examples with few ground truth objects for readability and only show predictions that are matched at least once with a non-background ground truth.}
    \label{fig:matchings_visualization_DETR_color-H}
\end{figure}

\begin{figure}
    \centering
    \begin{subfigure}{0.99\textwidth}
    \begin{tabularx}{\textwidth}{YYY}
        \adjustbox{valign=t}{\def\svgwidth{0.23\textwidth}\footnotesize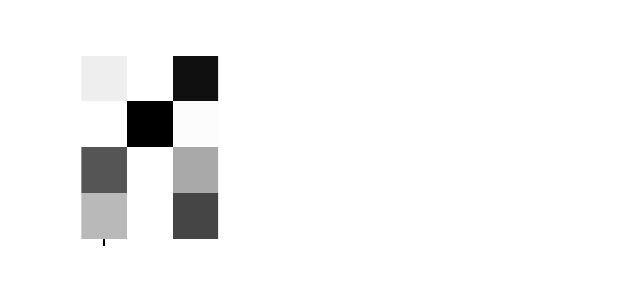}
        & 
        \adjustbox{valign=t}{\def\svgwidth{0.23\textwidth}\footnotesize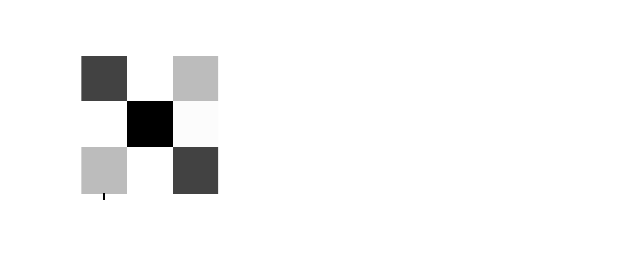}
        & 
        \adjustbox{valign=t}{\def\svgwidth{0.23\textwidth}\footnotesize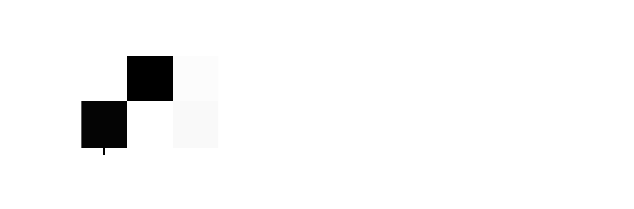}
        \\
        \adjustbox{valign=t}{\def\svgwidth{0.23\textwidth}\footnotesize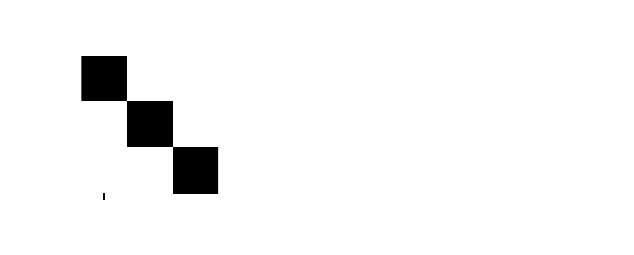}
        & 
        \adjustbox{valign=t}{\def\svgwidth{0.23\textwidth}\footnotesize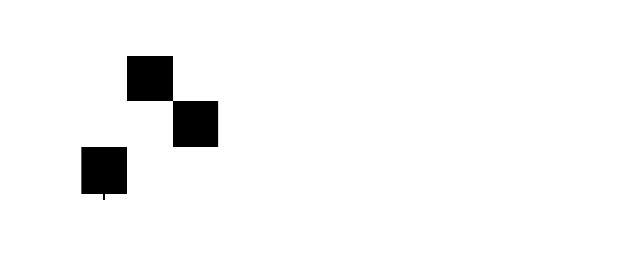}
        & 
        \adjustbox{valign=t}{\def\svgwidth{0.23\textwidth}\footnotesize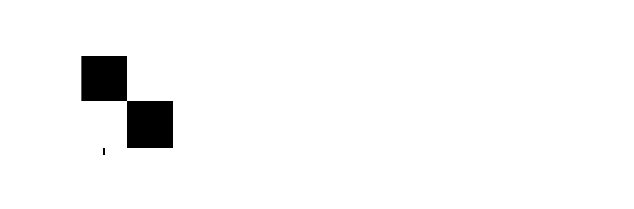}
    \end{tabularx}
    \caption{Result of the OT match (top row) and the Hungarian match (bottom row) on image \textnumero 630}
    \label{subfig:matchings_visualization_Def-DETR_color-H_630}
    \end{subfigure} 
    
    \begin{subfigure}{0.99\textwidth}
    \begin{tabularx}{\textwidth}{YYY}
        \adjustbox{valign=t}{\def\svgwidth{0.26\textwidth}\footnotesize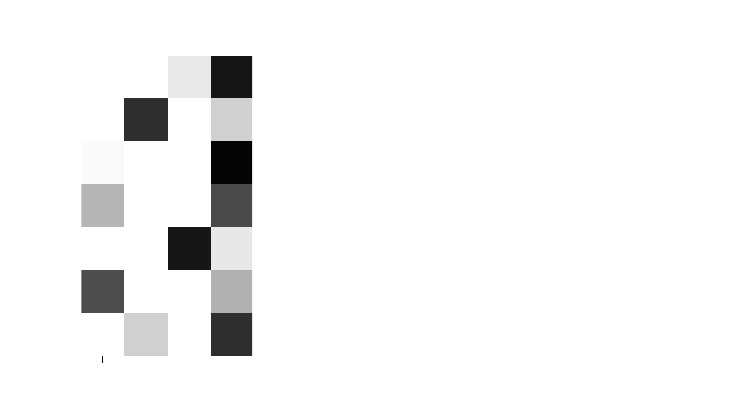}
        & 
        \adjustbox{valign=t}{\def\svgwidth{0.26\textwidth}\footnotesize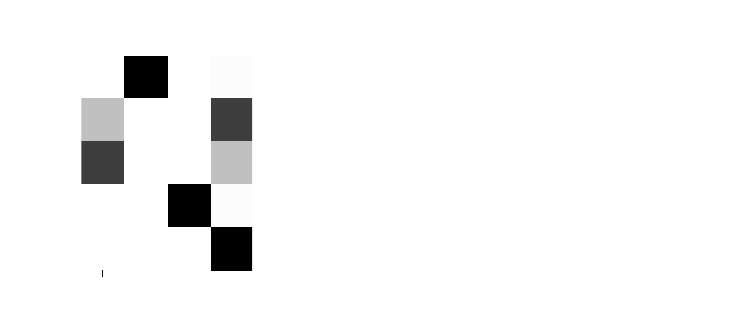}
        & 
        \adjustbox{valign=t}{\def\svgwidth{0.26\textwidth}\footnotesize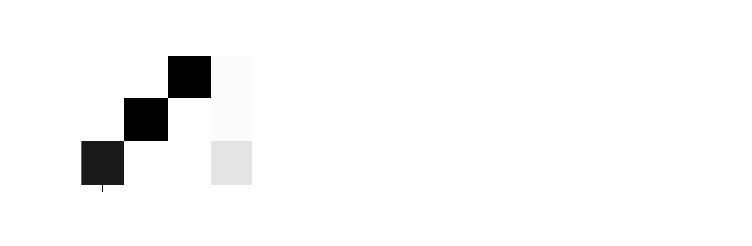}
        \\
        \adjustbox{valign=t}{\def\svgwidth{0.26\textwidth}\footnotesize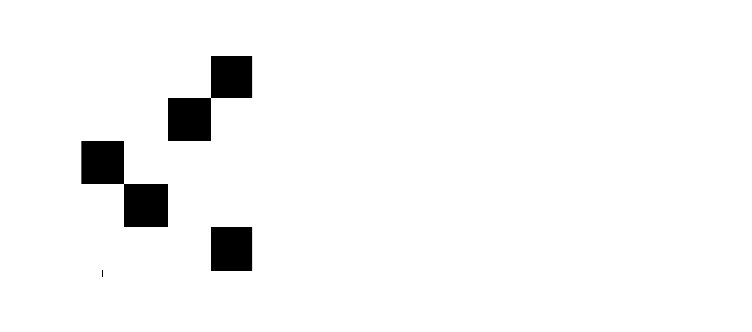}
        & 
        \adjustbox{valign=t}{\def\svgwidth{0.26\textwidth}\footnotesize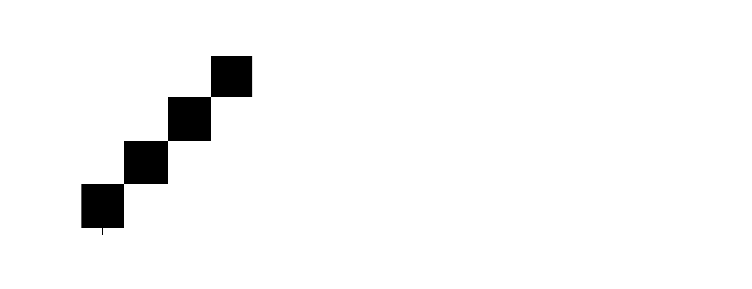}
        & 
        \adjustbox{valign=t}{\def\svgwidth{0.26\textwidth}\footnotesize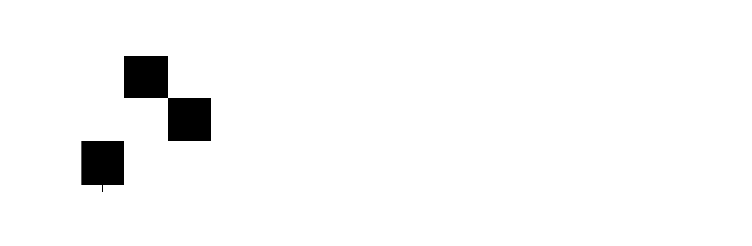}
    \end{tabularx}
    \caption{Result of the OT match (top row) and the Hungarian match (bottom row) on image \textnumero 180}
    \label{subfig:matchings_visualization_Def-DETR_color-H_180}
    \end{subfigure} 

    \begin{subfigure}{0.99\textwidth}
    \begin{tabularx}{\textwidth}{YYY}
        \adjustbox{valign=t}{\def\svgwidth{0.33\textwidth}\footnotesize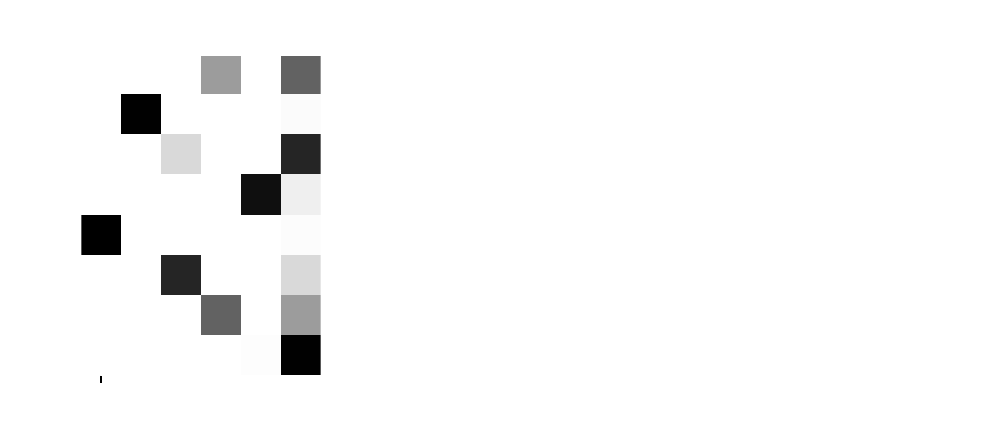}
        & 
        \adjustbox{valign=t}{\def\svgwidth{0.33\textwidth}\footnotesize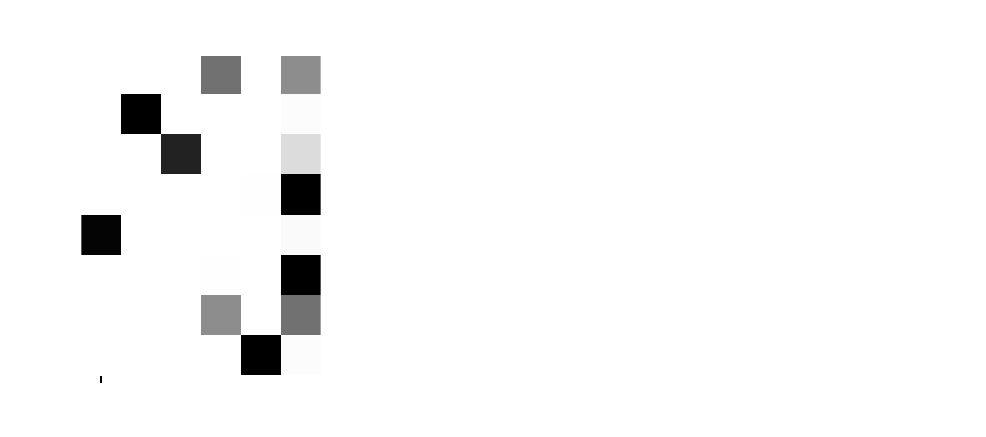}
        & 
        \adjustbox{valign=t}{\def\svgwidth{0.33\textwidth}\footnotesize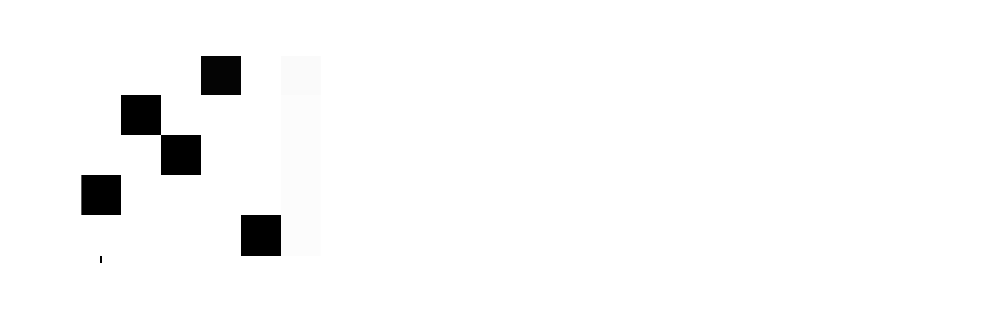}
        \\
        \adjustbox{valign=t}{\def\svgwidth{0.33\textwidth}\footnotesize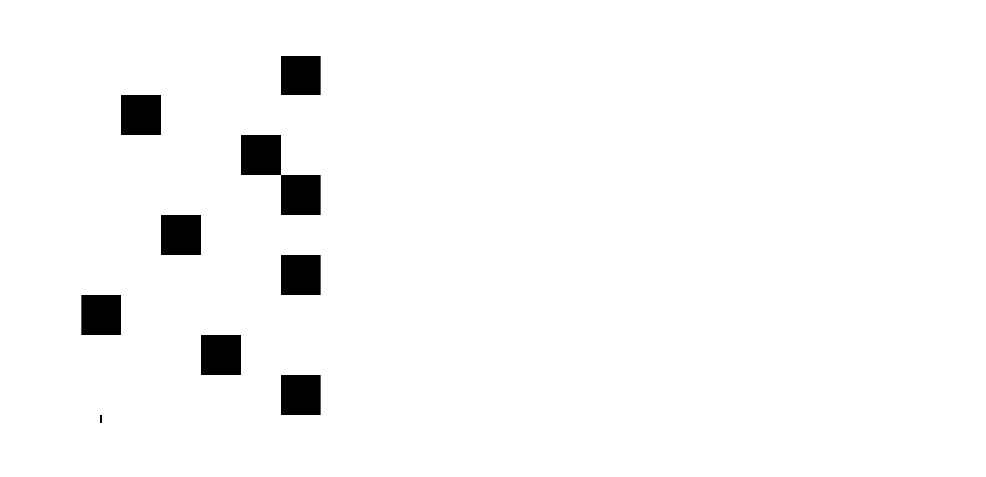}
        & 
        \adjustbox{valign=t}{\def\svgwidth{0.33\textwidth}\footnotesize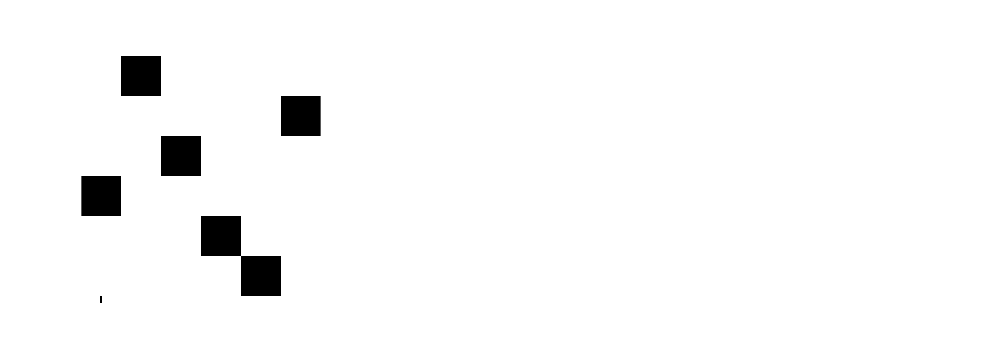}
        & 
        \adjustbox{valign=t}{\def\svgwidth{0.33\textwidth}\footnotesize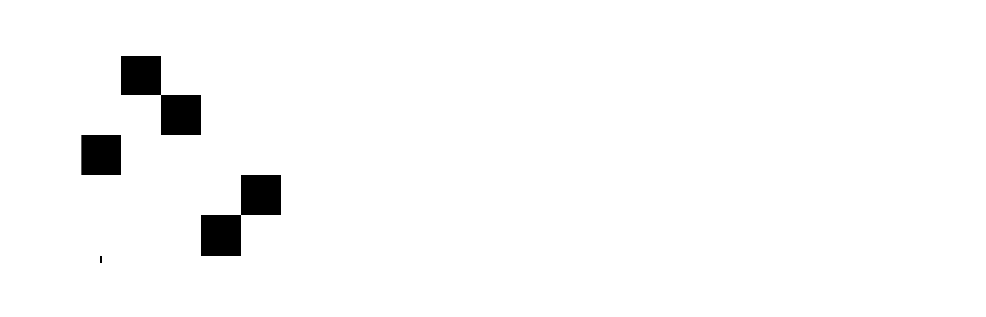}
    \end{tabularx}
    \caption{Result of the OT match (top row) and the Hungarian match (bottom row) on image \textnumero 613}
    \label{subfig:matchings_visualization_Def-DETR_color-H_613}
    \end{subfigure} 

    \caption{Output of the matching algorithms with Deformable-DETR on the validation set of the Color Boxes Dataset. The model is trained two times: once with an OT match and once with a Hungarian matching.
    The rows indicate the predictions and the columns indicate the ground truth objects (including the background $\varnothing$).
    We sample examples with few ground truth objects for readability and only show predictions that are matched at least once with a non-background ground truth.}
   \label{fig:matchings_visualization_Def-DETR_color-H}
\end{figure}

\begin{figure}
    \centering
    \def\svgwidth{0.8\textwidth}\footnotesize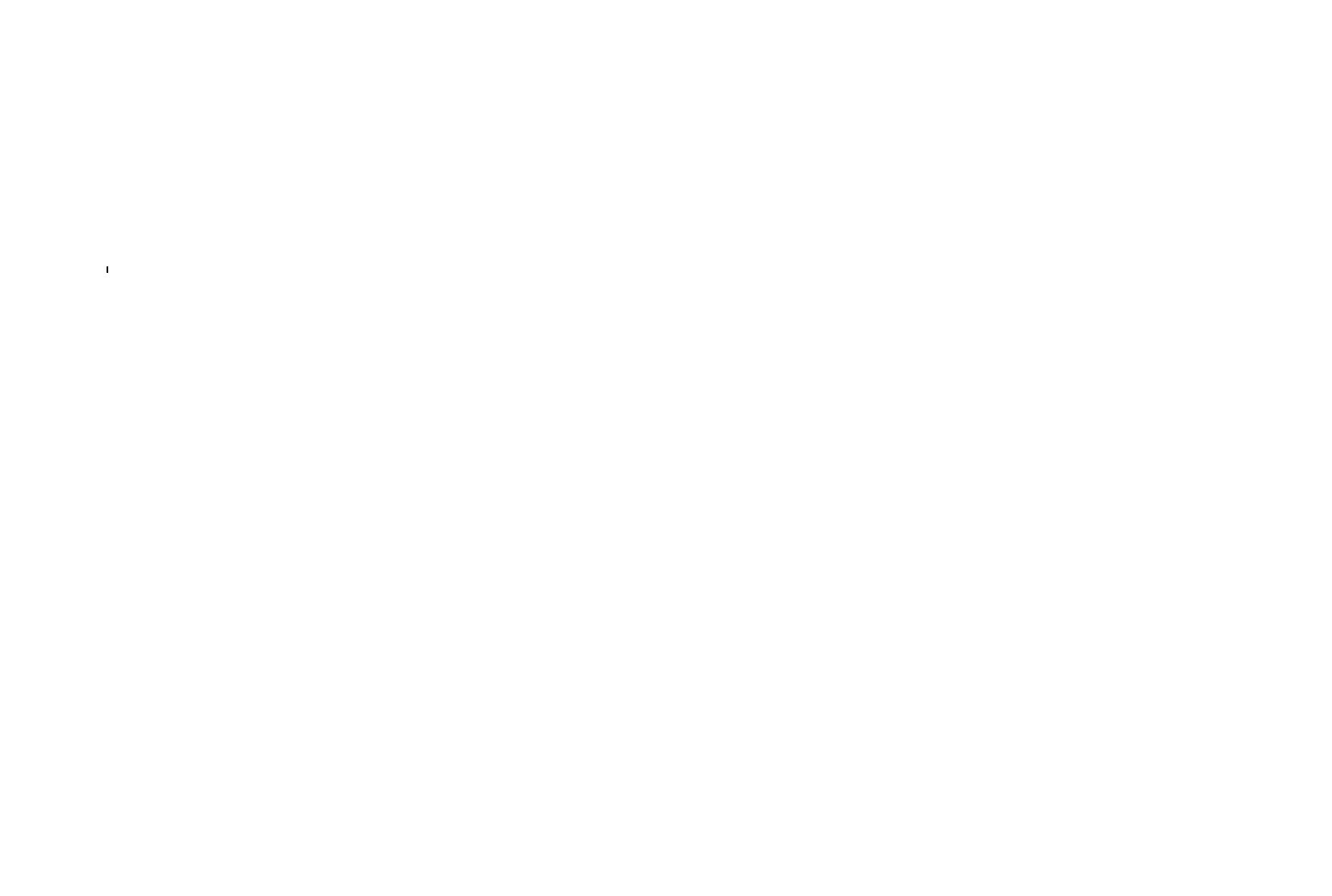
    \caption{Training and validation unscaled loss curves for DETR on the Color Boxes dataset. The training loss is the average over the epoch.}
    \label{fig:convergence_loss_curves_DETR_color-H}
\end{figure}

\begin{figure}
    \centering
    \def\svgwidth{0.8\textwidth}\footnotesize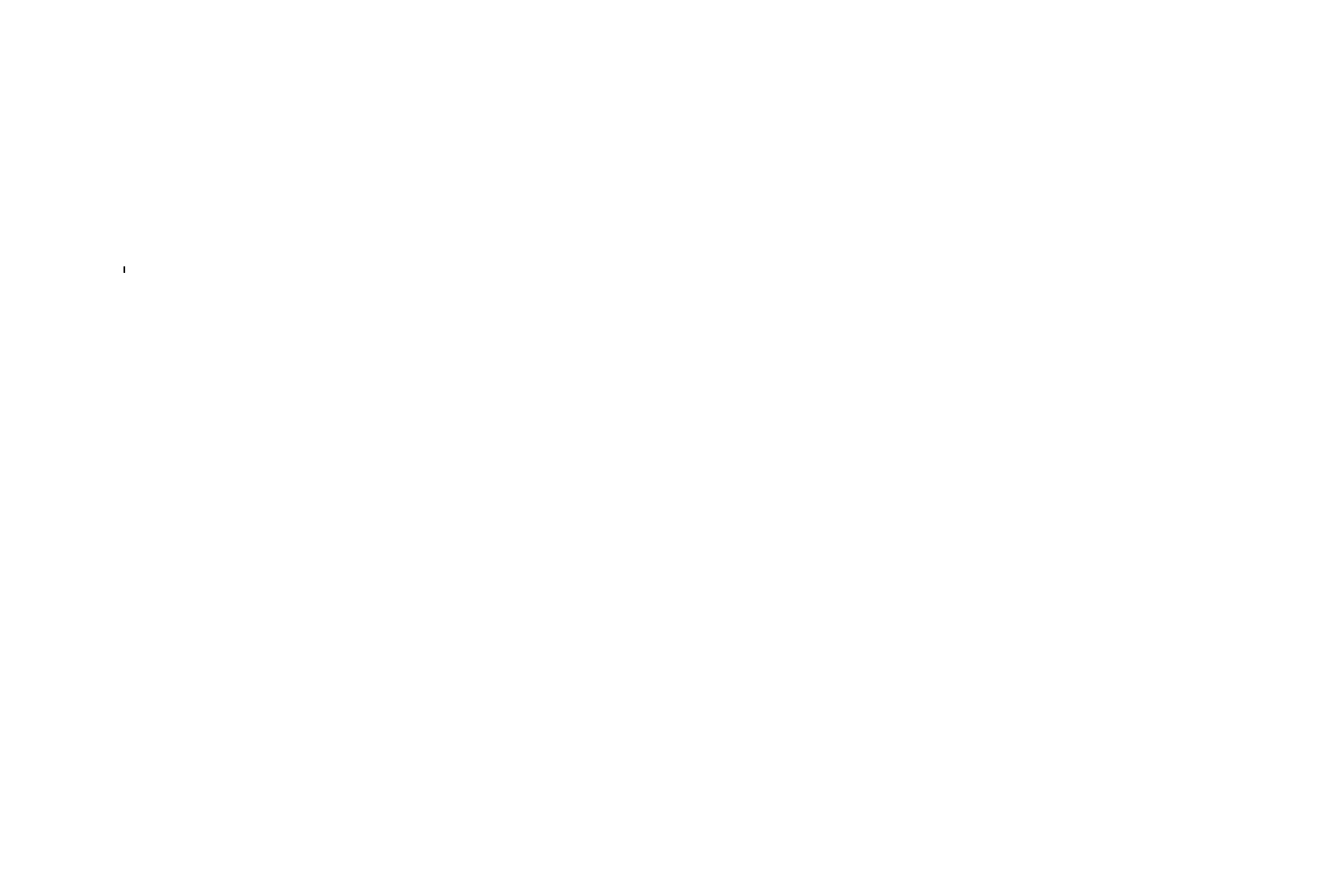
    \caption{Training and validation unscaled loss curves for Deformable DETR on the Color Boxes dataset. The training loss is the average over the epoch.}
    \label{fig:convergence_loss_curves_Def-DETR_color-H}
\end{figure}

\section{Number of Sinkhorn Iterations}
Using a stopping criterion is not straightforward when solving a batch of matching problems. The scaling algorithm is therefore set to a fixed number of iterations. Figure \ref{fig:num_iter_influence} displays the results for different numbers of iterations.
For the balanced OT with 300 predictions (Figure \ref{fig:num_iter_influence_def_detr}), the AP increases only slightly when more than 10 iterations are performed. Furthermore, it is sufficient to run 1 iteration in terms of the AR. 
For the \emph{Unbalanced OT} with 8,732 predictions (Figure \ref{fig:num_iter_influence_ssd}), the metrics are significantly lower when running for less than 5 iterations. Again, running more than 10 iterations only slightly improves the final performance. This fact is supported by Prop.~\ref{prop:unb-softmin}, which shows that in the limit case where $\tau_1 = 0$ and $\tau_2 \rightarrow +\infty$, only one or two iterations are required for convergence (depending on the implementation).

\begin{figure}
    \centering
    \begin{subfigure}{.49\textwidth}
    \centering
        \def\svgwidth{0.97\textwidth}\footnotesize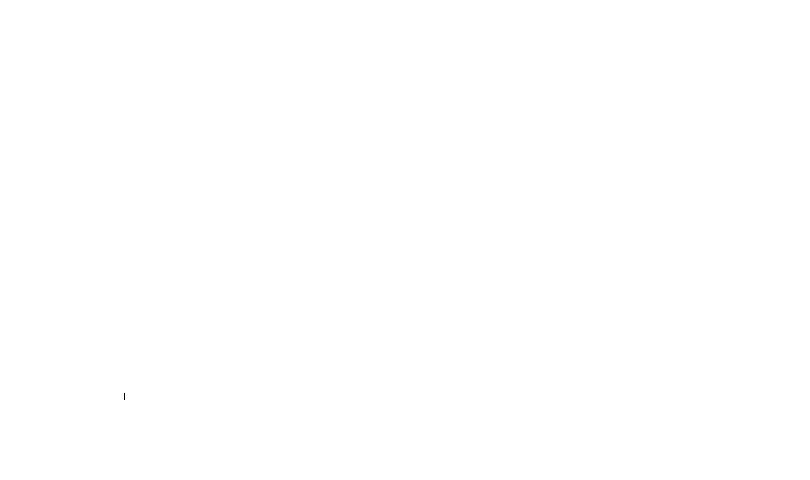
        \caption{Deformable DETR with Balanced OT.}
        \label{fig:num_iter_influence_def_detr}
    \end{subfigure}
    \hfill
    \begin{subfigure}{.49\textwidth}
    \centering
        \def\svgwidth{0.97\textwidth}\footnotesize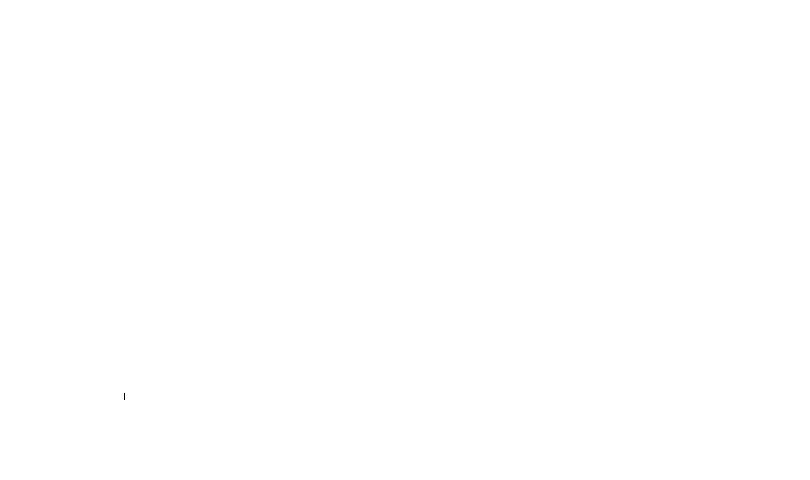
        \caption{SSD300 with Unbalanced OT ($\tau_2 = 0.01$).}
        \label{fig:num_iter_influence_ssd}
    \end{subfigure}
    \caption{Influence of the number of Sinkhorn iterations on the final metrics on the Color Boxes dataset.}
    \label{fig:num_iter_influence}
\end{figure}

\section{First Constraint Parameter}

In this section, we analyze the effect of the prediction's mass constraint parameter $\tau_1$, while we fix parameter $\tau_2$ to a large value $\tau_2=100$ to simulate a hard constraint. Parameter $\tau_1$ controls the degree to which variations in the prediction masses are penalized. 
Each ground-truth object can be matched to the best prediction in the limit case $\tau_2 \to +\infty$ and $\tau_1 = 0$. However, some predictions cannot be matched, and others multiple times. The results for SSD on Color Boxes are displayed in Table \ref{tab:removal_of_nms_ssd_tau1}. Wa can therefore conclude that the first constraint parameter $\tau_1$ has a small influence on the metrics, both with and without NMS. Nevertheless, a higher performance is reached in the balanced case, \ie, when $\tau_1 \to +\infty$.

\begin{table}[] \small
    \centering 
    \begin{tabular}{llllll}
        \toprule
        \multirow{2}{*}{\textbf{Matching}}& \multirow{2}{*}{\textbf{$\bm{\tau_1}$}} & \multicolumn{2}{c}{\textbf{with NMS}} & \multicolumn{2}{c}{\textbf{w/o NMS}} \\ \cmidrule(lr){3-4} \cmidrule(lr){5-6}
         &  & \textbf{AP} & \textbf{AR} & \textbf{AP} & \textbf{AR} \\ 
        \midrule
        Unb. OT & 0.01  & 47.2 & 62.0  & 41.9  & 71.1  \\ 
        Unb. OT & 0.1   & 47.7 & 63.7  & 44.7  & 72.3  \\
        Unb. OT & 1     & 47.7 & 64.0  & 44.8  & 72.7   \\
        Unb. OT & 10    & 47.8 & 63.8  & 45.0  & 72.6   \\ \midrule
        OT      & ($\infty$)   & \textbf{48.1} & \textbf{64.3}  & \textbf{45.2} & \textbf{73.0}  \\
        \bottomrule
    \end{tabular}
    \caption{Comparison of matching strategies on the Color Boxes dataset. SSD300 is evaluated both with and without NMS.}
    \label{tab:removal_of_nms_ssd_tau1}
\end{table}

\section{Timing Analysis for SSD}

As can be seen in Table~\ref{tab:ssd_timings}, OT-based matches improve the epoch time (forward pass, compute the match cost, matching algorithm, and backward pass; in blue) for SSD with the Hungarian algorithm by almost 50\%. The difference is smaller for DETR and variants as the models are proportionally heavier and the number of predictions smaller. 
\begin{table}[h!]
    \small
    \centering
    \begin{tabular}{l | rrrr}
    \toprule
        \textbf{Epoch step}         & \textbf{OT} & \textbf{Unb. OT} & \textbf{Hung.} & \textbf{2-step} \\ \midrule
        Preprocessing & 6.3 ms & \textit{idem} & \textit{idem} & \textit{idem} \\
        \rowcolor{blue!10} Forward pass    & 5.8 ms & \textit{idem} & \textit{idem} & \textit{idem} \\
        Anchor gen. & 54.2 ms & \textit{idem} & \textit{idem} & \textit{idem} \\
        \rowcolor{blue!10}Match cost & 4.2 ms & \textit{idem} & \textit{idem} & \textit{idem} \\
        \rowcolor{blue!20}Matching        & 1.1 ms & 1.5 ms & 18.3 ms & 2.3 ms \\
        \rowcolor{blue!10}Backward pass   & 8.2 ms & \textit{idem} & \textit{idem} & \textit{idem} \\
        Final losses  & 11.6 ms & 11.6 ms & 9.7 ms & 9.7 ms \\ 
        \bottomrule
        \end{tabular}
    \caption{Timing for each step in SSD300 on Color Boxes and a batch size of 16, computed with an Nvidia TITAN X GPU and Intel Core i7-4770K CPU @ 3.50GHz. Likewise the models we built upon, we used \textit{Torchvision's} anchor generation implementation, which extensively relies on heavy loops and could drastically be improved (not the focus of our work). The final losses timings are partially due to the expensive hard-negative mining.} 
    \label{tab:ssd_timings}
\end{table}
\section{Color Boxes Dataset} \label{app:color_boxes_dataset}

This section provides a discussion of the Color Boxes synthetic dataset. It is split into 4,800 training and 960 validation images of 500 $\times$ 400 pixels.
Images have a gray background. We uniformly randomly draw between 0 and 30 rectangles of 20 different colors, which define the category of the rectangle. The dimension of the rectangles vary from 12 to 80 pixels and are uniformly randomly rotated. They are placed such that the $\mathrm{IoU}$ between their bounding boxes is at most $0.25$. A gaussian noise of mean 0 and standard deviation 0.05 is added to each pixel value independently. Sample images are drawn in Fig. \ref{fig:color_boxes_overview}.

\begin{figure*}
    \centering
    \begin{tabular}{cccc}
    \subfloat{\includegraphics[width=0.22\textwidth]{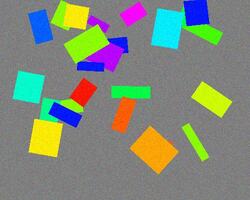}} &
    \subfloat{\includegraphics[width=0.22\textwidth]{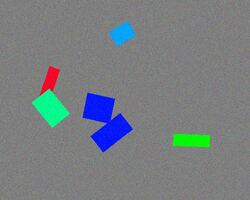}} &
    \subfloat{\includegraphics[width=0.22\textwidth]{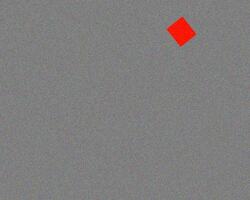}} &
    \subfloat{\includegraphics[width=0.22\textwidth]{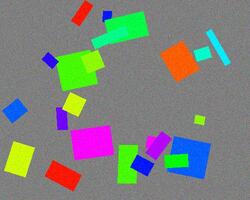}} \\
    \subfloat{\includegraphics[width=0.22\textwidth]{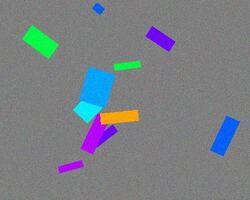}} &
    \subfloat{\includegraphics[width=0.22\textwidth]{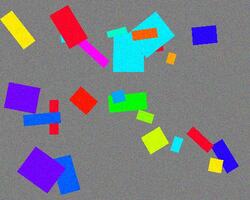}} &
    \subfloat{\includegraphics[width=0.22\textwidth]{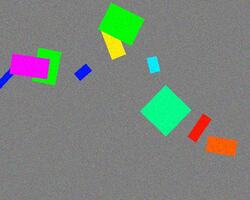}} &
    \subfloat{\includegraphics[width=0.22\textwidth]{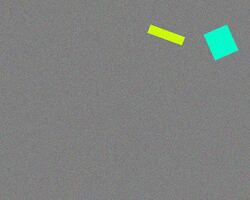}} \\
    \subfloat{\includegraphics[width=0.22\textwidth]{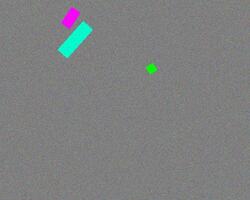}} &
    \subfloat{\includegraphics[width=0.22\textwidth]{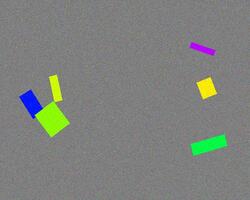}} &
    \subfloat{\includegraphics[width=0.22\textwidth]{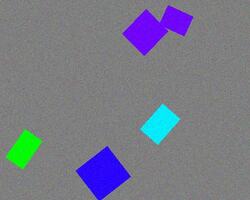}} &
    \subfloat{\includegraphics[width=0.22\textwidth]{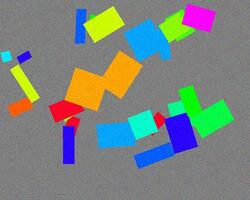}}
    \end{tabular}
    \caption{Sample images from the Color Boxes Dataset.}
    \label{fig:color_boxes_overview}
\end{figure*}


\end{document}